\documentclass{article}

\usepackage{mkolar_definitions}
\usepackage{commath}
\usepackage{enumitem}

\usepackage[final,nonatbib]{neurips_2021}

\usepackage[utf8]{inputenc} % allow utf-8 input
\usepackage[T1]{fontenc}    % use 8-bit T1 fonts
\usepackage{hyperref}       % hyperlinks
\usepackage{url}            % simple URL typesetting
\usepackage{booktabs}       % professional-quality tables
\usepackage{amsfonts}       % blackboard math symbols
\usepackage{nicefrac}       % compact symbols for 1/2, etc.
\usepackage{microtype}      % microtypography
\usepackage{xcolor}         % colors

\usepackage{setspace}
\usepackage[vlined,ruled,linesnumbered,noend]{algorithm2e}
\usepackage{mathtools}
\DeclarePairedDelimiter\ceil{\lceil}{\rceil}

\newcommand{\RHS}{\mathrm{RHS}}
\newcommand{\LHS}{\mathrm{LHS}}
\newcommand{\rw}{\mathrm{rw}}
\newcommand{\prob}{\mathrm{prob}}
\newcommand{\str}{\mathrm{str}}
\newcommand{\opt}{\mathrm{opt}}

\newcommand{\bindp}{\ensuremath{\text{ind-}b}}
\newcommand{\baggp}{\ensuremath{\text{agg-}b}}
\newcommand{\qagglb}{\underline{\ensuremath{\text{agg-}Q}}}
\newcommand{\qaggub}{\overline{\ensuremath{\text{agg-}Q}}}
\newcommand{\qindlb}{\underline{\ensuremath{\text{ind-}Q}}}
\newcommand{\qindub}{\overline{\ensuremath{\text{ind-}Q}}}
\newcommand{\qlb}{\underline{\ensuremath{Q}}}
\newcommand{\qub}{\overline{\ensuremath{Q}}}
\newcommand{\vlb}{\underline{\ensuremath{V}}}
\newcommand{\vub}{\overline{\ensuremath{V}}}

\newcommand{\gap}{\mathrm{gap}}
\newcommand{\Reg}{\mathrm{Reg}}
\newcommand{\Alg}{\mathrm{Alg}}
\newcommand{\fut}{\mathrm{fut}}
\newcommand{\lead}{\mathrm{lead}}
\newcommand{\indiv}{\mathrm{ind}}
\newcommand{\aggre}{\mathrm{agg}}
\newcommand{\val}{\mathrm{val}}
\newcommand{\varia}{\mathrm{var}}
\newcommand{\samp}{\mathrm{sample}}

\DeclareMathOperator{\clip}{clip}
\DeclareMathOperator{\KL}{KL}
\DeclareMathOperator{\Ber}{Ber}
\DeclareMathOperator{\Unif}{Unif}

\newcommand{\hide}[1]{}

\newcommand{\mteuler}{\ensuremath{\textsc{Multi-task-Euler}}\xspace}
\newcommand{\euler}{\ensuremath{\textsc{Euler}}\xspace}
\newcommand{\strongeuler}{\ensuremath{\textsc{Strong-Euler}}\xspace}
\renewcommand{\VV}{\mathrm{var}}

\makeatletter
\newcommand{\vast}{\bBigg@{4}}
\newcommand{\Vast}{\bBigg@{5}}
\makeatother

\title{Provably Efficient Multi-Task Reinforcement Learning with Model Transfer}
\author{%
  Chicheng Zhang \\
  University of Arizona \\
  \texttt{chichengz@cs.arizona.edu} \\
  \And
  Zhi Wang \\
  University of California San Diego \\
  \texttt{zhiwang@eng.ucsd.edu}
}

\begin{document}

\maketitle

\begin{abstract}
We study multi-task reinforcement learning (RL) in tabular episodic Markov decision processes (MDPs). We formulate a heterogeneous multi-player RL problem, in which a group of players concurrently face similar but not necessarily identical MDPs, with a goal of improving their collective performance through inter-player information sharing. We design and analyze an algorithm based on the idea of model transfer, and provide gap-dependent and gap-independent upper and lower bounds that characterize the intrinsic complexity of the problem.
\end{abstract}

\section{Introduction}
\label{sec:intro}
In many real-world applications, reinforcement learning (RL) agents can be deployed as a group to complete similar tasks at the same time.
For example, in healthcare robotics, robots are paired with people with dementia to perform personalized cognitive training activities by learning their preferences \cite{tam16,kubota2020jessie}; in autonomous driving, a set of autonomous vehicles learn how to navigate and avoid obstacles in various environments~\cite{liang2019federated}.
In these settings, each learning agent alone may only be able to acquire a limited amount of data, while the agents as a group have the potential to collectively learn faster through sharing knowledge among themselves. Multi-task learning~\cite{caruana1997multitask} is a practical framework that can be used to model such settings, where a set of learning agents share/transfer knowledge to improve their collective performance.

Despite many empirical successes of multi-task RL (see, e.g.,~\cite{zhuo2019federated,liu2019lifelong,liang2019federated}) and transfer learning for RL (see, e.g., \cite{lrb08,tjs08}), a theoretical understanding of when and how information sharing or knowledge transfer can provide benefits remains limited. Exceptions include~\cite{guo2015concurrent,brunskill2013sample,D'Eramo2020Sharing,hu2021near,pp16,lr11}, which study multi-task learning from parameter- or representation-transfer perspectives. However, these works still do not provide a completely satisfying answer: for example, in many application scenarios, the reward structures and the environment dynamics are only slightly different for each task---this is, however, not captured by representation transfer~\cite{D'Eramo2020Sharing,hu2021near} or existing works on clustering-based parameter transfer~\cite{guo2015concurrent,brunskill2013sample}.
In such settings, is it possible to design provably efficient multi-task RL algorithms that have guarantees never worse than agents learning individually, while outperforming the individual agents in favorable situations? 

%\blue{, e.g.,~

In this work, we formulate an online multi-task RL problem that is applicable to the aforementioned settings.
Specifically, inspired by a recent study on multi-task multi-armed bandits~\cite{wzsrc21}, we formulate the $\epsilon$-Multi-Player Episodic Reinforcement Learning (abbreviated as $\epsilon$-MPERL) problem, in which all tasks share the same state and action spaces, and the tasks are assumed to be similar---i.e., the dissimilarities between the environments of different tasks (specifically, the reward distributions and transition dynamics associated with the players/tasks) are bounded in terms of a dissimilarity parameter $\epsilon \geq 0$.
This problem not only models concurrent RL~\cite{silver2013concurrent,guo2015concurrent} as a special case by taking $\epsilon = 0$, but also captures richer multi-task RL settings when $\epsilon$ is nonzero. We study regret minimization for the $\epsilon$-MPERL problem, specifically:

\begin{enumerate}

    \item We identify a problem complexity notion named {\em subpar} state-action pairs, which captures the amenability to information sharing among tasks in $\epsilon$-MPERL problem instances. As shown in the multi-task bandits literature (e.g., \cite{wzsrc21}), inter-task information sharing is {\em not} always helpful in reducing the players' collective regret. Subpar state-action pairs, intuitively speaking, are clearly suboptimal for all tasks, for which we can robustly take advantage of (possibly biased) data collected for other tasks to achieve a lower regret in a certain task.
    
    \item In the setting where the dissimilarity parameter $\epsilon$ is known, we design a model-based algorithm \mteuler (Algorithm~\ref{alg:multi-task_euler}), which is built upon state-of-the-art algorithms for learning single-task Markov decision processes (MDPs)~\cite{azar2017minimax,zanette2019tighter,simchowitz2019non}, as well as algorithmic ideas of model transfer in RL~\cite{tjs08}. \mteuler crucially utilizes the dissimilarity assumption to robustly take advantage of information sharing among tasks, and achieves regret upper bounds in terms of  subpar state-action pairs, in both (value function suboptimality) {gap-dependent} and {gap-independent} fashions. Specifically, compared with a baseline algorithm that does not utilize information sharing, \mteuler has a regret guarantee that: (1) is never worse, i.e., it avoids negative transfer \cite{rosenstein2005transfer}; (2) can be much superior when there are a large number of subpar state-action pairs.
    
    \item We also present gap-dependent and gap-independent regret lower bounds for the $\epsilon$-MPERL problem in terms of subpar state-action pairs. These lower bounds nearly match the upper bounds when the episode length of the MDP is a constant.
    Together, the upper and lower bounds can be used to characterize the intrinsic complexity of the $\epsilon$-MPERL problem. %\zhi{Is this last sentence too assertive given the gap between the upper and lower bounds?}
    
\end{enumerate}

\section{Preliminaries}
\label{sec:prelims}

Throughout this paper, we denote by $[n] := \cbr{1, \ldots, n}$. For a set $A$ in a universe $U$, we use $A^C = U \setminus A$ to denote its complement.
Denote by $\Delta(\Xcal)$ the set of probability distributions over $\Xcal$. For functions $f, g$, we use $f \lesssim g$ or $f = O(g)$ (resp. $f \gtrsim g$ or $f = \Omega(g)$) to denote that there exists some constant $c>0$, such that $f \leq c g$ (resp. $f \geq c g$), and use $f \eqsim g$ to denote $f \lesssim g$ and $f \gtrsim g$ simultaneously.
Define $a \vee b := \max(a,b)$, and $a \wedge b := \min(a,b)$. We use $\EE$ to denote the expectation operator, and use $\VV$ to denote the variance operator. 
Throughout, 
we use $\tilde{O}(\cdot)$ and $\tilde{\Omega}(\cdot)$ notation to hide polylogarithmic factors.

\paragraph{Multi-task RL in episodic MDPs.} 
We have a set of $M$ MDPs $\cbr{\Mcal_p = (H, \Scal, \Acal, p_0, \PP_p ,r_p)}_{p=1}^M$, each associated with a player $p \in [M]$. Each MDP $\Mcal_p$ is regarded as a task.
The MDPs share the same episode length $H \in \NN_+$, finite state space $\Scal$, finite action space $\Acal$, and initial state distribution $p_0 \in \Delta(\Scal)$.
Let $\bot$ be a default terminal state that is not contained in $\Scal$. The transition probabilities $\PP_p: \Scal \times \Acal \to \Delta(\Scal \cup \cbr{\bot})$
%The transition probabilities $\PP_p: \Scal \times \Acal \to \Delta(\Scal)$ 
and reward distributions $r_p: \Scal \times \Acal \to \Delta([0,1])$ of the players are not necessarily identical. We assume that the MDPs are layered\footnote{This is a standard assumption (see, e.g., \cite{xmd21}). It is worth noting that any episodic MDP (with possibly nonstationary transition and reward) can be converted to a layered MDP with stationary transition and reward, with the state space size  being $H$ times the size of the original state space.}, in that the state space $\Scal$ can be partitioned into disjoint subsets $(\Scal_h)_{h=1}^H$, where $p_0$ is supported on $\Scal_1$, and for every $p \in [M]$, $h \in [H]$, and every $s \in \Scal_h, a \in \Acal$, $\PP_p(\cdot \mid s, a)$ is supported on $\Scal_{h+1}$; here, we define $\Scal_{H+1} = \cbr{\bot}$.
%so that it contains a default terminal state $\bot$ (note that $\bot \notin \Scal$). 
We denote by $S := \abr{\Scal}$ the size of the state space, and $A := \abr{\Acal}$ the size of the action space. 

\paragraph{Interaction process.} The interaction process between the players and the environment is as follows: 
at the beginning, both $(r_p)_{p=1}^M$ and $(\PP_p)_{p=1}^M$ are unknown to the players.
For each episode $k \in [K]$, conditioned on the interaction history up to episode $k-1$, each player $p \in [M]$ independently interacts with its respective MDP $\Mcal_p$; specifically, player $p$ starts at state $s_{1,p}^k \sim p_0$, and at every step (layer) $h \in [H]$, it chooses action $a_{h,p}^k$, transitions to next state $s_{h+1,p}^k \sim \PP_p(\cdot \mid s_{h,p}^k, a_{h,p}^k)$ and receives a stochastic immediate reward $r_{h,p}^k \sim r_p(\cdot \mid s_{h,p}^k, a_{h,p}^k)$; after all players have finished their  $k$-th episode, they can communicate and share their interaction history. The goal of the players is to maximize their expected collective reward $\EE \sbr{ \sum_{k=1}^K \sum_{p=1}^M \sum_{h=1}^H r_{h,p}^k }$.

%\blue{Removed a sentence.}
%as the expected return of player $p$ conditioned on its being at a state at step $h$, and its being at a state and taking an action at step $h$, respectively. %recursive formula
% for that player and that step

\paragraph{Policy and value functions.} A deterministic, history-independent policy $\pi$ is a mapping from $\Scal$ to $\Acal$, which can be used by a player to make decisions in its respective MDP. For player $p$ and step $h$, we use $V_{h,p}^{\pi}: \Scal_h \to [0,H]$ and $Q_{h,p}^{\pi}: \Scal_h \times \Acal \to [0,H]$ to denote its respective value and action-value functions, respectively.
They satisfy the following recurrence known as the Bellman equation:
\[
\forall h \in [H]: \quad V_{h,p}^\pi(s) = Q_{h,p}^{\pi}(s, \pi(s)),
\;\;
Q_{h,p}^\pi(s, a) 
= 
R_p(s, a) 
+ (\PP_{p} V_{h+1,p}^\pi)( s, a),
\]
where we use the convention that $V_{H+1,p}^\pi(\bot) = 0$, and for $f: \Scal_{h+1} \to \RR$, $(\PP_p f)(s, a) := \sum_{s' \in \Scal_{h+1}} \PP_{p}(s' \mid s, a) f(s')$, and $R_p(s, a) := \EE_{\hat{r} \sim r_p(\cdot \mid s,a)}\sbr{\hat{r}}$ is the expected immediate reward of player $p$. For player $p$ and policy $\pi$, denote by $V_{0,p}^\pi = \EE_{s_1 \sim p_0}\sbr{ V_{1,p}^{\pi}(s_1) }$ its expected reward. 

For player $p$, we also define its optimal value function $V_{h,p}^\star: \Scal_h \to [0,H]$ and the optimal action-value function $Q_{h,p}^\star: \Scal_h \times \Acal \to [0,H]$ using the Bellman optimality equation:
\begin{equation}
\forall h \in [H]: \quad
V_{h,p}^\star(s) = \max_{a \in \Acal} Q_{h,p}^{\star}(s, a),
\;\;
Q_{h,p}^\star(s, a) 
= 
R_p(s, a) 
+
(\PP_p V_{h+1,p}^\star)(s,a)
,
\label{eqn:bellman-optimality-q-p}
\end{equation}
where we again use the convention that $V_{H+1,p}^\star(\bot) = 0$. For player $p$, denote by $V_{0,p}^\star = \EE_{s_1 \sim p_0}\sbr{ V_{1,p}^{\star}(s_1) }$ its optimal expected reward. 

Given a policy $\pi$, as $V_{h,p}^\pi$ for different $h$'s are only defined in the respective layer $\Scal_h$, we  ``collate'' the value functions $(V_{h,p}^\pi)_{h=1}^H$ and obtain a single value function $V_p^\pi: \Scal \cup \cbr{\bot} \to \RR$. Formally, for every $h \in [H+1]$ and  $s \in \Scal_h$,
\[
V_p^{\pi}(s) := V_{h, p}^{\pi}(s). 
\]
We define $Q_{p}^{\pi}, V_{p}^{\star}, Q_{p}^{\star}$ similarly. For player $p$, given its optimal action value function $Q_p^\star$, any of its greedy policies $\pi_p^\star(s) \in \argmax_{a \in \Acal} Q_p^\star(s,a)$ is optimal with respect to $\Mcal_p$.  

%its optimal policy $\pi_p^\star: \Scal \to \Acal$ is greedy with respect to $Q_p^\star$, that is, 

\paragraph{Suboptimality gap.} For player $p$, we define the suboptimality gap of state-action pair $(s,a)$ as 
$\gap_p(s,a) = V_p^\star(s) - Q_p^\star(s,a)$. We define the mininum suboptimality gap of player $p$ as $\gap_{p,\min} = \min_{(s,a): \gap_p(s,a) > 0} \gap_p(s,a)$, and the minimum suboptimality gap over all players as $\gap_{\min} = \min_{p \in [M]} \gap_{p,\min}$. 
For player $p \in [M]$, define
$Z_{p,\opt} := \cbr{ (s,a): \gap_p(s,a) = 0}$ as the set of optimal state-action pairs with respect to $p$. 

%studied in this paper
%metric 
\paragraph{Performance metric.} We measure the performance of the players using their collective regret, i.e., over a total of $K$ episodes, how much extra reward they would have collected in expectation if they were executing their respective  optimal policies from the beginning. Formally, suppose for each episode $k$, player $p$ executes policy $\pi^k(p)$, then the collective regret of the players is defined as: 
\[
\Reg(K) = 
\sum_{p=1}^M \sum_{k=1}^K 
\rbr{  
V_{0,p}^\star 
-
V_{0,p}^{\pi^k(p)}
}.
\]

\paragraph{Baseline: individual \strongeuler.} A naive baseline for multi-task RL is to let each player run a separate RL algorithm
without communication. 
For concreteness, we choose to let each player run the state-of-the-art \strongeuler algorithm~\cite{simchowitz2019non} (see also its precursor \euler~\cite{zanette2019tighter}), which enjoys minimax gap-independent~\cite{azar2017minimax,dann2015sample} and gap-dependent regret guarantees, and we refer to this strategy as individual \strongeuler.
Specifically, as it is known that \strongeuler has a regret of $\tilde{O}(\sqrt{ H^2 S A K } + H^4 S^2 A$), individual \strongeuler has a collective regret of $\tilde{O}(M \sqrt{H^2 S A K} + M H^4 S^2 A)$.
In addition, by a union bound and summing up the gap-dependent regret guarantees of \strongeuler for the $M$ MDPs altogether, it can be checked that with probability $1-\delta$,
individual \strongeuler has a collective regret of order\footnote{The originally-stated gap-dependent regret bound of \strongeuler (\cite{simchowitz2019non}, Corollary 2.1) uses a slightly different notion of suboptimality gap, which takes an extra minimum over all steps. A close examination of their proof shows that \strongeuler has regret bound~\eqref{eqn:ind-strong-euler-gap-dept} in layered MDPs. See also Remark~\ref{rem:clipping} in Appendix~\ref{subsec:conclude-reg-bounds}.} 
% in this paper
%\blue{
%$\Reg(K) 
%\lesssim$ 
\begin{align}
\begin{split}
\ln\del{\frac{MSAK}{\delta}} \vast( & \sum_{p \in [M]} \rbr{ 
\sum_{(s,a) \in Z_{p,\opt}} \frac{H^3}{\gap_{p,\min}}
+
\sum_{(s,a) \in Z_{p,\opt}^C} \frac{H^3}{\gap_p(s,a)}
}  + M H^4 S^2 A  \ln\frac{SA}{\gap_{\min}} \vast).
\end{split}
\label{eqn:ind-strong-euler-gap-dept}
\end{align}
%\\
%&
%}

Our goal is to design multi-task RL algorithms that can achieve collective regret strictly lower than this baseline in both gap-dependent and gap-independent fashions when the tasks are similar. 

\paragraph{Notion of similarity.} Throughout this paper, we will consider the following notion of similarity between MDPs in the multi-task episodic RL setting. 

\begin{definition}
A collection of MDPs $(\Mcal_p)_{p=1}^M$ is said to be $\epsilon$-dissimilar for $\epsilon \geq 0$, if for all $p, q \in [M]$, and $(s,a) \in \Scal \times \Acal$,
\[
 \abr{ R_p(s,a) - R_q(s,a) } \leq \epsilon, \;
 \| \PP_p(\cdot \mid s,a) - \PP_q(\cdot \mid s,a) \|_1 \leq \frac{\epsilon}{H}. 
\]
If this happens, we call  $(\Mcal_p)_{p=1}^M$ an $\epsilon$-Multi-Player Episodic Reinforcement Learning (abbrev. $\epsilon$-MPERL) problem instance.
\label{def:dissimilar}
\end{definition}

If the MDPs in $(\Mcal_p)_{p=1}^M$ are $0$-dissimilar, then they are identical by definition, and our interaction protocol degenerates to the concurrent RL protocol~\cite{silver2013concurrent}. Our dissimilarity notion is complementary to those of \cite{brunskill2013sample,guo2015concurrent}: they require the MDPs to be either identical, or have well-separated parameters for at least one state-action pair; in contrast, our dissimilarity notion allows the MDPs to be nonidentical and arbitrarily close. 

We have the following intuitive lemma that shows the closeness of optimal value functions of different MDPs, in terms of the dissimilarity parameter $\epsilon$: 
\begin{lemma}
\label{lem:q-dissimilar}
If $(\Mcal_p)_{p=1}^M$ are $\epsilon$-dissimilar, then for every $p, q \in [M]$, and $(s,a) \in \Scal \times \Acal$, 
$\abr{ Q_p^\star(s,a) - Q_q^\star(s,a) } \leq 2 H \epsilon$; consequently,
$
\abr{ \gap_p(s,a) - \gap_q(s,a) } \leq 4H\epsilon.
$
\end{lemma}

\section{Algorithm}

We now describe our main algorithm, \mteuler (Algorithm~\ref{alg:multi-task_euler}). Our model-based algorithm is built upon recent works on episodic RL
that provide algorithms with sharp instance-dependent guarantees in the single task setting~\cite{zanette2019tighter,simchowitz2019non}. In a nutshell, for each episode $k$ and each player $p$, the algorithm performs optimistic value iteration to construct high-probability upper and lower bounds for the optimal value and action value functions $V_p^\star$ and $Q_p^\star$, and uses them to guide its exploration and decision making process.

\begin{algorithm}[t]
\setstretch{1.25}
\SetAlgoLined
\SetKwInOut{Input}{Input}
\SetKw{Return}{Return}
\SetKwInput{kwInit}{Initialize}
\newcommand\mycommfont[1]{\footnotesize\ttfamily\textcolor{blue}{#1}}
\SetCommentSty{mycommfont}

\Input{Failure probability $\delta \in (0,1)$, dissimilarity parameter $\epsilon \geq 0$.}
\kwInit{Set $V_{p}(\bot) = 0$ for all $p$ in $[M]$, where $\bot$ is the only state in $\Scal_{H+1}$ \;}

\For{$k=1,2,\ldots,K$}{
    \For{$p=1,2,\ldots,M$}{
        \tcp{Construct optimal value estimates for player $p$}
        \For{$h=H,H-1,\ldots,1$}{
        \label{line:vi-start}
            \For{$(s,a) \in \Scal_h \times \Acal$}{
                Compute:\\
                \label{line:q-estimates-start} %
            
                $\qindub_{p}(s,a) = \hat{R}_{p}(s,a) +  (\hat{\PP}_{p}\vub_p)(s,a) + \bindp_{p}(s,a)$\; 
                \label{line:self-q-estimate-ub}%
  
                $\qindlb_{p}(s,a) = \hat{R}_{p}(s,a) +  (\hat{\PP}_{p}\vlb_p)(s,a) - \bindp_{p}(s,a)$\;
                \label{line:self-q-estimate-lb}%
                
                $\qaggub_p(s,a) = \hat{R}(s,a) + (\hat{\PP} \vub_p)(s,a) + \baggp_p(s,a)$\;
                \label{line:agg-q-estimate-ub} %
               
                $\qagglb_p(s,a) = \hat{R}(s,a) + (\hat{\PP} \vlb_p)(s,a) - \baggp_p(s,a)$\;
                \label{line:agg-q-estimate-lb}
 
                Update optimal action value function upper and lower bound estimates:
                
                $\qub_p(s,a) = \min \cbr{H-h+1, \qindub_p(s,a), \qaggub_p(s,a)}$\;
                \label{line:q-ucb}
                
                $\qlb_p(s,a) = \max \cbr{0, \qindlb_p(s,a), \qagglb_p(s,a)}$\;
                \label{line:q-lcb}
                \label{line:q-estimates-end}
            }
            \For{$s \in \Scal_h$}{
                Define $\pi^k(p)(s) = \argmax_{a \in \Acal} \qub_{p}(s,a)$ \;  
                \label{line:greedypolicy}
            % (breaking ties lexicographically)
                Update $\vub_p(s) = \qub_p\del{ s,\pi^k(p)(s) }$, $\vlb_p(s) = \qlb_p\del{ s,\pi^k(p)(s) }$.
                \label{line:v-lucb}
            }
        }
        \label{line:vi-end}
    }
    
    \tcp{All players $p$ interact with their respective environments, and update reward and transition estimates}
    
    \For{$p=1,2,\ldots,M$}{
            \label{line:execute-start}
            Player $p$ executes policy $\pi^k(p)$ on $\Mcal_p$ and obtains trajectory $(s_{h,p}^k, a_{h,p}^k, r_{h,p}^k)_{h=1}^H$.
            
            Update individual estimates of transition probability $\hat{\PP}_p$, reward $\hat{R}_p$ and count $n_p(\cdot,\cdot)$ using the first parts of Equations~\eqref{eqn:n-estimates},~\eqref{eqn:r-estimates} and~\eqref{eqn:p-estimates}.
            }
            Update aggregate estimates of transition probability $\hat{\PP}$, reward $\hat{R}$ and count $n(\cdot,\cdot)$ using the second parts of Equations~\eqref{eqn:n-estimates},~\eqref{eqn:r-estimates} and~\eqref{eqn:p-estimates}.
            \label{line:execute-end}
}
\caption{\mteuler}
\label{alg:multi-task_euler}
\end{algorithm}

\paragraph{Empirical estimates of model parameters.} For each player $p$, the construction of its value function bound estimates relies on empirical estimates on its transition probability and expected reward function. For both estimands, we use two estimators with complementary roles, which are at two different points of the bias-variance tradeoff spectrum: 
one estimator uses only the player's own data (termed {\em individual estimate}), which has large variance; the other estimator uses the data collected by all players (termed {\em aggregate estimate}), which has lower variance but can easily be biased, as transition probabilities and reward distributions are heterogeneous. Such an algorithmic idea of ``model transfer'', where one estimates model in one task using data collected from other tasks has appeared in prior works (e.g., ~\cite{tjs08}).
%one estimator uses only the player's own data (termed {\em individual estimate}), which is unbiased but has large variance, the other estimator uses the data collected by all players (termed {\em aggregate estimate}), which is biased but has lower variance.
Specifically, at the beginning of episode $k$, for every $h \in [H]$ and $(s,a) \in \Scal_h \times \Acal$, the algorithm has its empirical count of encountering $(s,a)$ for each player $p$, along with its total empirical count across all players,  respectively: 
%\zhi{Does it make more sense for these values to be at the beginning of episode $k$? Also, it is more consistent with the proofs.} \chicheng{I agree.}
\begin{equation}
n_{p}(s,a) := \sum_{l=1}^{k-1} \one\rbr{ (s_{h,p}^l, a_{h,p}^l) = (s,a) }, \;
n(s,a) := \sum_{l=1}^{k-1} \sum_{p=1}^M \one\rbr{ (s_{h,p}^l, a_{h,p}^l) = (s,a) }.
\label{eqn:n-estimates}
\end{equation}
The individual and aggregate estimates of immediate reward $R(s,a)$ are defined as:%
\begin{equation}
\hat{R}_p(s,a) := \frac{\sum_{l=1}^{k-1} \one\rbr{ (s_{h,p}^l, a_{h,p}^l) = (s,a) } r_{h,p}^l }{n_p(s,a)}, \;
\hat{R}(s,a) := \frac{\sum_{l=1}^{k-1} \sum_{p=1}^M \one\rbr{ (s_{h,p}^l, a_{h,p}^l) = (s,a) } r_{h,p}^l}{n(s,a)}.
\label{eqn:r-estimates}
\end{equation}
Similarly, for every $h \in [H]$ and $(s,a,s') \in \Scal_h \times \Acal \times \Scal_{h+1}$, we also define the individual and aggregate estimates of transition probability as:
\begin{align}
\begin{split}
\hat{\PP}_{p}(s' \mid s,a) & := \frac{\sum_{l=1}^k \one\rbr{ (s_{h,p}^l, a_{h,p}^l, s_{h+1,p}^l) = (s,a,s') }}{n_{p}(s,a)},\\
\hat{\PP}(s' \mid s,a) & := \frac{\sum_{l=1}^k \sum_{p=1}^M \one\rbr{ (s_{h,p}^l, a_{h,p}^l, s_{h+1,p}^l) = (s,a,s') }}{n(s,a)}.
\end{split}
\label{eqn:p-estimates}
\end{align}
If $n(s,a) = 0$, we define $\hat{R}(s,a) := 0$ and $\hat{\PP}(s' \mid s,a) := \frac{1}{\abr{\Scal_{h+1}}}$; and if $n_p(s,a) = 0$, we define $\hat{R}_p(s,a) := 0$ and $\hat{\PP}_p(s' \mid s,a) := \frac{1}{\abr{\Scal_{h+1}}}$.
%\zhi{Should we explicitly mention that $n(s,a)$ and $n_p(s,a)$ are initialized as $0$?} \chicheng{I don't think we need to do so - Equation (3) already defines them.}
The counts and reward estimates can be maintained by \mteuler efficiently in an incremental manner. 

\paragraph{Constructing value function estimates via optimistic value iteration.} For each player $p$, based on these model parameter estimates, \mteuler performs optimistic value iteration to compute the value function estimates for states at all layers (lines~\ref{line:vi-start} to~\ref{line:vi-end}). For the terminal layer $H+1$, $V_{H+1}^\star(\bot) = 0$ trivially, so nothing needs to be done. For earlier layers $h \in [H]$, \mteuler iteratively builds its value function estimates in a backward fashion. At the time of estimating values for layer $h$, the algorithm has already obtained optimal value estimates for layer $h+1$. Based on the Bellman optimality equation~\eqref{eqn:bellman-optimality-q-p}, \mteuler estimates $(Q_p^\star(s, a))_{s \in \Scal_h, a \in \Acal}$ using model parameter estimates and its estimates of $(V_p^\star(s))_{s \in \Scal_{h+1}}$, i.e., $(\vub_p(s))_{s \in \Scal_{h+1}}$ and $(\vlb_p(s))_{s \in \Scal_{h+1}}$ (lines~\ref{line:q-estimates-start} to~\ref{line:q-estimates-end}). 

Specifically, \mteuler constructs estimates of $Q_p^\star(s,a)$ for all $s \in \Scal_h, a \in \Acal$ in two different ways. First, it uses the individual estimates of model of player $p$ to construct $\qindlb_p$ and $\qindub_p$, upper and lower bound estimates of $Q_p^\star$ (lines~\ref{line:agg-q-estimate-ub} and~\ref{line:agg-q-estimate-lb}); this construction is reminiscent of \euler and \strongeuler~\cite{zanette2019tighter,simchowitz2019non}, in that if we were only to use $\qindlb_p$ and $\qindub_p$ as our optimal action value function estimate $\qub_p$ and $\qlb_p$, our algorithm becomes individual \strongeuler. 
The individual value function estimates are key to establishing \mteuler's fall-back guarantees, ensuring that it never performs worse than the individual \strongeuler baseline. 
Second, it uses the aggregate estimate of model to construct $\qagglb_p$ and $\qaggub_p$, also upper and lower bound estimates of $Q_p^\star$ (lines~\ref{line:self-q-estimate-ub} and~\ref{line:self-q-estimate-lb}); this construction is unique to the multitask learning setting, and is our new algorithmic contribution. 

To ensure that $\qaggub_p$ and $\qindub_p$ (resp. $\qagglb_p$ and $\qindlb_p$) are valid upper bounds (resp. lower bounds) of $Q_p^\star$, \mteuler adds bonus terms $\bindp_p(s,a)$ and $\baggp_p(s,a)$, respectively, in the optimistic value iteration process, to account for estimation error of the model estimates against the true models. 
Specifically, both bonus terms comprise three parts:
%\blue{
\begin{align*}
\bindp_p(s,a) := b_{\rw}\del{n_p(s,a), 0} + &\ b_{\prob}\del{\hat{\PP}_p(\cdot \mid s,a), n_p(s,a), \overline{V}_p, \underline{V}_p, 0} + \\
& \qquad \qquad \qquad b_{\str}\del{\hat{\PP}_p(\cdot \mid s,a), n_p(s,a), \overline{V}_p, \underline{V}_p, 0},\\
\baggp_p(s,a) := b_{\rw}\del{n(s,a), \epsilon} + &\ b_{\prob}\del{\hat{\PP}(\cdot \mid s,a), n(s,a), \overline{V}_p, \underline{V}_p, \epsilon} + \\
& \qquad \qquad \qquad b_{\str}\del{\hat{\PP}(\cdot \mid s,a), n(s,a), \overline{V}_p, \underline{V}_p, \epsilon},
\end{align*}
%}
%\begin{align*}
%\bindp_p(s,a) & := b_{\rw}\del{n_p(s,a), 0} + b_{\prob}\del{\hat{\PP}_p(\cdot \mid s,a), n_p(s,a), \overline{V}_p, \underline{V}_p, 0} +  b_{\str}\del{\hat{\PP}_p(\cdot \mid s,a), n_p(s,a), \overline{V}_p, \underline{V}_p, 0},\\
%\baggp_p(s,a) & := b_{\rw}\del{n(s,a), \epsilon} + b_{\prob}\del{\hat{\PP}(\cdot \mid s,a), n(s,a), \overline{V}_p, \underline{V}_p, \epsilon} + b_{\str}\del{\hat{\PP}(\cdot \mid s,a), n(s,a), \overline{V}_p, \underline{V}_p, \epsilon}, 
%\end{align*}
where
%\blue{
\begin{align*}
b_{\rw}\del{n, \kappa} & := 1 \;\; \wedge \;\;  \kappa + \Theta\del{ \sqrt{ \frac{L(n)}{n} } } , \\
b_{\prob}\del{q, n, \overline{V}, \underline{V}, \kappa} & := H \;\; \wedge \;\; 2\kappa + \Theta\vast( \sqrt{ \frac{\VV_{s' \sim q}\sbr{ \overline{V}(s') } L(n) }{n} } + \\
& \qquad \qquad \qquad \qquad \ \sqrt{ \frac{ \EE_{s' \sim q}\sbr{ (\overline{V}(s') - \underline{V}(s'))^2 } L(n) }{n} }  + \frac{H L(n)}{n} \vast) , \\
b_{\str}\del{q, n, \overline{V}, \underline{V} ,\kappa} & := \kappa + \Theta\del{ \sqrt{ \frac{S \; \EE_{s' \sim q}\sbr{ (\overline{V}(s') - \underline{V}(s'))^2} L(n) }{n} } + \frac{H S L(n)}{n} },
\end{align*}
%}
%\begin{align*}
%b_{\rw}\del{n, \kappa} & := 1 \;\; \wedge \;\;  \kappa + \Theta\del{ \sqrt{ \frac{L(n)}{n} } } , \\
%b_{\prob}\del{q, n, \overline{V}, \underline{V}, \kappa} & := H \;\; \wedge \;\; 2\kappa + \Theta\del{ \sqrt{ \frac{\VV_{s' \sim q}\sbr{ \overline{V}(s') } L(n) }{n} } + \sqrt{ \frac{ \EE_{s' \sim q}\sbr{ (\overline{V}(s') - \underline{V}(s'))^2 } L(n) }{n} }  + \frac{H L(n)}{n} } , \\
%b_{\str}\del{q, n, \overline{V}, \underline{V} ,\kappa} & := \kappa + \Theta\del{ \sqrt{ \frac{S \; \EE_{s' \sim q}\sbr{ (\overline{V}(s') - \underline{V}(s'))^2} L(n) }{n} } + \frac{H S L(n)}{n} },
%\end{align*}
and $L(n) \eqsim \ln(\frac{M S A n}{\delta})$.

The bonus terms altogether ensures strong optimism~\cite{simchowitz2019non}, i.e.,  
\begin{equation}
\text{for any $p$ and $(s,a)$,} \;\; \qub_p(s,a) \geq R_p(s,a) + (\PP_p \vub_p)(s,a).
\label{eqn:strong-optimism-main}
\end{equation}
In short, strong optimism is a stronger form of optimism (the weaker requirement that for any $p$ and $(s,a)$, $\qub_p(s,a) \geq Q^\star_p(s,a)$ and $\vub_p(s) \geq V^\star_p(s)$), which allows us to use the clipping lemma (Lemma B.6 of~\cite{simchowitz2019non}, see also Lemma~\ref{lem:clipping-main} in Appendix~\ref{subsec:conclude-reg-bounds}) to obtain sharp gap-dependent regret guarantees. 
The three parts in the bonus term
serve for different purposes towards establishing~\eqref{eqn:strong-optimism-main}: 
\begin{enumerate}
    \item The first component accounts for the uncertainty in reward estimation: with  probability $1-O(\delta)$, $\abs{ \hat{R}_p(s,a) - R_p(s,a) } \leq b_{\rw} \del{ n_p(s,a), 0} $, and $\abs{ \hat{R}(s,a) - R_p(s,a) } \leq b_{\rw} \del{ n(s,a), \epsilon }$. 
    \item The second component accounts for the uncertainty in estimating $(\PP_p V_p^\star)(s,a)$: with probability $1-O(\delta)$, $\abs{ (\hat{\PP}_p V_p^\star)(s,a) - (\PP_p V_p^\star)(s,a) } \leq b_{\prob} \del{ \hat{\PP}_p(\cdot \mid s,a), n_p(s,a), \vub_p, \vlb_p, 0 }$ and $\abs{ (\hat{\PP} V_p^\star)(s,a) - (\PP_p V_p^\star)(s,a) } \leq b_{\prob} \del{ \hat{\PP}(\cdot \mid s,a), n(s,a), \vub_p, \vlb_p, \epsilon }$.
    \item The third component accounts for the lower order terms for strong optimism: with probability $1-O(\delta)$,
    $\abs{ (\hat{\PP}_p - \PP_p) (\vub_p - V_p^\star)(s,a)} \leq b_{\str} \del{ \hat{\PP}_p(\cdot \mid s,a), n_p(s,a), \vub_p, \vlb_p, 0 }$, and $\abs{ (\hat{\PP} - \PP_p) (\vub_p - V_p^\star)(s,a)} \leq b_{\str} \del{ \hat{\PP}(\cdot \mid s,a), n(s,a), \vub_p, \vlb_p, \epsilon }$. 
    %\zhi{The last term here should be $b_{\str}$? Should we explain what strong optimism means? Because readers may find that the three components are set up to show $\bar{Q} - (R + \PP \bar{V}) \ge 0$ instead of $\bar{Q} - Q^\star = \bar{Q} - (R + \PP V^\star)$ which is mentioned in the next paragraph.}
\end{enumerate}

Based on the above concentration inequalities and the definitions of bonus terms, it can be shown inductively that, with probability $1-O(\delta)$, both $\qaggub_p$ and $\qindub_p$ (resp. $\qagglb_p$ and $\qindlb_p$) are valid upper bounds (resp. lower bounds) of $Q_p^\star$.

Finally, observe that for any $(s,a) \in \Scal_h \times \Acal$, $Q_p^\star(s,a)$ has range $[0,H-h+1]$. By taking intersections of all confidence bounds of $Q_p^\star$ it has obtained, \mteuler constructs its final upper and lower bound estimates for $Q_p^\star(s,a)$, $\qub_p(s,a)$ and $\qlb_p(s,a)$ respectively, for $(s,a) \in \Scal_h \times \Acal$ (line~\ref{line:q-ucb} to~\ref{line:q-lcb}). 
Similar ideas on using data from multiple sources to construct confidence intervals and guide explorations have been used by~\cite{soaremulti,wzsrc21} for multi-task noncontextual and contextual bandits.
Using the relationship between the optimal value $V_p^\star(s)$ and and optimal action values $\cbr{ Q_p^\star(s,a): a \in \Acal}$, \mteuler also constructs upper and lower bound estimates for $V_p^\star(s)$, $\vub_p(s)$ and $\vlb_p(s)$, respectively for $s \in \Scal_h$ (line~\ref{line:v-lucb}).

\paragraph{Executing optimistic policies.} At each episode $k$, for each player $p$, its optimal action-value function upper bound estimate $\overline{Q}_p$ induces a greedy policy $\pi^k(p): s \mapsto \argmax_{a \in \Acal} \overline{Q}_p(s,a)
$ (line~\ref{line:greedypolicy}); the player then executes this policy in this episode to collect a new trajectory and use this to update its individual model parameter estimates. After all players finish their episode $k$, the algorithm also updates its aggregate model parameter estimates (lines~\ref{line:execute-start} to~\ref{line:execute-end}) using Equations~\eqref{eqn:n-estimates},~\eqref{eqn:r-estimates} and~\eqref{eqn:p-estimates}, and continues to the next episode.

\section{Performance guarantees}
\label{sec:guarantees}

Before stating the guarantees of Algorithm~\ref{alg:multi-task_euler}, we define an instance-dependent complexity measure that characterizes the amenability to information sharing. 

\begin{definition}
The set of subpar state-action pairs is defined as:
\[
\Ical_\epsilon := \cbr{ (s,a) \in \Scal \times \Acal: \exists p \in [M], \gap_p(s,a) > 96 H \epsilon }, 
\]
where we recall that $\gap_p(s,a) = V^\star_p(s) - Q^\star_p(s,a)$. %
\label{def:subpar-sa}
\end{definition}
%\chicheng{Note: I changed "$\geq 96 H \epsilon$" to $> 96 H \epsilon$. If we use the former, then the lemma below is not true for $\epsilon = 0$.}

Definition~\ref{def:subpar-sa} generalizes the notion of subpar arms defined for multi-task multi-armed bandit learning~\cite{wzsrc21} in two ways: first, it is with regards to state-action pairs as opposed to actions only; second, 
in RL, suboptimality gaps depend on optimal value function, which in turn depends on both immediate reward and subsequent long-term return.

To ease our later presentation, we also present the following  lemma.

\begin{lemma}
\label{lem:gap-const}
For any $(s,a) \in \Ical_\epsilon$, we have that: (1) for all $p \in [M]$, $(s,a) \notin Z_{p,\opt}$, where we recall that $Z_{p,\opt} = \cbr{ (s,a): \gap_p(s,a) = 0}$ is the set of optimal state-action pairs with respect to $p$; (2) for all $p,q \in [M]$, $\gap_p(s,a) \geq \frac12\gap_q(s,a)$.
\end{lemma}

The lemma follows directly from Lemma~\ref{lem:q-dissimilar}; its proof can be found in the Appendix along with proofs of the following theorems. Item 1 implies that any subpar state-action pair is suboptimal for all players. In other words, for every player $p$, the state-action space $\Scal \times \Acal$ can be partitioned to three disjoint sets: $\Ical_\epsilon, Z_{p,\opt}, (\Ical_\epsilon \cup Z_{p,\opt})^C$. Item 2 implies that for any subpar $(s,a)$, its suboptimal gaps with respect to all players are within a constant of each other. 

\subsection{Upper bounds}
With the above definitions, we are now ready to present the performance guarantees of Algorithm~\ref{alg:multi-task_euler}. 
We first present a gap-independent collective regret bound of \mteuler.

\begin{theorem}[Gap-independent bound] 
\label{thm:gap_indept_upper}
If $\cbr{\Mcal_p}_{p=1}^M$ are $\epsilon$-dissimilar, then \mteuler satisfies that with probability $1-\delta$,
\[
\Reg(K) \leq \tilde{O}\del{ 
M \sqrt{ H^2 |\Ical_\epsilon^C| K } + \sqrt{ M H^2 |\Ical_\epsilon| K }  + M H^4  S^2 A  }. \]
\end{theorem}
We again compare this regret upper bound with individual \strongeuler's gap independent regret bound. Recall that individual \strongeuler guarantees that with probability $1-\delta$, 
\[ 
\Reg(K) \leq
\tilde{O}\del{ M \sqrt{ H^2 S A K } +  M H^4  S^2 A  }.
\]
We focus on the comparison on the leading terms, i.e., the $\sqrt{K}$ terms. As $M \sqrt{ H^2 S A K } \eqsim M \sqrt{ H^2 \abs{\Ical_\epsilon} K } + M \sqrt{ H^2 \abs{\Ical_\epsilon^C} K}$, we see that an improvement in the collective regret bound comes from the contributions from subpar state-action pairs: the $M \sqrt{ H^2 |\Ical_\epsilon| K }$ term is reduced to $\sqrt{ M H^2 |\Ical_\epsilon| K }$, a factor of $\tilde{O}(\sqrt{\frac 1 M})$ improvement. Moreover, if  $\abr{\Ical_\epsilon^C} \ll SA$ and $M \gg 1$, \mteuler provides a regret bound of lower order than individual \strongeuler. 

We next present a gap-dependent upper bound on its collective regret. 

\begin{theorem}[Gap-dependent upper bound]
\label{thm:gap_dept_upper}
If $\cbr{\Mcal_p}_{p=1}^M$ are $\epsilon$-dissimilar, then \mteuler satisfies with probability $1-\delta$,
\begin{align*}
\Reg(K) 
\lesssim 
\ln(\frac{M S A K}{\delta}) & \vast(
\sum_{p \in [M]} \rbr{ 
\sum_{(s,a) \in Z_{p,\opt}} \frac{ H^3}{\gap_{p,\min}}
+
\sum_{(s,a) \in (\Ical_\epsilon \cup Z_{p,\opt})^C} \frac{H^3}{\gap_p(s,a)} }
+ \\
& \sum_{(s,a) \in \Ical_\epsilon} \frac{H^3}{\min_p \gap_p(s,a)}
\vast) 
+  \ln(\frac{M S A K}{\delta}) \cdot M H^4 S^2 A  \ln\frac{MSA}{\gap_{\min}},
\end{align*}
where we recall that $\gap_{p,\min} = \min_{(s,a): \gap_p(s,a) > 0} \gap_p(s,a)$, and $\gap_{\min} = \min_p \gap_{p,\min}$.
\end{theorem}

Comparing this regret bound with the regret bound obtained by the individual \strongeuler baseline, recall that by  summing over the regret guarantees of \strongeuler for all players $p \in [M]$, and taking a union bound over all $p$, individual \strongeuler guarantees a collective regret bound of
\begin{align*}
\Reg(K) \lesssim \ln (\frac{MSAK}{\delta}) & \vast( \sum_{p \in [M]} \rbr{ 
\sum_{(s,a) \in Z_{p,\opt}} \frac{H^3}{\gap_{p,\min}}
+
\sum_{(s,a) \in (\Ical_\epsilon \cup Z_{p,\opt})^C} \frac{H^3}{\gap_p(s,a)} 
} 
+ \\
& \sum_{(s,a) \in \Ical_\epsilon} \sum_{p \in [M]} \frac{H^3}{\gap_p(s,a)}  \vast) 
+ \ln(\frac{MSAK}{\delta}) \cdot M H^4 S^2 A  \ln\frac{SA}{\gap_{\min}},
\end{align*}
that holds with probability $1-\delta$.
We again focus on comparing the leading terms, i.e., the terms that have polynomial dependences on the suboptimality gaps in the above two bounds. 
It can be seen that an improvement in the regret bound by \mteuler comes from the contributions from the subpar state-action pairs: for each $(s,a) \in \Ical_\epsilon$, the regret bound is reduced from 
$\sum_{p \in [M]} \frac{H^3}{\gap_p(s,a)}$ to $\frac{H^3}{\min_{p} \gap_p(s,a)}$, a factor of $O(\frac1M)$ improvement.
Recent work of~\cite{xmd21} has shown that in the single-task setting, it is possible to replace $\sum_{(s,a) \in Z_{p,\opt}} \frac{H^3}{\gap_{p,\min}}$ with a sharper problem-dependent complexity term that depends on the multiplicity of optimal state-action pairs. 
We leave improving the guarantee of Theorem~\ref{thm:gap_dept_upper} in a similar manner as an interesting open problem.

Key to the proofs of Theorems~\ref{thm:gap_indept_upper} and~\ref{thm:gap_dept_upper}  is a new bound on the {\em surplus}~\cite{simchowitz2019non} of the value function estimates. 
Our new surplus bound is a minimum of two terms: one depends on the usual state-action visitation counts of player $p$, the other depends on the task dissimilarity parameter $\epsilon$ and the state-action visitation counts of all players. 
Detailed proofs can be found at Appendix~\ref{sec:proof-upper}. 

%Informally, surplus measures the amount of value function overestimation locally at every layer.
%Informally, surplus measures the amount of value function overestimation locally at every layer.

\subsection{Lower bounds}
To complement the above upper bounds, we now present gap-dependent and gap-independent regret lower bounds that also depend on our subpar state-action pair notion. Our lower bounds
are inspired by regret lower bounds for episodic RL~\cite{simchowitz2019non,dann2015sample} and multi-task bandits~\cite{wzsrc21}. 

\begin{theorem}[Gap-independent lower bound]
\label{thm:gap_ind_lb}
For any $A \ge 2$, $H \ge 2$, {$S \ge 4H$}, $K \ge SA$, $M \in \NN$, and $l, l^C \in \NN$ with $l + l^C = SA$ and $l \leq SA - 4(S + HA)$, there exists some $\epsilon$ that satisfies: for any algorithm $\Alg$, there exists an $\epsilon$-MPERL problem instance
with $S$ states, $A$ actions, $M$ players and an episode length of $H$
such that $\abr{\Ical_{\frac{\epsilon}{192H}}} \geq l$, and
\[
\EE \sbr{\Reg_\Alg(K)} \ge \Omega \rbr{ M \sqrt{H^2 l^C K } + \sqrt{ M H^2 l K }}.
\]
\end{theorem}

We also present a gap-dependent lower bound. Before that, we first formally define the notion of sublinear regret algorithms:
for any fixed $\epsilon$, we say that an algorithm $\Alg$ is a sublinear regret algorithm for the $\epsilon$-MPERL problem if there exists some $C > 0$ (that possibly depends on the state-action space, the number of players, and $\epsilon$) and $\alpha < 1$ such that for all $K$ and all $\epsilon$-MPERL environments, $\EE\sbr{\Reg_{\Alg}(K)} \le C K^\alpha$.

\begin{theorem}[Gap-dependent lower bound]
\label{thm:gap_dep_lb}
Fix $\epsilon \ge 0$. For any $S \in \NN$, $A \ge 2$, $H \ge 2$, $M \in \NN$, with $S \geq 2(H-1)$, let $S_1 = S - 2(H-1)$; and let $\cbr{\Delta_{s,a,p}}_{(s,a,p) \in [S_1] \times [A] \times [M]}$ be any set of values that satisfies: 
(1) each $\Delta_{s,a,p} \in [0,H/(48\sqrt{M})]$, (2) for every $(s,p) \in [S_1] \times [M]$, there exists at least one action $a \in [A]$ such that $\Delta_{s,a,p} = 0$, and (3) for every $(s,a) \in [S_1] \times [A]$ and $p,q \in [M]$, $\abs{\Delta_{s,a,p} - \Delta_{s,a,q}} \leq \epsilon/4$.
There exists an $\epsilon$-MPERL problem instance with $S$ states, $A$ actions, $M$ players and an episode length of $H$, such that
$\Scal_1 = [S_1]$, $\abr{\Scal_h} = 2$ for all $h \ge 2$, and
\[
\gap_p(s,a) = 
\Delta_{s,a,p}, \quad \forall (s,a,p) \in [S_1] \times [A] \times [M];
\]
for this problem instance, any sublinear regret algorithm $\Alg$ for the $\epsilon$-MPERL problem must satisfy:
\[
\EE \sbr{\Reg_\Alg(K)} \ge \Omega \rbr{ \ln K \rbr{ \sum_{p \in [M]} \sum_{\substack{(s,a) \in \Ical^C_{ \rbr{\epsilon/768H}}: \\ \gap_p(s,a) > 0}} \frac{H^2}{\gap_p(s,a)}
+
\sum_{(s,a) \in \Ical_{ \rbr{\epsilon/768H}}} \frac{H^2}{\min_p \gap_p(s,a)}}}.
\]
\end{theorem}
%have regret at least 

Comparing the lower bounds with \mteuler's regret upper bounds in Theorems~\ref{thm:gap_indept_upper} and~\ref{thm:gap_dept_upper}, 
we see that the upper and lower bounds nearly match for any constant $H$.
When $H$ is large,
a key difference between the upper and lower bounds is that the former are in terms of $\Ical_\epsilon$, while the latter are in terms of $\Ical_{\Theta(\frac\epsilon H)}$. We conjecture that our upper bounds can be improved by replacing $\Ical_\epsilon$ with $\Ical_{\Theta(\frac\epsilon H)}$---our analysis uses a clipping trick similar to~\cite{simchowitz2019non}, which may be the reason for a suboptimal dependence on $H$. We leave closing this gap as an open question.

\section{Related Work}
\label{sec:related_work}

\paragraph{Regret minimization for MDPs.} Our work belongs to the literature of regret minimization for MDPs, e.g.,~\cite{bartlett2009regal,jaksch2010near,dann2015sample,azar2017minimax,dann2017unifying,jin2018q,dann2019policy,zanette2019tighter, simchowitz2019non,zhang2020almost,yang2021q,xmd21}. In the episodic setting,~\cite{azar2017minimax,dann2019policy,zanette2019tighter,simchowitz2019non,zhang2020almost} achieve minimax $\sqrt{H^2 S A K}$ regret bounds for general stationary MDPs. Furthermore, the \euler algorithm~\cite{zanette2019tighter} achieves adaptive problem-dependent regret guarantees when the total reward within an episode is small or when the environmental norm of the MDP is small. ~\cite{simchowitz2019non} refines \euler, proposing \strongeuler that provides more fine-grained gap-dependent $O(\log K)$ regret guarantees. ~\cite{yang2021q,xmd21} show that the optimistic Q-learning algorithm~\cite{jin2018q} and its variants can also achieve gap-dependent logarithmic regret guarantees. Remarkably,~\cite{xmd21} achieves a regret bound that improves over that of~\cite{simchowitz2019non}, in that it replaces the dependence on the number of optimal state-action pairs with the number of non-unique state-action pairs.

%\cite[see][for surveys from different angles]
\paragraph{Transfer and lifelong learning for RL.} A considerable portion of related works concerns transfer learning for RL tasks (see~\cite{ts09,l12,zlz20} for surveys from different angles), and many studies investigate a batch setting: given some source tasks and target tasks, transfer learning agents have access to batch data collected for the source tasks (and sometimes for the target tasks as well).
In this setting, model-based approaches have been explored in e.g., \cite{tjs08}; theoretical guarantees for transfer of samples across tasks have been established in e.g., \cite{lr11,tsr19}.
Similarly, sequential transfer has been studied under the framework of lifelong RL in e.g., \cite{ty03,ajgkl18,gt19,ltm19}---in this setting, an agent faces a sequence of RL tasks and aims to take advantage of knowledge gained from previous tasks for better performance in future tasks; in particular, analyses on the sample complexity of transfer learning algorithms are presented in \cite{brunskill2013sample,lgb16} under the assumption that an upper bound on the total number of unique (and well-separated) RL tasks is known.
We note that, in contrast, we study an online setting in which no prior data are available and multiple RL tasks are learned concurrently by RL agents. 

\paragraph{Concurrent RL.} Data sharing between multiple RL agents that learn concurrently has also been investigated in the literature. For example, in \cite{k02,snbwm13,guo2015concurrent,dv18}, 
a group of agents interact in parallel with {\em identical} environments. Another setting is studied in \cite{guo2015concurrent}, in which agents solve different RL tasks (MDPs); however, similar to \cite{brunskill2013sample,lgb16}, it is assumed that there is a finite number of unique tasks, and different tasks are well-separated, i.e., there is a minimum gap. In this work, we assume that players face similar but not necessarily identical MDPs, and we do not assume a minimum gap.
\cite{hu2021near} study multi-task RL with linear function approximation with representation transfer, where it is assumed that the optimal value functions of all tasks are from a low dimensional linear subspace.  
Our setting and results are most similar to \cite{pp16} and \cite{dubey2021provably}. \cite{pp16} study concurrent exploration in similar MDPs with continuous states in the PAC setting; however, their PAC guarantee does not hold for target error rate arbitrarily close to zero; in contrast, our algorithm has a fall-back guarantee, in that it always has a sublinear regret. Concurrent RL from similar \textit{linear} MDPs has also been recently studied in \cite{dubey2021provably}: under the assumption of small heterogeneity between different MDPs (a setting very similar to ours), the provided regret guarantee involves a term that is linear in the number of episodes, whereas our algorithm in this paper always has a sublinear regret; concurrent RL under the assumption of large heterogeneity is also studied in that work, but additional contextual information is assumed to be available for the players to ensure a sublinear regret.

\paragraph{Other related topics and models.} In many multi-agent RL models~\cite{zyb19,oh19}, a set of learning agents interact with a common environment and have
shared global states; in particular, \cite{zylzb18} study
the setting with heterogeneous reward distributions, and provide convergence guarantees for two policy gradient-based algorithms. 
In contrast, in our setting, our learning agents interact with separate environments. 
Multi-agent bandits with similar, heterogeneous reward distributions are investigated in \cite{soaremulti,wzsrc21}; herein, we generalize their multi-task bandit problem setting to the episodic MDP setting.

%multi-armed 

\section{Conclusion and Future Directions}
\label{sec:conclusion}
In this paper, we generalize the multi-task bandit learning framework in \cite{wzsrc21} and formulate a multi-task concurrent RL problem, in which tasks are similar but not necessarily identical. We provide a provably efficient model-based algorithm that takes advantage of knowledge transfer between different tasks. Our instance-dependent regret upper and lower bounds formalize the intuition that subpar state-action pairs are amenable to information sharing among tasks.

There still remain gaps between our upper and lower bounds which can be closed by either a finer analysis or a better algorithm: first, the dependence on $\Ical_\epsilon$ in the upper bound does not match the dependence of $\Ical_{\Theta(\epsilon/H)}$ in the lower bound when $H$ is large; second, the gap-dependent upper bound has $O(H^3)$ dependence, whereas the gap-dependent lower bound only has $\Omega(H^2)$ dependence; third, the additive dependence on the number of optimal state-action pairs can potentially be removed by new algorithmic ideas~\cite{xmd21}. 

Furthermore, one major obstacle in deploying our algorithm in practice is its requirement for knowledge of $\epsilon$; an interesting avenue is to apply model selection strategies in bandits and RL to achieve adaptivity to unknown $\epsilon$. Another interesting future direction is to consider more general parameter transfer for online RL, for example, in the context of function approximation.

\section{Acknowledgements}
We thank Kamalika Chaudhuri for helpful initial discussions, and thank Akshay Krishnamurthy and Tongyi Cao for discussing the applicability of adaptive RL in metric spaces to the multitask RL problem studied in this paper. 
CZ acknowledges startup funding support from the University of Arizona.
ZW thanks the National Science Foundation under IIS 1915734 and CCF 1719133 for research support.

%%%%%%%%%%%%%%%%%%%%%%%%%%%%%%%%%%%%%%%%%%%%%%%%%%%%%%%%%%%%
\bibliographystyle{plain}
\bibliography{ref}

\newpage
%%%%%%%%%%%%%%%%%%%%%%%%%%%%%%%%%%%%%%%%%%%%%%%%%%%%%%%%%%%%
\hide{
\section*{Checklist}

%%% BEGIN INSTRUCTIONS %%%
%The checklist follows the references.  Please
%read the checklist guidelines carefully for information on how to answer these
%questions.  For each question, change the default \answerTODO{} to \answerYes{},
%\answerNo{}, or \answerNA{}.  You are strongly encouraged to include a {\bf
%justification to your answer}, either by referencing the appropriate section of
%your paper or providing a brief inline description.  For example:
%\begin{itemize}
%  \item Did you include the license to the code and datasets? \answerYes{See Section~\ref{gen_inst}.}
% \item Did you include the license to the code and datasets? \answerNo{The code and the data are proprietary.}
%  \item Did you include the license to the code and datasets? \answerNA{}
%\end{itemize}
%Please do not modify the questions and only use the provided macros for your
%answers.  Note that the Checklist section does not count towards the page
%limit.  In your paper, please delete this instructions block and only keep the
%Checklist section heading above along with the questions/answers below.
%%% END INSTRUCTIONS %%%

\begin{enumerate}

\item For all authors...
\begin{enumerate}
  \item Do the main claims made in the abstract and introduction accurately reflect the paper's contributions and scope?
    \answerYes{We study multi-task learning for regret minimization on episodic tabular MDPs, formulating a new notion of task similarity and design and analyze provably efficient algorithms in this setting. See the second half of Section~\ref{sec:intro} for a summary of our contributions.}
  \item Did you describe the limitations of your work?
    \answerYes{We have discussed the limitations of our work in Section~\ref{sec:conclusion}.}
  \item Did you discuss any potential negative societal impacts of your work?
    \answerYes{We have discussed potential negative societal impacts of our work in Section~\ref{sec:conclusion}.}
  \item Have you read the ethics review guidelines and ensured that your paper conforms to them?
    \answerYes{}
\end{enumerate}

\item If you are including theoretical results...
\begin{enumerate}
  \item Did you state the full set of assumptions of all theoretical results?
    \answerYes{See Sections~\ref{sec:prelims} and~\ref{sec:guarantees} for our assumptions of the theoretical results.}
	\item Did you include complete proofs of all theoretical results?
    \answerYes{See Appendix for the full proofs.}
\end{enumerate}

\item If you ran experiments...
\begin{enumerate}
  \item Did you include the code, data, and instructions needed to reproduce the main experimental results (either in the supplemental material or as a URL)?
    \answerNA{}
  \item Did you specify all the training details (e.g., data splits, hyperparameters, how they were chosen)?
    \answerNA{}
	\item Did you report error bars (e.g., with respect to the random seed after running experiments multiple times)?
    \answerNA{}
	\item Did you include the total amount of compute and the type of resources used (e.g., type of GPUs, internal cluster, or cloud provider)?
    \answerNA{}
\end{enumerate}

\item If you are using existing assets (e.g., code, data, models) or curating/releasing new assets...
\begin{enumerate}
  \item If your work uses existing assets, did you cite the creators?
    \answerNA{}
  \item Did you mention the license of the assets?
    \answerNA{}
  \item Did you include any new assets either in the supplemental material or as a URL?
    \answerNA{}
  \item Did you discuss whether and how consent was obtained from people whose data you're using/curating?
    \answerNA{}
  \item Did you discuss whether the data you are using/curating contains personally identifiable information or offensive content?
    \answerNA{}
\end{enumerate}

\item If you used crowdsourcing or conducted research with human subjects...
\begin{enumerate}
  \item Did you include the full text of instructions given to participants and screenshots, if applicable?
    \answerNA{}
  \item Did you describe any potential participant risks, with links to Institutional Review Board (IRB) approvals, if applicable?
    \answerNA{}
  \item Did you include the estimated hourly wage paid to participants and the total amount spent on participant compensation?
    \answerNA{}
\end{enumerate}

\end{enumerate}
}

%%%%%%%%%%%%%%%%%%%%%%%%%%%%%%%%%%%%%%%%%%%%%%%%%%%%%%%%%%%%
\newpage
\appendix

\section{Proofs of Lemmas~\ref{lem:q-dissimilar} and~\ref{lem:gap-const}}
\subsection{Proof of Lemma~\ref{lem:q-dissimilar}}
\newtheorem*{L1}{Lemma~\ref{lem:q-dissimilar}}
\label{lem:q-dissimilar-repeated}
\begin{L1}
If $(\Mcal_p)_{p=1}^M$ is $\epsilon$-dissimilar, then for every $p, q \in [M]$, and $(s,a) \in \Scal \times \Acal$, 
\[
\abr{ Q_p^\star(s,a) - Q_q^\star(s,a) } \leq 2 H \epsilon,
\]
consequently,
$
\abr{ \gap_p(s,a) - \gap_q(s,a) } \leq 4H\epsilon.
$
\end{L1}

\begin{proof}
For the first claim, we prove a stronger statement by backward induction on $h$, namely, for every $p, q \in [M]$, every $h \in [1,H+1]$, and $(s,a) \in \Scal_h \times \Acal$,
\[
\abr{ Q_p^\star(s,a) - Q_q^\star(s,a) } \leq 2 (H - h + 1) \epsilon.
\]

\paragraph{Base case:} For $h = H+1$, we have $Q_p^\star(s,a) = 0$ for every $(s,a) \in \Scal_h \times \Acal$, and $p \in [M]$. It follows trivially that $\abr{Q_p^\star(s,a) - Q_q^\star(s,a)} = 0 \le 2(H-h+1)\epsilon$. 

\paragraph{Inductive case:} Suppose by inductive hypothesis that for some $h \in [1,H]$ and, for every $(s,a) \in \Scal_{h+1} \times \Acal$ and $p,q \in [M]$, $\abr{Q_p^\star(s,a) - Q_q^\star(s,a)} \le 2(H - h)\epsilon$. 

We first prove the following auxiliary statement: for every $s \in \Scal_{h+1}$ and $p,q \in [M]$,
\begin{align}
    \abs{V_p^\star(s) - V_q^\star(s)} \le  2(H - h)\epsilon.
    \label{eqn:inductive_V_eps}
\end{align}
Let $a_p = \argmax_{a \in \Acal} Q_p^\star(s,a)$ and $a_q = \argmax_{a \in \Acal} Q_q^\star(s,a)$. The above auxiliary statement can be easily proven by contradiction: without loss of generality, suppose that $V_p^\star(s) - V_q^\star(s) = Q_p^\star(s,a_p) -  Q_q^\star(s,a_q) >  2(H - h)\epsilon$. Since $Q_q^\star(s,a_p) \ge Q_p^\star(s,a_p) - 2(H-h)\epsilon$, it follows that $ Q_q^\star(s,a_p) > Q_q^\star(s,a_q)$, which contradicts the fact that $a_q = \argmax_{a \in \Acal} Q_q^\star(s,a)$.

We now return to the inductive proof, and we show that given the inductive hypothesis, 
for every $(s,a) \in \Scal_{h} \times \Acal$ and $p,q \in [M]$,
\begin{align*}
    & \abr{Q_p^\star(s,a) - Q_q^\star(s,a)} \\
    \le & \abr{R_p(s,a) - R_q(s,a)} + \abr{\sum_{s' \in \Scal_{h+1}} \sbr{\PP_p(s' \mid s,a) V_p^\star(s') - \PP_q(s' \mid s,a) V_q^\star(s')}} \\
    \le & \epsilon + \abr{\sum_{s' \in \Scal_{h+1}} \sbr{\PP_p(s' \mid s,a) V_p^\star(s') - \PP_q(s' \mid s,a) V_p^\star(s')}} + \abr{\sum_{s' \in \Scal_{h+1}} \PP_q(s' \mid s,a) \rbr{V_p^\star(s') - V_q^\star(s')}} \\
    \le & \epsilon + \|\PP_p(\cdot \mid s,a) - \PP_q(\cdot \mid s,a))\|_1 \rbr{\max_{s' \in \Scal_{h+1}} \abs{V_p^\star(s')}} + \|\PP_q(\cdot \mid s,a)\|_1 \rbr{ \max_{s' \in \Scal_{h+1}} \abs{V_p^\star(s') - V_q^\star(s')}} \\
    \le & \epsilon + \frac{\epsilon}{H}\cdot H + 2(H - h) \epsilon  \\
    = & 2(H-h+1)\epsilon,
\end{align*}
where the first inequality follows from Eq.~\eqref{eqn:bellman-optimality-q-p} and the triangle inequality; the second inequality follows from Definition~\ref{def:dissimilar} and the triangle inequality; the third inequality follows from H\"older's inequality; and the fourth inequality uses Definition~\ref{def:dissimilar} and Eq.~\eqref{eqn:inductive_V_eps}.

For the second claim, we note that from the first claim, we have for any $p,q,s$,
\[
\abr{ V_p^\star(s) - V_q^\star(s) } = \abr{ \max_{a \in \Acal} Q_p^\star(s,a) - \max_{a \in \Acal} Q_p^\star(s,a) } \leq 2 H \epsilon,
\]
therefore, for any $p,q,s,a$,
\[
\abr{ \gap_p(s,a) - \gap_q(s,a) } 
\leq 
\abr{ V_p^\star(s) - V_q^\star(s) }
+
\abr{ Q_p^\star(s,a) - Q_p^\star(s,a) }
\leq
4H\epsilon.
\qedhere
\]
\end{proof}

\subsection{Proof of Lemma~\ref{lem:gap-const}}

\newtheorem*{L2}{Lemma~\ref{lem:gap-const}}
\begin{L2}
For any $(s,a) \in \Ical_\epsilon$, we have that: (1) for all $p \in [M]$, $(s,a) \notin Z_{p,\opt}$, where we recall that $Z_{p,\opt} = \cbr{ (s,a): \gap_p(s,a) = 0}$ is the set of optimal state-action pairs with respect to $p$; (2) for all $p,q \in [M]$, $\gap_p(s,a) \geq \frac12\gap_q(s,a)$.
\end{L2}

\begin{proof}
For any $(s,a) \in \Ical_\epsilon$, there exists some $p_0$ such that $\gap_{p_0}(s,a) > 96H\epsilon$.
From Lemma~\ref{lem:q-dissimilar} we know that
$
\abr{ \gap_p(s,a) - \gap_{p_0}(s,a) } \leq 4H\epsilon.
$
Therefore, for all $p$,
\[
 \gap_p(s,a) 
 \geq \gap_{p_0}(s,a) - 4H\epsilon
 >
 92H\epsilon
 \geq 
 0.
\]
This proves the first item.

For the second item, for all $p,q \in [M]$, 
\[
\frac{\gap_p(s,a)}{\gap_q(s,a)}
= 
\frac{\gap_q(s,a) - 4H\epsilon}{\gap_q(s,a)}
\geq 
1 - \frac{ 4H\epsilon}{\gap_q(s,a)}
\geq 
1 - \frac{4}{92} \geq \frac12.
\qedhere
\]
\end{proof}

\section{Additional Definitions Used in the Proofs}
\label{sec:additional-notations}

In this section, we define a few useful notations that will be used in our proofs. 
For state-action pair $(s,a) \in \Scal \times \Acal$, player $p \in [M]$, episode $k \in [K]$:
\begin{enumerate}

\item Define $n^k(s,a)$ (resp. $n_p^k(s,a)$, $\hat{\PP}^k$, $\hat{\PP}_p^k$, $\hat{R}^k$, $\hat{R}_p^k$) to be the value of $n(s,a)$ (resp. $n_p(s,a)$, $\hat{\PP}$, $\hat{\PP}_p$, $\hat{R}$, $\hat{R}_p$) at the {\em beginning} of episode $k$ of \mteuler.

\item Denote by $\qub_p^k$ (resp. $\qlb_p^k, \vub_p^k, \vlb_p^k$, $\bindp_p^k(s,a)$, $\baggp_p^k(s,a)$) the values of $\qub_p$ (resp. $\qlb_p, \vub_p, \vlb_p$, $\bindp_p(s,a)$, $\baggp_p(s,a)$) right after \mteuler finishes its optimistic value iteration (line~\ref{line:vi-end}) at episode $k$. 

\item Define the {\em surplus}~\cite{simchowitz2019non} (also known as the Bellman error) of $(s,a)$ at episode $k$ and player $p$ as:
\[
E_p^k(s,a) := \overline{Q}_p^k(s,a) - R_p(s,a) - (\PP_p \overline{V}_p^k) (s,a).
\]

\item Define $w_p^k(s,a) := \frac{n_p^k(s,a)}{n^k(s,a)}$ be the proportion of player $p$ on $(s,a)$ at the beginning of episode $k$; this induces $(s,a)$'s {\em mixture expected reward}:
\[ \bar{R}^k(s,a) := \sum_{q=1}^M w_q^k(s,a) R_q(s,a), 
\] 
and {\em mixture transition probability}:
\[ 
\bar{\PP}^k(\cdot \mid s,a) := \sum_{q=1}^M w_q^k(s,a) \PP_q(\cdot \mid s,a).
\]

\item Define $\rho_p^k(s,a) := \PP( (s_h, a_h) = (s, a) \mid \pi^k(p), \Mcal_p)$ to be the occupancy measure of $\pi^k(p)$ over $\Mcal_p$ on $(s,a)$, where $h \in [H]$ is the layer $s$ is in (so that $s \in \Scal_h$). 
It can be seen that 
$\rho_p^k$, when restricted to $\Scal_h \times \Acal$, is a probability distribution on this set.

Define $\rho^k(s,a) := \sum_{p=1}^M \rho_p^k(s,a)$;
it can be seen that $\rho^k(s,a) \in [0, M]$. 
Define $\bar{n}_p^k(s,a) := \sum_{j=1}^k \rho_p^j(s,a)$, and $\bar{n}^k(s,a) := \sum_{j=1}^k \rho^j(s,a)$.\footnote{These are the cumulative occupancy measures up to episode $k$, inclusively; this is in contrast with the definition of $n^k(s,a)$ and $n_p^k(s,a)$, which do not count the trajectories observed at episode $k$.}

\item Define $N^{k}(s) := \sum_{a \in \Acal} n^{k}(s,a)$ and $N_p^{k}(s) := \sum_{a \in \Acal} n_p^{k}(s,a)$ to be the total number of encounters of state $s$ by all players, and by player $p$ only, respectively, at the beginning of episode $k$.

\item Define $N_1 \eqsim M \ln(\frac{S A K}{\delta})$, and $N_2 \eqsim \ln(\frac{M S A K}{\delta})$; define $\tau(s,a) := \min\cbr{k: \bar{n}^k(s,a) \geq N_1}$, and $\tau_p(s,a) := \min\cbr{k: \bar{n}_p^k(s,a) \geq N_2}$. With high probability, so long as $k \geq \tau(s,a)$ (resp. $k \geq \tau_p(s,a)$), $n^k(s,a)$ and $\bar{n}^k(s,a)$ (resp. $n_p^k(s,a)$ and $\bar{n}_p^k(s,a)$) are within a constant factor of each other; see Lemma~\ref{lem:esamp}.

\item Define
$
\check{\gap}_p(s,a)
:=
\frac{\gap_p(s,a)}{4H} \vee \frac{\gap_{p,\min}}{4H}
$; recall the definitions of $\gap_p(s,a)$ and  $\gap_{p,\min}$ in Section~\ref{sec:prelims}.
\end{enumerate}

Define
$
\Reg(K, p) := \sum_{k=1}^K \del{ V_{0,p}^{\star} - V_{0,p}^{\pi^k(p)} } 
$ as player $p$'s contribution to the collective regret; in this notation,
$\Reg(K) = \sum_{p=1}^M \Reg(K, p)
$.

Define the clipping function $\clip(\alpha, \Delta) := \alpha \one(\alpha \geq \Delta)$. 

We also adopt the following conventions in our proofs:
\begin{enumerate}

\item As $\epsilon$-dissimilarity with $\epsilon > 2H$ does not impose any constraints on $\cbr{\Mcal_p}_{p=1}^M$ (recall Definition~\ref{def:dissimilar}), throughout the proof, we only focus on the regime that $\epsilon \leq 2H$.

\item We will use $\pi^k(p)$ and $\pi_p^k$ interchangeably. To avoid notational clutter, we will also sometimes slightly abuse notation and use $V_{p,h}^{\pi^k}$, $V_{p}^{\pi^k}$ to denote $V_{p,h}^{\pi^k(p)}$, $V_{p}^{\pi^k(p)}$, respectively.  

\end{enumerate}

\section{Proof of the Upper Bounds}
\label{sec:proof-upper}
  
  \paragraph{Proof outline.} This section establishes the regret guarantees of \mteuler (Theorems~\ref{thm:gap_indept_upper} and~\ref{thm:gap_dept_upper}).
  The proof follows a similar outline as \strongeuler's analysis~\cite{simchowitz2019non}, with important modifications tailored to the multitask setting. The proof has the following structure:
  \begin{enumerate}
      \item Subsection~\ref{subsec:clean} defines a ``clean'' event $E$ that we show happens with probability $1-\delta$. When $E$ happens, the observed samples are representative enough so that standard concentration inequalities apply. This will serve as the basis of our subsequent arguments. 
      
      \item Subsection~\ref{subsec:value-bounds-valid} shows that when $E$ happens, the value function upper and lower bounds are valid; furthermore,  \mteuler satisfies strong optimism~\cite{simchowitz2019non}, in that all players' surpluses are always nonnegative for all state-action pairs at all time steps.
      
      \item Subsection~\ref{subsec:surplus-bound} establishes a distribution-dependent upper bound on \mteuler's surpluses when $E$ happens, which is key to our regret theorems. In comparison with  \strongeuler~\cite{simchowitz2019non} in the single task setting, \mteuler exploits inter-task similarity, so that its surpluses on state-action pair $(s,a)$ for player $p$ are further controlled by a new term that depends on the dissimilarity parameter $\epsilon$, along with $n^k(s,a)$, the total visitation counts of $(s,a)$ {\em by all players}.
      
      \item Subsection~\ref{subsec:conclude-reg-bounds} uses the strong optimism property and the surplus bounds established in the previous two subsections to conclude our final gap-independent and gap-dependent regret guarantees, via the clipping lemma of~\cite{simchowitz2019non} (see also Lemma~\ref{lem:clipping-main}).
      
      \item Finally, Subsection~\ref{subsec:misc-lemmas} collects miscellaneous technical lemmas used in the proofs.
  \end{enumerate}

  \subsection{A clean event}
  \label{subsec:clean}
  
  Below we define a ``clean'' event $E$ in which all concentration bounds used in the analysis hold, which we will show happens with high probability. 
  Specifically, we will define $E = E_{\indiv} \cap E_{\aggre} \cap E_{\samp}$, where $E_{\indiv}, E_{\aggre}, E_{\samp}$ are defined respectively below.
  
  In subsequent definitions of events, we will abbreviate $\forall k \in [K], h \in [H], p \in [M], s \in \Scal_h, a \in \Acal, s' \in \Scal_{h+1}$ as $\forall k,h,p,s,a,s'$. Also,  recall that $L(n) \eqsim \ln (\frac{MSAn}{\delta})$.
  
  Define event $E_{\indiv}$ as: 
  \begin{align}
    & E_{\indiv} = E_{\indiv, \rw} \cap E_{\indiv, \val} \cap E_{\indiv, \prob} \cap E_{\indiv, \varia}, \\
    & E_{\indiv, \rw} = \cbr{\forall k,h,p,s,a \centerdot \abs{ \hat{R}_p^k(s,a) - R_p(s,a) }
  \leq 
  \sqrt{\frac{L(n^k(s,a))}{2 n^k(s,a)}} }, \\
    & E_{\indiv, \val} = \cbr{\forall k,h,p,s,a \centerdot \abs{ 
    (\hat{\PP}_p^k V_p^\star - \PP_p V_p^\star)(s,a) }
    \leq 4\sqrt{ \frac{ \VV_{\PP_p(\cdot \mid s,a)}[V_p^\star] L(n_p^k(s,a))}{n_p^k(s,a)} } + \frac{2 H L(n_p^k(s,a))}{n_p^k(s,a)} },
    \label{eqn:ind-optval}
    \\
    & E_{\indiv, \prob} = \cbr{\forall k,h,p,s,a,s' \centerdot \abr{ (\hat{\PP}_p^k - \PP_p)  (s' \mid s, a) }
  \leq 
  4 \sqrt{\frac{L(n_p^k(s,a)) \cdot \PP_p(s' \mid s, a)}{n_p^k(s,a)}} + \frac{2 L(n_p^k(s,a))}{n_p^k(s,a)}},
      \label{eqn:ind-transition}
      \\
    & E_{\indiv, \varia} =  \Bigg\{ \forall k,h,p,s,a \centerdot \abr{ \frac{1}{n_p^k(s,a)} \sum_{i=1}^{n_p^k(s,a)} 
    (V_p^\star((s_i^p)') - (\PP_p V_p^\star) (s,a) )^2 -  
    \VV_{\PP_p(\cdot \mid s,a)}[V_p^\star] },
    \label{eqn:ind-emp-variance}
    \\
    & \qquad \qquad \qquad \qquad \leq 4 \sqrt{ \frac{ H^2  \VV_{\PP_p(\cdot \mid s,a)}[V_p^\star] L(n_p^k(s,a))}{n_p^k(s,a)}} + \frac{2H^2 L(n_p^k(s,a))}{n_p^k(s,a)} \; \nonumber
  \Bigg\}, 
  \end{align}
  where in Equation~\eqref{eqn:ind-emp-variance}, $(s_i^p)'$ denotes the next state player $p$ transitions to, for the $i$-th episode it experiences $(s,a)$. $E_{\indiv}$ captures the concentration behavior of each player's individual model estimates. %
  
\begin{lemma}
$\PP(E_{\indiv}) \geq 1 - \frac \delta 3$. 
\label{lem:eind}
\end{lemma}
\begin{proof}
The proof follows a similar reasoning as the proof of e.g.,~\cite[Proposition F.9]{simchowitz2019non} using Freedman's Inequality. We would like to show that each of $E_{\indiv, \rw}, E_{\indiv, \val}, E_{\indiv, \prob}, E_{\indiv, \varia}$ happens with probability $1-\frac \delta {12}$, which would give the lemma statement by a union bound.
For brevity, we only show that $\PP(E_{\indiv, \varia}) \geq 1-\frac{\delta}{12}$, and the other probability statements follow from a similar reasoning.

Fix $h \in [H]$, $(s,a) \in \Scal_h \times \Acal$, and $p \in [M]$. 
We will show
\begin{align}
\begin{split}
\PP \vast( \exists k \in [K] \centerdot & \abr{ \frac{1}{n_p^k(s,a)}  \sum_{i=1}^{n_p^k(s,a)} 
    (V_p^\star((s_i^p)') - (\PP_p V_p^\star) (s,a) )^2 -  
    \VV_{\PP_p(\cdot \mid s,a)}[V_p^\star] }
    \\
    & \geq 4 \sqrt{ \frac{ H^2  \VV_{\PP_p(\cdot \mid s,a)}[V_p^\star] L(n_p^k(s,a))}{n_p^k(s,a)}} + \frac{2 H^2 L(n_p^k(s,a))}{n_p^k(s,a)} \; 
  \vast) \leq \frac{\delta}{12 M S A}. 
\end{split}
\label{eqn:ind-var}
\end{align}

For every $j \in \NN_+$, define stopping time $k_j$ as  the $j$-th episode when $(s,a)$ is experienced by player $p$, if such episode exists; otherwise, $k_j$ is defined as $\infty$.
it suffices to show that
\begin{align}
\begin{split}
\PP \vast( \exists j \in \NN_+ \centerdot \;\; & k_j < \infty \wedge \abr{ \frac{1}{j}  \sum_{i=1}^{j} 
    (V_p^\star((s_i^p)') - (\PP_p V_p^\star) (s,a) )^2 -  
    \VV_{\PP_p(\cdot \mid s,a)}[V_p^\star] }
    \\
    & \qquad \qquad \geq 4 \sqrt{ \frac{ H^2  \VV_{\PP_p(\cdot \mid s,a)}[V_p^\star] L(j)}{j}} + \frac{2 H^2 L(j)}{j} \; 
  \vast) \leq \frac{\delta}{12 M S A}. 
\end{split}
\label{eqn:ind-var-2}
\end{align}

%Chicheng: The following argument works, but is a bit complicated because of the usage of stopped sigma algebra. I decided to use a simpler proof, which avoids the mentioning of stopped sigma algebra. 

For every $k \in \NN_+$, Define $\Fcal_{k-1}$ as the $\sigma$-field generated by all players' observations up to episode $k-1$, along with all players' observations at episode $k$ up to them taking action at step $h$. Define 
\[ 
X_k := I( (s_{h,p}^k, a_{h,p}^k) = (s, a) ) \cdot \rbr{ (V_p^\star(s_{h+1,p}^k) - (\PP_p V_p^\star) (s,a) )^2 -  
\VV_{\PP_p(\cdot \mid s,a)}[V_p^\star] },
\]
it can be seen that $X_k$ is $\Fcal_k$-measurable, and 
$\EE\sbr{X_k \mid \Fcal_{k-1}} = 0$, i.e. $\cbr{X_k}_{k=1}^\infty$ is a martingale difference sequence adapted to $\cbr{\Fcal_k}_{k=1}^\infty$. 
In addition, 
\begin{align*}
\EE\sbr{ X_k^2 \mid \Fcal_{k-1} }
\leq &
\EE\sbr{  I( (s_{h,p}^k, a_{h,p}^k) = (s, a) ) (V_p^\star(s_{h+1,p}^k) - (\PP_p V_p^\star) (s,a) )^4 \mid \Fcal_{k-1} } \\
\leq &
H^2 \cdot I( (s_{h,p}^k, a_{h,p}^k) =  (s, a) ) \cdot \VV_{\PP_p(\cdot \mid s,a)}[V_p^\star]
=: U_k;
\end{align*}

Note that $\abs{X_k / H^2} \leq 1$; 
by~\cite[Corollary 1.4]{freedman1975tail} applied to $\cbr{X_k / H^2}_{k=1}^\infty$, for any $\lambda \geq 0$, \[ \cbr{ Y_k(\lambda) = \exp\del{\lambda (\sum_{i=1}^k \frac{X_i}{H^2} ) - \del{ (e^\lambda - \lambda - 1) \sum_{i=1}^k \frac{U_i}{H^4}  } } }_{k=0}^\infty
\]
is 
a nonnegative supermartingale. Applying optional sampling theorem on $Y_k(\lambda)$ and  stopping time $k_j$, we get $\EE\sbr{ Y_{k_j}(\lambda) I(k_j < \infty) } \leq \EE\sbr{ Y_0(\lambda) } = 1$. 
As a result, for any fixed thresholds $b, v \geq 0$ \cite[see][Theorem 1.6]{freedman1975tail}, 
\begin{align*}
& \PP\del{ \sum_{i=1}^{k_j} X_i \geq b \wedge \sum_{i=1}^{k_j} U_i \leq v \wedge k_j < \infty } \\
= & 
\PP\del{ \sum_{i=1}^{k_j} \frac{X_i}{H^2} \geq \frac{b}{H^2} \wedge \sum_{i=1}^{k_j} \frac{ U_i}{H^4} \leq \frac{v}{H^4} \wedge k_j < \infty } \\
\leq &
\exp\del{ -\frac{(b/H^2)^2}{2 (v/H^4) + 2 (b/H^2)} }
=
\exp\del{ -\frac{b^2}{2v + 2 H^2 b} }.
\end{align*}

Note that if $k_j < \infty$,  $\sum_{i=1}^{k_j} X_i = \sum_{i=1}^{j} (V_p^\star((s_i^p)') - (\PP_p V_p^\star) (s,a) )^2 -  
    \VV_{\PP_p(\cdot \mid s,a)}[V_p^\star]$, 
and $\sum_{i=1}^{k_j} U_i = j \cdot H^2 \cdot \VV_{\PP_p(\cdot \mid s,a)}[V_p^\star]$, and the above inequality can be rewritten as: for any $b, v \geq 0$,
\begin{align*}
& \PP\del{ \sum_{i=1}^{j} (V_p^\star((s_i^p)') - (\PP_p V_p^\star) (s,a) )^2 -  
    \VV_{\PP_p(\cdot \mid s,a)}[V_p^\star] \geq b \wedge j \cdot H^2 \cdot \VV_{\PP_p(\cdot \mid s,a)}[V_p^\star] \leq v \wedge k_j < \infty } \\
\leq &
\exp\del{ -\frac{b^2}{2v + 2 H^2 b} }.
\end{align*}

\hide{
Define $\Gcal_j$ as the stopped $\sigma$-algebra generated by all observations up to stopping time $k_j$. We have that $\cbr{\Gcal_j}_{j=0}^{\infty}$ is a filtration.
It can be seen that the sequence $\cbr{ X_j := \rbr{(V_p^\star((s_j^p)') - (\PP_p V_p^\star) (s,a) )^2 -  
\VV_{\PP_p(\cdot \mid s,a)}[V_p^\star]} \one(k_j < \infty)}_{j=1}^\infty$ is a martingale difference sequence adapted to $\cbr{\Gcal_j}_{j=0}^{\infty}$;
%\zhi{mark} \chicheng{See the lemma after this proof. I would rather this lemma not appear in this paper, this is just for your / our understanding..} 
in addition, for every $j$, $\abs{ X_j} \leq H^2$, and $\EE\sbr{ X_j^2 \mid \Gcal_{j-1} } \leq \EE\sbr{ (V_p^\star((s_j^p)') - (\PP_p V_p^\star) (s,a) )^4 \one(k_j < \infty) \mid \Gcal_{j-1} } \leq H^2 \VV_{\PP_p(\cdot \mid s,a)}[V_p^\star]$. This implies that for any $\lambda \geq 0$, \[ \cbr{ Y_j(\lambda) = \exp\del{\lambda \frac{1}{H^2} (\sum_{i=1}^j X_i) - \del{ (e^\lambda - \lambda - 1) \frac{j}{H^2} \VV_{\PP_p(\cdot \mid s,a)}[V_p^\star] } } }_{j=0}^\infty
\]
is 
a nonnegative supermartingale~\cite{freedman1975tail}, and $\EE \sbr{ Y_{j}(\lambda) \one(k_j < \infty) } \leq \EE \sbr{ Y_{j}(\lambda) } \leq \EE \sbr{Y_0(\lambda)} = 1$. 
As a result, for any fixed thresholds $a, v \geq 0$ \cite[see][Theorem 1.6]{freedman1975tail}, 
\[
\PP\del{ \sum_{i=1}^j X_i \geq a \wedge \sum_{i=1}^j H^2 \VV_{\PP_p(\cdot \mid s,a)}[V_p^\star] \leq v \wedge k_j < \infty }
\leq 
\exp\del{ -\frac{a^2}{2 v + 2 a H^2 / 3 } }
\]
}

Now, by the doubling argument of~\cite[Lemma 2]{bartlett2008high} (observe that $j \cdot H^2 \cdot \VV_{\PP_p(\cdot \mid s,a)}[V_p^\star] \in [0, H^4 j]$), we have that for all $j \in \NN_+$:
\begin{align*}
\PP \vast(  & k_j < \infty \wedge \abr{ \frac{1}{j} \sum_{i=1}^{j} 
    (V_p^\star((s_i^p)') - (\PP_p V_p^\star) (s,a) )^2 -  
    \VV_{\PP_p(\cdot \mid s,a)}[V_p^\star] }
    \\
    & \geq 4 \sqrt{ \frac{ H^2  \VV_{\PP_p(\cdot \mid s,a)}[V_p^\star] L(j)}{j}} + \frac{2H^2 L(j)}{j} \; 
  \vast) \leq \ln(4j) \cdot \frac{\delta}{48 j^2 M S A}.
\end{align*}
A union bound over all $j \in \NN_+$ yields Equation~\eqref{eqn:ind-var-2}.
\end{proof}

Define event $E_{\aggre}$ as: 
%\zhi{quantifiers below can be a little confusing.}
%\chicheng{I forgot why, but I think the current version is fine. We have clarified the abbreviation at the beginning of this subsection.}
\begin{align}
    & E_{\aggre} =  E_{\aggre, \rw} \cap E_{\aggre, \val} \cap E_{\aggre, \prob} \cap E_{\aggre, \varia}, \\
    & E_{\aggre, \rw} = \cbr{ \forall k,h,s,a \centerdot \abs{ \hat{R}^k(s,a) - \bar{R}^k(s,a) }
  \leq 
  \sqrt{\frac{L(n^k(s,a))}{2 n^k(s,a)}} },
     \\
    & E_{\aggre, \val} =  \Bigg\{ \forall k,h,p,s,a \centerdot  \abs{ (\hat{\PP}^k V_p^\star - \bar{\PP}^k V_p^\star)  (s,a)  }, \\
    & \qquad \qquad \qquad \leq 4 \sqrt{  \frac{ \del{ \sum_{q=1}^M w_q^k(s,a) \VV_{\PP_q(\cdot \mid s,a)}[V_p^\star] }
    L(n^k(s,a)) }{n^k(s,a)} }  + \frac{2H L(n^k(s,a))}{n^k(s,a)} \Bigg\},
    \label{eqn:agg-optval}
    \\
    & E_{\aggre, \prob} = \cbr{ \forall k,h,s,a,s' \centerdot  \abr{ (\hat{\PP}^k - \bar{\PP}^k)(s' \mid s, a) }
  \leq 
  4 \sqrt{\frac{\bar{\PP}^k(s' \mid s, a) \cdot L(n^k(s,a))}{n^k(s,a)}} + \frac{2L(n^k(s,a))}{n^k(s,a)} },
  \label{eqn:agg-transition}
  \\  
     & E_{\aggre, \varia} = \Bigg\{ \forall k,h,p,s,a \centerdot  \abr{ \frac{1}{n^k(s,a)} \sum_{i=1}^{n^k(s,a)} 
    (V_p^\star(s_i') - (\PP_{p_i} V_p^\star) (s,a) )^2 -  
    \sum_{q=1}^M w_q^k(s,a) \VV_{\PP_q(\cdot \mid s,a)}[V_p^\star] }, \label{eqn:agg-emp-variance} \\
    & \qquad \qquad \qquad \leq 4 \sqrt{ \frac{ H^2  \del{ \sum_{q=1}^M w_q^k(s,a) \VV_{\PP_q(\cdot \mid s,a) }[V_p^\star]} L(n^k(s,a))}{n^k(s,a)}} + \frac{2H^2 L(n^k(s,a))}{n^k(s,a)}  \nonumber
    \Bigg\},
\end{align}
 where in Equation~\eqref{eqn:agg-emp-variance}, $s_i'$ and $p_i$ denote the next state and the player index for the $i$-th time some player experiences $(s,a)$, respectively, where within an episode, we order the experiences of the players by their indices from $1$ to $M$. %
 $E_{\aggre}$ captures the concentration behavior of the aggregate model estimates. 

\begin{lemma}
$\PP(E_{\aggre}) \geq 1 - \frac \delta 3$. 
\label{lem:eagg}
\end{lemma}
\begin{proof}
The proof follows a similar reasoning as the proof of e.g., \cite[Proposition F.9]{simchowitz2019non} using Freedman's Inequality. 
We would like to show that each of $E_{\aggre, \rw}, E_{\aggre, \val}, E_{\aggre, \prob}, E_{\aggre, \varia}$ happen with probability $1-\frac \delta {12}$, which would give the lemma statement by a union bound.
For brevity, we show that $\PP(E_{\aggre, \varia}) \geq 1-\frac{\delta}{12}$, and the other probability statements follow from a similar reasoning.

%denote by $p_i$ the identity of the player when $(s,a)$ is experienced for the $i$-th time for some player.
Fix $h \in [H]$, $(s,a) \in \Scal_h \times \Acal$ and $p \in [M]$.
It suffices to show that
\begin{align}
\begin{split}
\PP \vast( \exists k \in [K] \centerdot & \abr{ \frac{1}{n^k(s,a)} \sum_{i=1}^{n^k(s,a)} 
    \del{
    (V_p^\star(s_i') - (\PP_{p_i} V_p^\star) (s,a) )^2 -  
    \VV_{\PP_{p_i}(\cdot \mid s,a)}[V_p^\star] } } 
    \\
    \geq 4 & \sqrt{ \frac{ H^2  \del{ \sum_{i=1}^{n^k(s,a)} \VV_{\PP_{p_i}(\cdot \mid s,a) }[V_p^\star]} L(n^k(s,a))}{(n^k(s,a))^2}} + \frac{2 H^2 L(n^k(s,a))}{n^k(s,a)} \vast) \leq \frac{\delta}{12 M S A},
\end{split}
\label{eqn:agg-var}
\end{align}
because $\frac{1}{n^k(s,a)}\sum_{i=1}^{n^k(s,a)} \VV_{\PP_{p_i}(\cdot \mid s,a) }[V_p^\star] =  \sum_{q=1}^M w_q^k(s,a) \VV_{\PP_{q}(\cdot \mid s,a) }[V_p^\star]$.

%\zhi{TODO: further explain what micro-episode means.}
%\preceq (2,1) \preceq \ldots \preceq (2,M)
Define a {\em micro-episode} as an (episode, player) pair; we order them lexicographically, i.e. $(1,1) \preceq \ldots \preceq (1,M)  \preceq \ldots \preceq (K,1) \preceq \ldots \preceq (K,M)$. 
For micro-episode $(k,p)$, denote its {\em index} as $l = (k-1)M + p$; it can be easily seen that the ordering of micro-episodes' indices is consistent with their lexical ordering. 
For every $j \in \NN_+$, define stopping time $l_j \in \NN_+$ as follows: it is the index of the $j$-th micro-episode when $(s,a)$ is experienced by some player, if such micro-episode exists; and $l_j$ is defined to be $\infty$ otherwise. With this notation, it suffices to show:
\begin{align}
\begin{split}
\PP \vast( \exists j \in \NN_+ \centerdot \;\; & l_j < \infty \wedge \abr{ \frac{1}{j} \sum_{i=1}^{j} 
    \del{
    (V_p^\star(s_i') - (\PP_{p_i} V_p^\star) (s,a) )^2 -  
    \VV_{\PP_{p_i}(\cdot \mid s,a)}[V_p^\star] } } 
    \\
    &  \geq 4 \sqrt{ \frac{ H^2  \del{ \sum_{i=1}^{j} \VV_{\PP_{p_i}(\cdot \mid s,a) }[V_p^\star]} L(j)}{j^2}} + \frac{2 H^2 L(j)}{j} \vast) \leq \frac{\delta}{12 M S A}, 
\end{split}
\label{eqn:agg-var-2}
\end{align}

%Define $\Gcal_j$ as the stopped $\sigma$-algebra generated by all observations up to micro-episode $k_j$, which is a stopping time. 
%We have that $\cbr{\Gcal_j}_{j=0}^{\infty}$ is a filtration. It can be seen that $\cbr{ X_j := \rbr{(V_p^\star(s_j') - (\PP_{p_j} V_p^\star) (s,a) )^2 -  
%\VV_{\PP_{p_j}(\cdot \mid s,a)}[V_p^\star]} \one(k_j < \infty)}_{j=1}^\infty$ is a martingale difference sequence adapted to $\cbr{\Gcal_j}_{j=0}^{\infty}$; in addition, for every $j$, $\abs{ X_j } \leq H^2$, and $\EE\sbr{ X_j^2 \mid \Gcal_{j-1} } \leq \EE\sbr{ (V_p^\star(s_j') - (\PP_{p_j} V_p^\star) (s,a) )^4  \one(k_j < \infty) \mid \Gcal_{j-1} } \leq H^2 \VV_{\PP_{p_j}(\cdot \mid s,a)}[V_p^\star]$. 
For every $l \in \NN_+$, Define $\Fcal_{l-1}$ as the $\sigma$-field generated by all players' observations up to micro-episode $l-1$, along with micro-episode $l$'s corresponding player (player index $((l-1) \; \mathrm{mod} \; M)+1$)'s observations up to them taking action at step $h$. Define 
\[ 
X_l := I( (s_{h,p}^k, a_{h,p}^k) = (s, a) ) \cdot \rbr{ (V_p^\star(s_{h+1,p}^k) - (\PP_p V_p^\star) (s,a) )^2 -  
\VV_{\PP_p(\cdot \mid s,a)}[V_p^\star] },
\]
where in the above expression, to avoid notation clutter, we use $k$ and $p$ to denote microepisode $l$'s episode number and corresponding player number $k(l) = \lceil l/M \rceil$ and $p(l) = ((l-1) \; \mathrm{mod} \; M)+1$, respectively.

It can be seen that $X_l$ is $\Fcal_l$-measurable, and 
$\EE\sbr{X_l \mid \Fcal_{l-1}} = 0$, i.e. $\cbr{X_l}_{l=1}^\infty$ is a martingale difference sequence adapted to $\cbr{\Fcal_l}_{l=1}^\infty$. 
In addition, 
\begin{align*}
\EE\sbr{ X_l^2 \mid \Fcal_{k-1} }
\leq &
\EE\sbr{  I( (s_{h,p}^k, a_{h,p}^k) = (s, a) ) (V_p^\star(s_{h+1,p}^k) - (\PP_p V_p^\star) (s,a) )^4 \mid \Fcal_{k-1} } \\
\leq &
H^2 \cdot I( (s_{h,p}^k, a_{h,p}^k) =  (s, a) ) \cdot \VV_{\PP_p(\cdot \mid s,a)}[V_p^\star]
=: U_l;
\end{align*}

Note that $\abs{X_l / H^2} \leq 1$; 
by~\cite[Corollary 1.4]{freedman1975tail} applied to $\cbr{X_l / H^2}_{l=1}^\infty$, for any $\lambda \geq 0$, \[ \cbr{ Y_k(\lambda) = \exp\del{\lambda (\sum_{i=1}^l \frac{X_i}{H^2} ) - \del{ (e^\lambda - \lambda - 1) \sum_{i=1}^l \frac{U_i}{H^4}  } } }_{l=0}^\infty
\]
is 
a nonnegative supermartingale. Also, note that if $l_j < \infty$,  $\sum_{i=1}^{l_j} X_i = \sum_{i=1}^{j} (V_p^\star(s_i') - (\PP_{p_i} V_p^\star) (s,a) )^2 -  
    \VV_{\PP_{p_i}(\cdot \mid s,a)}[V_p^\star]$, 
and $\sum_{i=1}^{l_j} U_i = \sum_{i=1}^j H^2 \cdot \VV_{\PP_{p_i}(\cdot \mid s,a)}[V_p^\star]$.

Using the same reasoning as in the proof of Lemma~\ref{lem:eind} (and observing that $\sum_{l=1}^{l_j} U_l \in [0, H^4 j]$), we have that for all $j \in \NN_+$:
\begin{align*}
\PP \vast(  & k_j < \infty \wedge \abr{ \frac{1}{j} \sum_{i=1}^{j} 
    \del{ (V_p^\star(s_i') - (\PP_{p_i} V_p^\star) (s,a) )^2 -  
    \VV_{\PP_p(\cdot \mid s,a)}[V_p^\star] } }
    \\
    & \qquad \qquad \geq 4 \sqrt{ \frac{ H  \sum_{i=1}^j \VV_{\PP_{p_i}(\cdot \mid s,a)}[V_p^\star] L(j)}{j^2}} +  \frac{2H^2 L(j)}{j} \; 
  \vast) \leq \ln (4j) \cdot \frac{\delta}{48 j^2 M S A}.
\end{align*}
A union bound over all $j \in \NN_+$ implies that Equation~\eqref{eqn:agg-var-2} holds.
\end{proof}

Define %
\begin{align*}
& E_{\samp} = E_{\indiv,\samp} \cap E_{\aggre,\samp}, \\
& E_{\aggre,\samp} = \cbr{ \forall s,a,k \centerdot \bar{n}^k(s,a) \geq N_1 \implies n^k(s,a) \geq \frac12 \bar{n}^k(s,a)
},\\
& E_{\indiv,\samp} = \cbr{ \forall s,a,k,p \centerdot \bar{n}_p^k(s,a) \geq N_2 \implies n_p^k(s,a) \geq \frac12 \bar{n}_p^k(s,a)
},
\end{align*}
where we recall from Section~\ref{sec:additional-notations} that  $N_1 \eqsim M \ln(\frac{S A K}{\delta})$, and $N_2 \eqsim \ln(\frac{M S A K}{\delta})$. 

\begin{lemma}
$\PP(E_{\samp}) \geq 1- \frac \delta 3$.  
\label{lem:esamp}
\end{lemma}

\begin{proof}
We first show $\PP(E_{\aggre,\samp}) \geq 1-\frac \delta 6$.
Specifically, fix $h \in [H]$ and $(s,a) \in \Scal_h \times \Acal$, define random variable $X_k = \sum_{p=1}^M \del{ \one\del{ (s_{h,p}^k, a_{h,p}^k} = (s,a) ) - \rho_p^k(s,a) }$. 
Also, define $\Gcal_k$ as the $\sigma$-algebra generated by all observations up to episode $k$. It can be readily seen that $\cbr{X_k}_{k=1}^K$ is a martingale difference sequence adapted to filtration $\cbr{ \Gcal_k }_{k=0}^K$. 
Freedman's inequality (specifically, Lemma 2 of~\cite{bartlett2008high}) implies that for every fixed $k$, with probability $1-\frac{\delta}{6 K}$,
\begin{equation}
\abs{ n^k(s,a) - \bar{n}^{k-1}(s,a) } \leq   4 \sqrt{ \bar{n}^{k-1}(s,a) \cdot M \ln\del{\frac{6SAK^2}{\delta}}  } + 4 M \ln\del{\frac{6SAK^2}{\delta}},
\label{eqn:occupancy-count}
\end{equation}
If Equation~\eqref{eqn:occupancy-count} happens, then by AM-GM inequality that 
$\sqrt{ \bar{n}^{k-1}(s,a) \cdot M \ln\del{\frac{6SAK^2}{\delta}}  } \leq \frac14 \bar{n}^{k-1}(s,a) + 16 M \ln\del{\frac{6SAK^2}{\delta}}$,
we have
\[
\bar{n}^{k-1}(s,a) - n^k(s,a)
\leq 
\frac14 \bar{n}^{k-1}(s,a) + 20 M \ln\del{\frac{6SAK^2}{\delta}},
\]
implying that 
\[
n^k(s,a) \geq \frac34 \bar{n}^{k-1}(s,a) - 20 M \ln\del{\frac{6SAK^2}{\delta}}.
\]
Additionally, as 
$\bar{n}^{k-1}(s,a) \geq \bar{n}^k(s,a) - M$ always holds, we have
\[
n^k(s,a) \geq \frac34 \bar{n}^{k}(s,a) - 21 M \ln\del{\frac{6SAK^2}{\delta}}.
\]
In summary, for any fixed $k$, with probability $1-\frac{\delta}{6K}$, if $\bar{n}^{k}(s,a) \geq N_1 := 84 M \ln\del{\frac{6SAK^2}{\delta}}$, 
\[
n^k(s,a) \geq \frac12 \bar{n}^{k}(s,a).
\]
Taking a union bound over all $k \in [K]$, we have $\PP(E_{\aggre,\samp}) \geq 1-\frac{\delta}{6}$.

It follows similarly that $\PP(E_{\indiv,\samp}) \geq 1-\frac\delta 6$; the only difference in the proof is that, we need to take an extra union bound over all $p \in [M]$ - hence an additional factor of $M$ within $\ln(\cdot)$ in the definition of $N_2$.
The  lemma statement follows from a union bound over these two statements.
\end{proof}

\begin{lemma}
$\PP(E) \geq 1 - \delta$.
\label{lem:clean}
\end{lemma}
\begin{proof}
Follows from Lemmas~\ref{lem:eind},~\ref{lem:eagg}, and~\ref{lem:esamp}, along with a union bound. 
\end{proof}

\subsection{Validity of value function bounds}
\label{subsec:value-bounds-valid}

In this section, we show that if the clean event $E$ happens, then for all $k$ and $p$, the value function  estimates $\qub_p^k$, $\qlb_p^k$, $\vub_p^k$, $\vlb_p^k$ are valid upper and lower bounds of the optimal value functions $Q_p^\star$, $V_p^\star$ (Lemma~\ref{lem:value-bounds-valid}). As a by-product, we also give a general bound on the surplus (Lemma~\ref{lem:surplus-bounds}) which will be refined and used in the subsequent regret bound calculations.
Before going into the proof of the above two lemmas, we need a technical lemma below (Lemma~\ref{lem:bonus-motivation}) that gives necessary concentration results which motivate the bonus constructions; its proof can be found at Section~\ref{sec:bonus-motivation}.

\begin{lemma}
Fix $p \in [M]$. Suppose $E$ happens, and suppose that for episode $k$ and step $h$, we have that for all $s' \in \Scal_{h+1}$, $\vlb^k_p(s') \leq V_p^\star(s') \leq \vub^k_p(s')$. Then, for all $(s,a) \in \Scal_h \times \Acal$: 
\begin{enumerate}
    \item 
    \begin{equation}
      \abs{ \hat{R}_p^k(s,a) - R_p(s,a) } \leq b_{\rw}\del{ n_p^k(s,a), 0},
      \label{eqn:ind-rw-valid}
    \end{equation} 
    \begin{equation}
        \abs{ \hat{R}^k(s,a) - R_p(s,a) } \leq b_{\rw}\del{ n^k(s,a), \epsilon}.
        \label{eqn:agg-rw-valid}
    \end{equation} 
    \item 
    \begin{equation} 
    \abs{ (\hat{\PP}_p^k - \PP_p) (V_p^\star) (s,a)  } \leq b_{\prob}\del{ \hat{\PP}_p^k(\cdot \mid s,a), n_p^k(s,a), \vub_p^k, \vlb_p^k, 0},
    \label{eqn:ind-prob-valid}
    \end{equation}
    \begin{equation}
    \abs{ (\hat{\PP}^k - \PP_p) (V_p^\star) (s,a)  } \leq b_{\prob}\del{ \hat{\PP}^k (\cdot \mid s,a), n^k(s,a), \vub_p^k, \vlb_p^k, \epsilon }.
    \label{eqn:agg-prob-valid}
    \end{equation}
    \item For any $V_1, V_2: \Scal_{h+1} \to \RR$ such that $\vlb_p^k \leq V_1 \leq V_2 \leq \vub_p^k$,
    \begin{equation}
     \abs{ (\hat{\PP}_p^k - \PP_p) (V_2 - V_1)(s,a)} \leq b_{\str} \del{ \hat{\PP}_p^k(\cdot \mid s,a), n_p^k(s,a), \vub_p^k, \vlb_p^k, 0},   
     \label{eqn:ind-str-valid}
    \end{equation} 
    \begin{equation}
    \abs{ (\hat{\PP}^k - \PP_p) (V_2 - V_1)(s,a)} \leq b_{\str}\del{ \hat{\PP}^k(\cdot \mid s,a), n^k(s,a), \vub_p^k, \vlb_p^k, \epsilon }.
    \label{eqn:agg-str-valid}
    \end{equation}
\end{enumerate}
\label{lem:bonus-motivation}
\end{lemma}

\begin{lemma}
\label{lem:surplus-bounds}
If event $E$ happens, and suppose that for episode $k$ and step $h$, we have that for all $s' \in \Scal_{h+1}$, $\vlb^k_p(s') \leq V_p^\star(s') \leq \vub^k_p(s')$.
Then, for $(s,a) \in \Scal_h \times \Acal$,
\begin{equation}
    \qub_p^k(s,a) - \del{R_p(s,a) + ( \PP_p \vub_p^k) (s,a) } \in \sbr{ 0, (H-h+1) \wedge 2 \bindp_p^k(s,a) \wedge 2\baggp_p^k(s,a) }, 
    \label{eqn:strong-optimism}
\end{equation}
and
\begin{equation}
    \del{ R_p(s,a) + (\PP_p \vlb_p^k) (s,a) } - \qlb_p^k(s,a) \in \sbr{ 0, (H-h+1) \wedge 2 \bindp_p^k(s,a) \wedge 2\baggp_p^k(s,a) },
    \label{eqn:strong-pessimism}
\end{equation}
where we recall that 
\[
\bindp_p^k(s,a) = b_{\rw}\del{ n_p^k(s,a), 0} + b_{\prob}\del{ \hat{\PP}_p^k(\cdot \mid s,a), n_p^k(s,a), \overline{V}_p^k, \underline{V}_p^k, 0}  + b_{\str}\del{ \hat{\PP}_p^k(\cdot \mid s,a), n_p^k(s,a), \overline{V}_p^k, \underline{V}_p^k, 0},
\]
\[
\baggp_p^k(s,a) = b_{\rw}\del{ n^k(s,a), \epsilon} + b_{\prob}\del{ \hat{\PP}^k(\cdot \mid s,a), n^k(s,a), \overline{V}_p^k, \underline{V}_p^k, \epsilon} + b_{\str}\del{ \hat{\PP}^k(\cdot \mid s,a), n^k(s,a), \overline{V}_p^k, \underline{V}_p^k, \epsilon}.
\]
\end{lemma}
\begin{proof}
We only show Equation~\eqref{eqn:strong-optimism} for brevity; Equation~\eqref{eqn:strong-pessimism} follows from an exact symmetrical reasoning.

Recall that $\qub_p^k(s,a) = \min\del{ \qindub_p^k(s,a), \qaggub_p^k(s,a), H }$.
We compare each term in the $\min(\cdot)$ operator with $(R_p(s,a) + (\PP_p \vub_p^k) (s,a))$:

\begin{itemize}
\item For $\qindub_p^k(s,a)$, using Lemma~\ref{lem:bonus-motivation} and our assumption on $\vub^k_p$ and $\vlb^k_p$ on $\Scal_{h+1}$, we have:
\begin{align*}
 & \qindub_p^k(s,a) - \del{ R_p(s,a) + (\PP_p \vub_p^k) (s,a) } \\
= & \;  (\hat{R}_p^k - R_p) (s,a) + b_{\rw}\del{ n_p^k(s,a), 0 }  \\
&  + ((\hat{\PP}_p^k - \PP_p) V^\star_p)(s,a) + b_{\prob}\del{ \hat{\PP}_p^k(\cdot \mid s,a), n_p^k(s,a), \overline{V}_p^k, \underline{V}_p^k, 0} \\
& + (\hat{\PP}_p^k - \PP_p)(\vub_p^k - V_p^\star))(s,a) + b_{\str}\del{ \hat{\PP}_p^k(\cdot \mid s,a), n_p^k(s,a), \overline{V}_p^k, \underline{V}_p^k, 0} \\
\in & \; [0, 2 \bindp_p^k(s,a)].
\end{align*}

\item For $\qaggub_p^k(s,a)$, using Lemma~\ref{lem:bonus-motivation} and our assumptions on $\vub^k_p$ and $\vlb^k_p$ over $\Scal_{h+1}$, we have:
\begin{align*}
 & \qaggub_p^k(s,a) - \del{ R_p(s,a) + (\PP_p \vub_p^k) (s,a) } \\
= & \; (\hat{R}_p^k - R_p) (s,a) + b_{\rw}\del{ n^k(s,a), \epsilon }  \\
&  + ((\hat{\PP}^k - \PP_p) V^\star_p)(s,a) + b_{\prob}\del{ \hat{\PP}^k(\cdot \mid s,a), n^k(s,a), \overline{V}_p^k, \underline{V}_p^k, \epsilon} \\
& + ((\hat{\PP}^k - \PP_p)(\vub_p^k - V_p^\star))(s,a) + b_{\str}\del{ \hat{\PP}^k(\cdot \mid s,a), n^k(s,a), \overline{V}_p^k, \underline{V}_p^k, \epsilon } \\
\in & \; [0, 2\baggp_p^k(s,a)],
\end{align*}

\item For $H -h + 1$, we have:
\begin{align*}
    & (H - h + 1) - (R_p(s,a) + (\PP_p \vub_p^k) (s,a)) \in [0, H - h + 1],
\end{align*}
where we use the observation that $R(s,a) \in [0,1]$, and $(\PP_p \vub_p^k) (s,a) \in [0, H-h]$, and their sum is in $[0,H]$.

\end{itemize}
Combining the above three establishes that 
\[ 
\qub_p^k(s,a) - (R(s,a) + (\PP_p \vub_p^k) (s,a)) \in \sbr{ 0, (H - h + 1) \wedge 2 \bindp_p^k(s,a) \wedge 2\baggp_p^k(s,a) } . \qedhere
\]
\end{proof}

\begin{lemma}
\label{lem:value-bounds-valid}
Under event $E$, for every $k \in [K]$, and every $p \in [M]$, and 
for every $h \in [H]$,
    For all $(s,a) \in \Scal_h \times \Acal$,
    \begin{equation}
    \qlb_p^k (s,a)
    \leq
    Q_p^{\pi^k} (s,a)
    \leq 
    Q_p^\star (s,a)
    \leq 
    \qub_p^k (s,a),
    \label{eqn:q-order}
    \end{equation}
    and 
    \begin{equation}
    \vlb_p^k (s)
    \leq
    V_p^{\pi^k} (s)
    \leq 
    V_p^\star (s)
    \leq 
    \vub_p^k (s),
    \label{eqn:v-order}
    \end{equation}
here, recall that $V_p^{\pi_k}$ is the value function of policy $\pi^k(p)$ with respect to $\Mcal_p$ defined in Section~\ref{sec:prelims}.
%\zhi{Since this is the first time we mention $V_p^{\pi_k}$, should we add a sentence to explain how it is different from $\bar{V}_p^k$?}
%\chicheng{Recall that $V_p^{\pi_k}$ is..}
\end{lemma}
\begin{proof}
The proof of this lemma extends~\cite[Proposition F.1]{simchowitz2019non} to our multitask setting. 

For every $k$ and $p$, we show the above holds for all layers $h \in [H]$ and every $(s,a) \in \Scal_h \times \Acal$; to this end, we do backward induction on layer $h$.

\paragraph{Base case:} For layer $h = H+1$, we have $\vlb_p^k(\bot) 
=
V_p^{\pi^k}(\bot)
= 
V_p^\star(\bot)
= 
\vub_p^k(\bot) = 0$.

\paragraph{Inductive case:} By our inductive hypothesis, for layer $h+1$ and every $s \in \Scal_{h+1}$, 
\[
\vlb_p^k(s) 
\leq
V_p^{\pi^k}(s)
\leq
V_p^\star(s)
\leq
\vub_p^k(s).
\]

We will show that Equations~\eqref{eqn:q-order} and~\eqref{eqn:v-order} holds 
holds for all $(s,a) \in \Scal_h \times \Acal$. 

We first show Equation~\eqref{eqn:q-order}. First,
$Q_p^{\pi^k}(s,a)
\leq 
Q_p^\star(s,a)$ for all $(s,a) \in \Scal_h \times \Acal$ is trivial. 

To show $Q_p^\star(s,a) \leq \qub_p^k(s,a)$ for all $(s,a) \in \Scal_h \times \Acal$, by Lemma~\ref{lem:surplus-bounds} and inductive hypothesis, we have:
\[
Q_p^\star(s,a) 
= 
R_p(s,a) + (\PP_p V_p^\star) (s,a)
\leq 
R_p(s,a) + (\PP_p \vub_p^k) (s,a)
\leq 
\qub_p^k(s,a).
\]
Likewise, we show $Q_p^{\pi^k}(s,a) \geq \qlb_p^k(s,a)$ for all $(s,a) \in \Scal_h \times \Acal$, using Lemma~\ref{lem:surplus-bounds} and inductive hypothesis: 
\[
Q_p^{\pi^k}(s,a) 
= 
R_p(s,a) + (\PP_p V_p^{\pi^k}) (s,a)
\geq 
R_p(s,a) + (\PP_p \vub_p^k) (s,a)
\geq 
\qlb_p^k(s,a).
\]
This completes the proof of Equation~\eqref{eqn:q-order} for layer $h$.

We now show Equation~\eqref{eqn:v-order} for layer $h$. Again $V_p^{\pi^k} (s) \leq V_p^\star (s)$ for all $s \in \Scal_h$ is trivial. 

To show $V_p^\star(s) \leq \vub_p^k(s)$ for all $s \in \Scal_h$, observe that 
\[
V_p^\star(s) 
= 
\max_{a \in \Acal } Q_p^\star(s,a)
\leq 
\max_{a \in \Acal} \qub_p^k(s,a)
=
\vub_p^k(s).
\]
To show $V_p^{\pi^k}(s) \geq \vlb_p^k(s)$ for all $s \in \Scal_h$, observe that
\[
V_p^{\pi^k}(s)
= 
Q_p^{\pi^k}(s,\pi^k(p)(s))
\geq 
\qlb_p^k(s,\pi^k(p)(s))
=
\vlb_p^k(s).
\]
This completes the induction. 
\end{proof}

\subsubsection{Proof of Lemma~\ref{lem:bonus-motivation}}
\label{sec:bonus-motivation}

\begin{proof}[Proof of Lemma~\ref{lem:bonus-motivation}]
Equations~\eqref{eqn:ind-rw-valid},~\eqref{eqn:ind-prob-valid}, and~\eqref{eqn:ind-str-valid} essentially follow the same reasoning as in~\cite{simchowitz2019non}; we still include their proofs for completeness. 
Equations~\eqref{eqn:agg-rw-valid},~\eqref{eqn:agg-prob-valid}, and~\eqref{eqn:agg-str-valid} are new, and require a more involved analysis. Our proof also relies on a technical lemma, namely  Lemma~\ref{lem:emp-bernstein-eps}; we defer its statement and proof to the end of this subsection.

\begin{enumerate}
    \item Equation~\eqref{eqn:ind-rw-valid} follows directly from the definition of $E_{\indiv,\rw}$.
    Equation~\eqref{eqn:agg-rw-valid} follows from the definition of $E_{\aggre,\rw}$, and the fact that $\abs{ \bar{R}^k(s,a) - R_p(s,a) } \leq \epsilon$.
    
    \item We prove  Equation~\eqref{eqn:ind-prob-valid} as follows:
    \begin{align*}
    & \abs{ 
    (\hat{\PP}_p^k V^\star - \PP_p V_p^\star)(s,a) } \\
    \leq & O\del{\sqrt{ \frac{ \VV_{\PP_p(\cdot \mid s,a)}[V^\star] L(n_p^k(s,a))}{n_p^k(s,a)} } + \frac{H L(n_p^k(s,a))}{n_p^k(s,a)} } \\
    \leq & 
    O\del{ \sqrt{ \frac{ \VV_{\hat{\PP}_p^k(\cdot \mid s,a)}[V^\star] L(n_p^k(s,a))}{n_p^k(s,a)} } + \frac{H L(n_p^k(s,a))}{n_p^k(s,a)} } \\
    \leq & 
    O\del{ \sqrt{ \frac{ \VV_{\hat{\PP}_p^k(\cdot \mid s,a)}[\vub_p^k] \;  L(n_p^k(s,a))}{n_p^k(s,a)} } + \sqrt{ \frac{ \| V_p^\star - \vub_p^k \|_{\hat{\PP}_p^k(\cdot \mid s,a)}^2 \; L(n_p^k(s,a))}{n_p^k(s,a)} } + \frac{H L(n_p^k(s,a))}{n_p^k(s,a)} } \\
    \leq & 
    O\del{ \sqrt{ \frac{ \VV_{\hat{\PP}_p^k(\cdot \mid s,a)}[\vub_p^k] \;  L(n_p^k(s,a))}{n_p^k(s,a)} } + \sqrt{ \frac{ \| \vub_p^k - \vlb_p^k \|_{\hat{\PP}_p^k(\cdot \mid s,a)}^2 \; L(n_p^k(s,a))}{n_p^k(s,a)} } + \frac{H L(n_p^k(s,a))}{n_p^k(s,a)} } \\
    \leq & b_{\prob}\del{ \hat{\PP}_p^k(\cdot \mid s,a), n_p^k(s,a), \vub_p^k, \vlb_p^k, 0 },
    \end{align*}
    where the first inequality is from the definition of $E_{\indiv,\val}$; the second inequality is from Equation~\eqref{eqn:emp-bernstein-ind} of Lemma~\ref{lem:emp-bernstein-eps}; the third inequality is from Lemma~\ref{lem:var-x-y}; the fourth inequality is from our assumption that for all $s' \in \Scal_{h+1}$, $\vlb^k_p(s') \leq V^\star(s') \leq \vub^k_p(s')$, and thus $\abs{ (V_p^\star - \vlb_p^k)(s') } \leq \abs{(\vub_p^k - \vlb_p^k)(s')}$ for all $s'$ in the support of $\hat{\PP}_p^k(\cdot \mid s,a)$.
    
    We prove Equation~\eqref{eqn:agg-prob-valid} as follows:
    \begin{align*}
        & \abs{ (\hat{\PP}^k - \PP_p) (V_p^\star) (s,a)  } \\
        \leq & \epsilon + \abs{ (\hat{\PP}^k - \bar{\PP}^k) (V_p^\star) (s,a) } \\
        \leq & \epsilon + O\del{\sqrt{  \frac{ \del{ \sum_{q=1}^M w_q^k(s,a) \VV_{\PP_q(\cdot \mid s,a)}[V_p^\star] } L(n^k(s,a)) }{n^k(s,a)} } + \frac{H L(n^k(s,a))}{n^k(s,a)}} \\
        \leq & \epsilon + O\del{\sqrt{  \frac{   \VV_{\hat{\PP}^k(\cdot \mid s,a)}[V_p^\star] \; L(n^k(s,a)) }{n^k(s,a)} } + \sqrt{\frac{ L(n^k(s,a))}{n^k(s,a)} \cdot \epsilon H} + \frac{H L(n^k(s,a))}{n^k(s,a)} + \frac{H L(n^k(s,a))}{n^k(s,a)}}  \\
        \leq & 2\epsilon + O\del{ \sqrt{  \frac{   \VV_{\hat{\PP}^k(\cdot \mid s,a)}[\vub_p^k] \; L(n^k(s,a)) }{n^k(s,a)} } + \sqrt{ \frac{ \| \vub_p^k - V_p^\star \|_{\hat{\PP}^k(\cdot \mid s,a)}^2 \; L(n^k(s,a))}{n^k(s,a)} } + \frac{H L(n^k(s,a))}{n^k(s,a)}} \\
        \leq & 2\epsilon + O\del{ \sqrt{  \frac{   \VV_{\hat{\PP}^k(\cdot \mid s,a)}[\vub_p^k] \; L(n^k(s,a)) }{n^k(s,a)} } + \sqrt{ \frac{ \| \vub_p^k - \vlb_p^k \|_{\hat{\PP}^k(\cdot \mid s,a)}^2 \; L(n^k(s,a))}{n^k(s,a)} } + \frac{H L(n^k(s,a))}{n^k(s,a)} } \\
        \leq & 
      b_{\prob}\del{\hat{\PP}^k (\cdot \mid s,a), n^k(s,a), \vub_p^k, \vlb_p^k, \epsilon},
    \end{align*}
    where the first inequality is from the observation that $\| \bar{\PP}^k(\cdot \mid s,a) - \PP_p(\cdot \mid s,a) \|_1 \leq \frac{\epsilon}{H}$ and Lemma~\ref{lem:var-tv};
    the second inequality is from the definition of $E_{\aggre,\val}$; the third inequality is from Equation~\eqref{eqn:emp-bernstein-agg-1} of Lemma~\ref{lem:emp-bernstein-eps}; the fourth inequality is from Lemma~\ref{lem:var-x-y} and the observation that for constant $c > 0$, $c \sqrt{\frac{ L(n^k(s,a))}{n^k(s,a)} \cdot \epsilon H} \leq \epsilon + \frac{c^2}{4} \frac{ L(n^k(s,a))}{n^k(s,a)}$ by AM-GM inequality; the fifth inequality is from our assumption that for all $s' \in \Scal_{h+1}$, $\vlb^k_p(s') \leq V^\star(s') \leq \vub^k_p(s')$, and thus $\abs{ (V_p^\star - \vlb_p^k)(s') } \leq \abs{(\vub_p^k - \vlb_p^k)(s')}$ for all $s'$ in the support of $\hat{\PP}^k(\cdot \mid s,a)$.
    
    \item We prove Equation~\eqref{eqn:ind-str-valid} as follows:
    \begin{align*}
         & \abs{ (\hat{\PP}_p^k - \PP_p) (V_2 - V_1)(s,a)}  \\
         \leq & 
         \sum_{s' \in \Scal_{h+1}} \abs{ (\hat{\PP}_p^k - \PP_p)(s' \mid s,a)} \cdot (V_2 - V_1)(s') \\
         \leq & 
         O\del{ \sum_{s' \in \Scal_{h+1}} \del{
  \sqrt{\frac{L(n_p^k(s,a)) \cdot \PP_p(s' \mid s, a)}{n_p^k(s,a)}} + \frac{L(n_p^k(s,a))}{n_p^k(s,a)} } \cdot 
         (V_2 - V_1)(s') } \\
         \leq & 
         O \del{ \sum_{s' \in \Scal_{h+1}} \del{
  \sqrt{\frac{L(n_p^k(s,a)) \cdot \hat{\PP}_p^k(s' \mid s, a)}{n_p^k(s,a)}} + \frac{L(n_p^k(s,a))}{n_p^k(s,a)} } \cdot 
         (V_2 - V_1)(s') } \\
         \leq & O\del{
         \sum_{s' \in \Scal_{h+1}} \sqrt{\hat{\PP}_p^k(s' \mid s, a)} (\vub_p^k - \vlb_p^k)(s') \cdot
  \sqrt{\frac{L(n_p^k(s,a))}{n_p^k(s,a)} } + \sum_{s' \in \Scal_{h+1}} \frac{H L(n_p^k(s,a))}{n_p^k(s,a)} } \\
        \leq & O\del{
  \sqrt{\frac{S \| \vub_p^k - \vlb_p^k \|_{\hat{\PP}_p^k(\cdot \mid s,a)}^2 \; L(n_p^k(s,a))}{n_p^k(s,a)} } +  \frac{S H L(n_p^k(s,a))}{n_p^k(s,a)} }
  \\
    \leq & b_{\str}\del{\hat{\PP}_p^k(\cdot \mid s,a), n(s,a), \vub_p^k, \vlb_p^k, 0},
    \end{align*}
    where the first inequality is from the elementary fact that $\abs{\sum_{i=1}^n a_i} \leq \sum_{i=1}^n \abs{a_i}$; the second inequality is from the definition of $E_{\indiv,\prob}$; the third inequality is from the definition of $E_{\indiv,\prob}$ and Lemma~\ref{lem:rel-bound}; the fourth inequality is by algebra and $0 \leq (V_2 - V_1)(s') \leq \min(H, (\vub_p^k - \vlb_p^k)(s'))$ for all $s' \in \Scal_{h+1}$; the fifth inequality is by Cauchy-Schwarz.
    
    We now prove Equation~\eqref{eqn:agg-str-valid}:
       \begin{align*}
         & \abs{ (\hat{\PP}^k - \PP_p) (V_2 - V_1)(s,a)}  \\
         \leq & \abs{ (\bar{\PP}^k - \PP_p) (V_2 - V_1)(s,a)} + \abs{ (\hat{\PP}^k - \bar{\PP}^k) (V_2 - V_1)(s,a)}   \\
         \leq & 
         \epsilon + \sum_{s' \in \Scal_{h+1}} \abs{ (\hat{\PP}^k - \bar{\PP}^k)(s' \mid s,a)} \cdot (V_2 - V_1)(s')  \\
         \leq & 
         \epsilon + O\del{
         \sum_{s' \in \Scal_{h+1}} \del{
  \sqrt{\frac{L(n^k(s,a)) \cdot \bar{\PP}^k(s' \mid s, a)}{n^k(s,a)}} + \frac{L(n^k(s,a))}{n^k(s,a)} } \cdot 
         (V_2 - V_1)(s') } \\
         \leq & 
         \epsilon + O\del{\sum_{s' \in \Scal_{h+1}} \del{
  \sqrt{\frac{L(n^k(s,a)) \cdot \hat{\PP}^k(s' \mid s, a)}{n^k(s,a)}} + \frac{L(n^k(s,a))}{n^k(s,a)} } \cdot 
         (V_2 - V_1)(s')} \\
         \leq &
         \epsilon + O\del{\sum_{s' \in \Scal_{h+1}} \sqrt{\hat{\PP}^k(s' \mid s, a)} (\vub_p^k - \vlb_p^k)(s') \cdot
  \sqrt{\frac{L(n^k(s,a))}{n^k(s,a)} } + \sum_{s' \in \Scal_{h+1}} \frac{H L(n^k(s,a))}{n^k(s,a)}} \\
        \leq & \epsilon + O\del{
  \sqrt{\frac{S \| \vub_p^k - \vlb_p^k \|_{\hat{\PP}^k(\cdot \mid s,a)}^2 \; L(n^k(s,a))}{n^k(s,a)} } +  \frac{S H L(n^k(s,a))}{n^k(s,a)} }
  \\
    \leq & b_{\str}\del{\hat{\PP}^k(\cdot \mid s,a), n(s,a), \vub_p^k, \vlb_p^k, \epsilon},
    \end{align*}
    where the first inequality is triangle inequality; the second inequality is from
    the elementary fact that $\abs{\sum_{i=1}^n a_i} \leq \sum_{i=1}^n \abs{a_i}$, along with $\| \bar{\PP}_k(\cdot \mid s,a) - \PP_p(\cdot \mid s,a) \|_1 \leq \frac{\epsilon}{H}$ %
    and Lemma~\ref{lem:var-tv}; the third inequality is from the definition of $E_{\aggre,\prob}$; the fourth inequality is from the definition of $E_{\aggre,\prob}$ and Lemma~\ref{lem:rel-bound}; the fifth inequality is by algebra and $0 \leq (V_2 - V_1)(s') \leq \min(H, (\vub_p^k - \vlb_p^k)(s') )$ for all $s' \in \Scal_{h+1}$; the last inequality is by Cauchy-Schwarz.
    \qedhere
    \end{enumerate}
\end{proof}

Lemma~\ref{lem:bonus-motivation} relies on the following technical lemma on the concentrations of the conditional variances. Specifically, Equation~\eqref{eqn:emp-bernstein-ind} is well-known (see, e.g.,~\cite{audibert2007tuning,maurerempirical}); Equations~\eqref{eqn:emp-bernstein-agg-1} and~\eqref{eqn:emp-bernstein-agg-2} are new, and allow for heterogeneous data aggregation in the multi-task RL setting. We still include the proof of Equation~\eqref{eqn:emp-bernstein-ind} here, as it helps illustrate our ideas for proving the two new inequalities.

\begin{lemma}
\label{lem:emp-bernstein-eps}
If event $E$ happens, then for any $s,a,k,p$, we have:
\begin{enumerate}
    \item 
    \begin{equation}
    \abs{ 
\sqrt{\VV_{\hat{\PP}_p^k(\cdot \mid s,a)} \sbr{ V_p^\star }} - \sqrt{\VV_{\PP_p(\cdot \mid s,a)} \sbr{ V_p^\star }} } \lesssim H \sqrt{ \frac{  L(n_p^k(s,a))}{n_p^k(s,a)}},
    \label{eqn:emp-bernstein-ind}
    \end{equation}
    \item 
    \begin{equation}
    \abs{ 
\sqrt{\VV_{\hat{\PP}^k(\cdot \mid s,a)} \sbr{ V_p^\star }} - \sqrt{\sum_{q=1}^M w_q^k(s,a) \VV_{\PP_q(\cdot \mid s,a)} \sbr{ V_p^\star }} } \lesssim \sqrt{H \epsilon} + H \sqrt{ \frac{  L(n^k(s,a))}{n^k(s,a)}},
    \label{eqn:emp-bernstein-agg-1}
    \end{equation}
    and
    \begin{equation}
    \abs{ 
\sqrt{\VV_{\hat{\PP}^k(\cdot \mid s,a)} \sbr{ V_p^\star }} - \sqrt{\VV_{\PP_p(\cdot \mid s,a)} \sbr{ V_p^\star }} } \lesssim \sqrt{H \epsilon} + H \sqrt{ \frac{  L(n^k(s,a))}{n^k(s,a)}},
    \label{eqn:emp-bernstein-agg-2}
    \end{equation}
\end{enumerate}
\end{lemma}
\begin{proof}
\begin{enumerate}[wide, labelwidth=!, labelindent=0pt]
\item By the definition of $E$, we have
\[
\abr{ \frac{1}{n_p^k(s,a)} \sum_{i=1}^{n_p^k(s,a)} 
    (V_p^\star((s_i^p)') - (\PP_p V_p^\star) (s,a) )^2 -  
    \VV_{\PP_p(\cdot \mid s,a)}[V_p^\star] }
    \lesssim \sqrt{ \frac{ H^2  \VV_{\PP_p(\cdot \mid s,a)}[V_p^\star] L(n_p^k(s,a))}{n_p^k(s,a)}} + \frac{H^2 L(n_p^k(s,a))}{n_p^k(s,a)};
\]
this, when combined with Lemma~\ref{lem:rel-bound}, implies that
\begin{equation}
\abs{ \sqrt{ \frac{1}{n_p^k(s,a)} \sum_{i=1}^{n_p^k(s,a)} 
    (V_p^\star((s_i^p)') - (\PP_p V_p^\star) (s,a) )^2 } - \sqrt{ \VV_{\PP_p(\cdot \mid s,a)}[V_p^\star] } }
    \leq 
    H \sqrt{ \frac{ L(n_p^k(s,a))}{n_p^k(s,a)} }.
    \label{eqn:ind-pvar-tvar}
\end{equation}
Now, observe that 
\[
\VV_{\hat{\PP}_p^k(\cdot \mid s,a)} \sbr{ V_p^\star }
= 
\frac{1}{n_p^k(s,a)} \sum_{i=1}^{n_p^k(s,a)} 
    (V_p^\star((s_i^p)') - (\PP_p V_p^\star) (s,a) )^2 
-
    ((\hat{\PP}_p^k V_p^\star) (s,a) - (\PP_p V_p^\star) (s,a) )^2,
\]
which can be seen by applying Lemma~\ref{lem:bv} with $X$ being the random variable that is drawn uniformly from $\cbr{V_p^\star((s_i^p)')}_{i=1}^{n_p^k(s,a)}$, which has expectation $\mu = (\hat{\PP}_p^k V_p^\star) (s,a)$, and setting $m = (\PP_p V_p^\star) (s,a)$. 

%\chicheng{Add explanations here}
Recall that by the definition of event $E$, we have
\[
\abs{ (\hat{\PP}_p^k V_p^\star) (s,a) - (\PP_p V_p^\star) (s,a) } \leq H \wedge  
\del{ \sqrt{ \frac{ H^2  L(n_p^k(s,a))}{n_p^k(s,a)} } + \frac{H L(n_p^k(s,a))}{n_p^k(s,a)} }
\leq 
2H \sqrt{ \frac{   L(n_p^k(s,a))}{n_p^k(s,a)} } ,
\]
where the second inequality uses Lemma~\ref{lem:wedge-a-sqrta}. Using the elementary fact that $\abs{A - B} \leq C \Rightarrow \sqrt{A} \leq \sqrt{B} + \sqrt{C}$, we get that
\begin{align}
\begin{split}
& \abs{ 
\sqrt{\VV_{\hat{\PP}_p^k(\cdot \mid s,a)} \sbr{ V_p^\star }}
-
\sqrt{\frac{1}{n^k(s,a)} \sum_{i=1}^{n^k(s,a)} 
    (V_p^\star((s_i^p)') - (\PP_p V_p^\star) (s,a) )^2} } \\
\leq &
    \abr{ (\hat{\PP}_p^k V_p^\star) (s,a) - (\PP_p V_p^\star) (s,a) } 
\lesssim
H \sqrt{ \frac{  L(n_p^k(s,a))}{n_p^k(s,a)} }.
\end{split}
\label{eqn:ind-evar-pvar}
\end{align}
Combining Equations~\eqref{eqn:ind-pvar-tvar} and~\eqref{eqn:ind-evar-pvar}, using algebra, we get
\[
\abs{ 
\sqrt{\VV_{\hat{\PP}_p^k(\cdot \mid s,a)} \sbr{ V_p^\star }} - \sqrt{\VV_{\PP_p(\cdot \mid s,a)} \sbr{ V_p^\star }} } \lesssim H \sqrt{ \frac{  L(n_p^k(s,a))}{n_p^k(s,a)}},
\]
establishing Equation~\eqref{eqn:emp-bernstein-ind}.

\item We first show Equation~\eqref{eqn:emp-bernstein-agg-1}.
By the definition of $E$, we have
\begin{align*}
 & \abr{ \frac{1}{n^k(s,a)} \sum_{i=1}^{n^k(s,a)} 
    (V_p^\star(s_i') - (\PP_{p_i} V_p^\star) (s,a) )^2 -  
    \sum_{p=1}^M w_p^k(s,a) \VV_{\PP_p(\cdot \mid s,a)}[V_p^\star] } \\
    \lesssim & \sqrt{ \frac{ H^2  \del{ \sum_{p=1}^M w_p^k(s,a) \VV_{\PP_p(\cdot \mid s,a) }[V_p^\star]} L(n^k(s,a))}{n^k(s,a)}} + \frac{H^2 L(n^k(s,a))}{n^k(s,a)},
\end{align*}
this, combined with Lemma~\ref{lem:rel-bound}, implies that
\begin{equation}
\abs{ \sqrt{ \frac{1}{n^k(s,a)} \sum_{i=1}^{n^k(s,a)} 
    (V_p^\star(s_i') - (\PP_{p_i} V_p^\star) (s,a) )^2 } - \sqrt{ \sum_{p=1}^M w_p^k(s,a) \VV_{\PP_p(\cdot \mid s,a)}[V_p^\star] } }
    \lesssim H \sqrt{ \frac{ L(n^k(s,a))}{n^k(s,a)} }.
    \label{eqn:agg-avar1-pvar}
\end{equation}

For the first term on the left hand side, observe that for each $i$, $| (\PP_{p_i} V_p^\star)(s,a) - (\PP_p V_p^\star)(s,a) | \leq H \frac{\epsilon}{H} = \epsilon$, we therefore have $\abs{ (V_p^\star(s_i') - (\PP_{p_i} V_p^\star) (s,a) )^2 - (V_p^\star(s_i') - (\PP_p V_p^\star) (s,a) )^2 } \leq 2H\epsilon$ by  $2H$-Lipschitzness of function $f(x) = x^2$ on $[-H, H]$. By averaging over all $i$'s and taking square root, we have
\begin{equation}
\abs{
\sqrt{\frac{1}{n^k(s,a)} \sum_{i=1}^{n^k(s,a)} 
    (V_p^\star(s_i') - (\PP_{p_i} V_p^\star) (s,a) )^2 }
    - 
   \sqrt{ \frac{1}{n^k(s,a)} \sum_{i=1}^{n^k(s,a)}
    (V_p^\star(s_i') - (\PP_p V_p^\star) (s,a) )^2  }
}
\lesssim
\sqrt{H \epsilon}.
\label{eqn:agg-avar1-avar2}
\end{equation}
Furthermore,
\[
  \VV_{\hat{\PP}^k(\cdot \mid s,a)}\sbr{V_p^\star}
  =
  \frac{1}{n^k(s,a)} \sum_{i=1}^{n^k(s,a)}
    (V_p^\star(s_i') - (\PP_p V_p^\star) (s,a) )^2
    - ( (\hat{\PP}^k V_p^\star)(s,a) - (\PP_p V_p^\star) (s,a) )^2,
\]
which can be seen by applying Lemma~\ref{lem:bv} with $X$ being the random variable that is drawn uniformly from $\cbr{V_p^\star(s_i')}_{i=1}^{n^k(s,a)}$, which has expectation $\mu = (\hat{\PP}^k V_p^\star) (s,a)$, and setting $m = (\PP_p V_p^\star) (s,a)$.

In addition,
\[
\abs{ (\hat{\PP}^k V_p^\star)(s,a) - (\PP_p V_p^\star) (s,a) }
\lesssim 
\epsilon + H\sqrt{ \frac{ L(n^k(s,a))}{n^k(s,a)}}
\]
Together with our assumption that $\epsilon \leq 2H$ (which implies that $\epsilon \lesssim \sqrt{H \epsilon}$), this gives
\begin{equation}
\abs{
    \sqrt{\VV_{\hat{\PP}^k(\cdot \mid s,a)}\sbr{V_p^\star}}
    - 
   \sqrt{ \frac{1}{n^k(s,a)} \sum_{i=1}^{n^k(s,a)}
    (V_p^\star(s_i') - (\PP_p V_p^\star) (s,a) )^2  }
}
\lesssim 
\sqrt{H \epsilon} + H\sqrt{ \frac{ L(n^k(s,a))}{n^k(s,a)}}.
\label{eqn:agg-evar-avar2}
\end{equation}

Equation~\eqref{eqn:emp-bernstein-agg-1} is a direct consequence of  Equations~\eqref{eqn:agg-avar1-pvar},~\eqref{eqn:agg-avar1-avar2} and~\eqref{eqn:agg-evar-avar2} along with algebra.

We now show Equation~\eqref{eqn:emp-bernstein-agg-2} using Equation~\eqref{eqn:emp-bernstein-agg-1}.
By Lemma~\ref{lem:var-tv}, 
for every $q$, $\abs{ \VV_{\PP_q(\cdot \mid s,a)}\sbr{V_p^\star} - \VV_{\PP_p(\cdot \mid s,a)}\sbr{V_p^\star}} \leq 3H^2 \cdot \frac{\epsilon}{H} = 3H\epsilon$. 
Therefore, Jensen's inequality and the convexity of $|\cdot|$ implies $\abs{ \sum_{q=1}^M w_q^k(s,a) \VV_{\PP_p(\cdot \mid s,a)}[V_p^\star] - \VV_{\PP_p(\cdot \mid s,a)}\sbr{V_p^\star} } \leq  3H\epsilon$, and 
\[
 \abs{ \sqrt{\sum_{q=1}^M w_q^k(s,a) \VV_{\PP_q(\cdot \mid s,a)}[V_p^\star]} - \sqrt{ \VV_{\PP_p(\cdot \mid s,a)}\sbr{V_p^\star} }} \lesssim \sqrt{H \epsilon}
\]

This, together with Equation~\eqref{eqn:emp-bernstein-agg-1}, implies
\[
\abs{ \sqrt{\VV_{\hat{\PP}^k(\cdot \mid s,a)}\sbr{V_p^\star}} - \sqrt{\VV_{\PP_p(\cdot \mid s,a)}\sbr{V_p^\star}} }
\lesssim 
\sqrt{H \epsilon} + H\sqrt{ \frac{ L(n^k(s,a))}{n^k(s,a)}},
\]
establishing Equation~\eqref{eqn:emp-bernstein-agg-2}. \qedhere
\end{enumerate}
\end{proof}

\subsection{Simplifying the surplus bounds}
\label{subsec:surplus-bound}

In this section, we show a distribution-dependent bound on the surplus terms, namely Lemma~\ref{lem:final-surplus-bound}, which is key to establishing our regret bound. It can be seen as an extension of Proposition B.4 of~\cite{simchowitz2019non} to our multitask setting using the \mteuler algorithm, under the $\epsilon$-dissimilarity assumption.
Before we present Lemma~\ref{lem:final-surplus-bound} (Section~\ref{subsubsec:surplus_final}), we first show and prove two auxiliary lemmas, Lemma~\ref{lem:vub-vlb} and Lemma~\ref{lem:leading-terms}.

\begin{lemma}[Bounds on $\vub_p^k - \vlb_p^k$, generalization of~\cite{simchowitz2019non}, Lemma F.8] 
\label{lem:vub-vlb}
If $E$ happens, then for all $p \in [M]$, $k \in [K]$, $h \in [H+1]$ and $s \in \Scal_h$,
\begin{equation}
(\vub_p^k - \vlb_p^k)(s)
\leq 
4 \EE\sbr{ \sum_{t=h}^H \del{  H \wedge \bindp_p^k(s_t,a_t) \wedge \baggp_p^k(s_t,a_t) } \mid s_{h} = s, \pi^k(p), \Mcal_p };
\label{eqn:fine-vub-vlb}
\end{equation}
consequently,
\begin{equation}
(\vub_p^k - \vlb_p^k)(s)
\lesssim 
H \sum_{t=h}^H \EE\sbr{  \del{ 1 \wedge \sqrt{ \frac{S L(n_p^k(s_t, a_t))}{n_p^k(s_t, a_t)} } } \mid s_{h} = s, \pi^k(p), \Mcal_p }.
\label{eqn:coarse-vub-vlb}
\end{equation}
\end{lemma}
\begin{proof}
First, Lemmas~\ref{lem:value-bounds-valid} and~\ref{lem:surplus-bounds} together imply that if $E$ holds, Equations~\eqref{eqn:strong-optimism} and~\eqref{eqn:strong-pessimism} holds for all $p,k,s,a$. 
Under this premise, we show Equation~\eqref{eqn:fine-vub-vlb} by backward induction. 

\paragraph{Base case:} for $h = H+1$, we have that LHS is $(\vub_p^k - \vlb_p^k)(\bot) = 0$ which is equal to the RHS.

\paragraph{inductive case:} Suppose Equation~\eqref{eqn:fine-vub-vlb} holds for all $s \in \Scal_{h+1}$. Now consider $s \in \Scal_h$. 
By the definitions of $\vub_p^k$ and $\vlb_p^k$, 
\begin{align*}
    & (\vub_p^k - \vlb_p^k)(s)  \\
    = & \qub_p^k(s, \pi_p^k(s)) - \qlb_p^k(s, \pi_p^k(s)) \\
    \leq & (\PP_p (\vub_p^k - \vlb_p^k)) (s,\pi_p^k(s)) + 4 (H \wedge  \bindp_p^k(s,\pi_p^k(s)) \wedge \baggp_p^k(s,\pi_p^k(s))) \\
    = & \EE\sbr{ 4 \min(H, \bindp_p^k(s,a), \baggp_p^k(s,a)) + (\vub_p^k - \vlb_p^k)(s_{h+1}) \mid s_h = s, \pi_p^k, \Mcal_p } \\
    \leq & \EE\sbr{ 4 (H \wedge \bindp_p^k(s,a) \wedge \baggp_p^k(s,a)) + \EE\sbr{ \sum_{t=h+1}^H \del{  H \wedge 2\bindp_p^k(s_t,a_t) \wedge 2\baggp_p^k(s_t,a_t) } \mid s_{h+1} } \mid s_h = s, \pi_p^k, \Mcal_p } \\
    \leq & 4 \EE\sbr{ \sum_{t=h}^H \del{  H \wedge \bindp_p^k(s_t,a_t) \wedge \baggp_p^k(s_t,a_t) } \mid s_{h} = s, \pi_p^k, \Mcal_p },
\end{align*}
where the first inequality is from Equations~\eqref{eqn:strong-optimism} and~\eqref{eqn:strong-pessimism} for $(s,a)$ and player $p$ at episode $k$, and the second inequality is from the inductive hypothesis; the third inequality is by algebra. This completes the induction.

We now show Equation~\eqref{eqn:coarse-vub-vlb}. By the definition of $\bindp_p^k(s,a)$ and algebra, 
\begin{align*}
& \bindp_p^k(s,a) \\
\lesssim &
\sqrt{ \frac{ \VV_{\hat{\PP}_p^k (\cdot \mid s,a)}\sbr{\vub_p^k} \;  L(n_p^k(s,a))}{n_p^k(s,a)} } + \sqrt{ \frac{ L(n_p^k(s,a))}{n_p^k(s,a)} } + \sqrt{ \frac{ S \| \vub_p^k - \vlb_p^k \|_{\hat{\PP}_p^k (\cdot \mid s,a)}^2 \; L(n_p^k(s,a))}{n_p^k(s,a)} } + \frac{H S  L(n_p^k(s,a))}{n_p^k(s,a)} \\
\lesssim & 
H \sqrt{ \frac{SL(n_p^k(s_t, a_t))}{n_p^k(s_t, a_t)} }+ \frac{HSL(n_p^k(s_t, a_t))}{n_p^k(s_t, a_t)},
\end{align*}
where the second inequality uses $\VV_{\hat{\PP}_p^k (\cdot \mid s,a)}\sbr{\vub_p^k} \leq H^2$ and $\| \vub_p^k - \vlb_p^k \|_{\hat{\PP}_p^k}^2 \leq H^2$. 

As a consequence, using Lemma~\ref{lem:wedge-a-sqrta},
\begin{align*}
H \wedge \bindp_p^k(s_t,a_t) \wedge \baggp_p^k(s_t,a_t)
\lesssim & H \wedge \del{ H \sqrt{ \frac{SL(n_p^k(s_t, a_t))}{n_p^k(s_t, a_t)} }+ \frac{HSL(n_p^k(s_t, a_t))}{n_p^k(s_t, a_t)} } \\
\lesssim & H \del{ 1 \wedge \sqrt{ \frac{SL(n_p^k(s_t, a_t))}{n_p^k(s_t, a_t)} } }. \qedhere
\end{align*}
\end{proof}

\begin{lemma}%
\label{lem:leading-terms}
If $E$ happens, we have the following statements holding for all $p,k,s,a$:
\begin{enumerate}[wide, labelwidth=!, labelindent=0pt]
    \item For two terms that appear in $\bindp_p^k(s,a)$, they are bounded respectively as:
        \begin{equation}
        \| \vub_p^k - \vlb_p^k \|_{\hat{\PP}_p^k(\cdot \mid s,a)}^2
        \lesssim 
        \| \vub_p^k - \vlb_p^k \|_{\PP_p(\cdot \mid s,a)}^2
        +
        \frac{H^2 S L(n_p^k(s,a)) }{n_p^k(s,a)}
        \label{eqn:ind-ell-2}
        \end{equation}
        \begin{align}
        \begin{split}
        \sqrt{ \frac{ \VV_{\hat{\PP}_p^k(\cdot \mid s,a)}\sbr{ \vub_p^k } L(n_p^k(s,a)) }{ n_p^k(s,a) } }
        \lesssim & 
        \sqrt{ \frac{ \VV_{\PP_p(\cdot \mid s,a)}\sbr{ V_p^{\pi^k} } L(n_p^k(s,a)) }{ n_p^k(s,a) } } \\
        & + 
        \sqrt{ \frac{ \| \vub_p^k - \vlb_p^k \|_{\PP_p(\cdot \mid s,a)}^2 L(n_p^k(s,a)) }{ n_p^k(s,a) } }
        + 
        \frac{H \sqrt{S} L(n_p^k(s,a))}{n_p^k(s,a)}
        \end{split}
       \label{eqn:ind-leading-var}
        \end{align}

    \item For two terms that appear in $\baggp_p^k(s,a)$, they are bounded respectively as:
        \begin{equation}
        \| \vub_p^k - \vlb_p^k \|_{\hat{\PP}^k(\cdot \mid s,a)}^2
        \lesssim 
        2 \| \vub_p^k - \vlb_p^k \|_{\PP_p(\cdot \mid s,a)}^2
        +
        \frac{H^2 S L(n_p^k(s,a)) }{n_p^k(s,a)}
        + H \epsilon
        \label{eqn:agg-ell-2}
        \end{equation}
        \begin{align}
        \sqrt{ \frac{ \VV_{\hat{\PP}^k(\cdot \mid s,a)}\sbr{ \vub_p^k } L(n^k(s,a)) }{ n^k(s,a) } }
        \lesssim & 
        \sqrt{ \frac{ \VV_{\PP_p(\cdot \mid s,a)}\sbr{ V_p^{\pi^k} } L(n^k(s,a)) }{ n^k(s,a) } }
        +
        \sqrt{ \frac{ \| \vub_p^k - \vlb_p^k \|_{\PP_p(\cdot \mid s,a)}^2 L(n^k(s,a)) }{ n^k(s,a) } }
        \nonumber \\
        & + \frac{H \sqrt{S} L(n^k(s,a))}{n^k(s,a)}
        +
        \sqrt{ \frac{ H \epsilon L(n^k(s,a)) }{ n^k(s,a) } }
        \label{eqn:agg-leading-var}
        \end{align}
\end{enumerate}

\end{lemma}
\begin{proof} 
First, Lemmas~\ref{lem:value-bounds-valid} and~\ref{lem:surplus-bounds} together imply that if $E$ happens, the value function upper and lower bounds are valid.
Conditioned on $E$ happening, we prove the two items respectively. 

\begin{enumerate}[wide, labelwidth=!, labelindent=0pt]
\item For Equation~\eqref{eqn:ind-ell-2}, using the definition of $E_{\indiv,\prob}$ and AM-GM inequality, when $E$ happens, 
we have for all $p,k,s, a, s'$, 
\begin{equation}
    \hat{\PP}_p^k(s' \mid s,a) \lesssim \PP_p(s' \mid s,a) + \frac{ L(n_p^k(s,a))}{n_p^k(s,a)}. \label{eqn:ppk-rel-bound}
\end{equation}
This implies that 
\begin{align*}
& \| \vub_p^k - \vlb_p^k \|_{\hat{\PP}_p^k(\cdot \mid s,a)}^2 \\
= & \sum_{s' \in \Scal_{h+1}} \hat{\PP}_p^k(s' \mid s,a) (\vub_p^k(s') - \vlb_p^k(s'))^2 \\
\lesssim & \sum_{s' \in \Scal_{h+1}} \PP_p^k(s' \mid s,a) (\vub_p^k(s') - \vlb_p^k(s'))^2 + \sum_{s' \in \Scal_{h+1}} \frac{ L(n_p^k(s,a))}{n_p^k(s,a)} \cdot H^2 \\
\lesssim & \| \vub_p^k - \vlb_p^k \|_{\PP_p(\cdot \mid s,a)}^2 + \frac{S  H^2 L(n_p^k(s,a))}{n_p^k(s,a)},
\end{align*}
where the first inequality is from Equation~\eqref{eqn:ppk-rel-bound}, and the fact that $\vub_p^k(s') - \vlb_p^k(s') \in [0,H]$ for any $s' \in \Scal_{h+1}$; the second inequality is by algebra.

For Equation~\eqref{eqn:ind-leading-var}, we have:
\begin{align*}
    & \sqrt{ \frac{ \VV_{\hat{\PP}_p^k(\cdot \mid s,a)}\sbr{ \vub_p^k } L(n_p^k(s,a)) }{ n_p^k(s,a) } } \\
    \lesssim & \sqrt{ \frac{ \VV_{\hat{\PP}_p^k(\cdot \mid s,a)}\sbr{ V_p^\star } L(n_p^k(s,a)) }{ n_p^k(s,a) } } + \sqrt{ \frac{ \| \vub_p^k - \vlb_p^k \|_{\hat{\PP}_p^k(\cdot \mid s,a)}^2 L(n_p^k(s,a)) }{ n_p^k(s,a) } } \\
    \lesssim & \sqrt{ \frac{ \VV_{\PP_p(\cdot \mid s,a)}\sbr{ V_p^\star } L(n_p^k(s,a)) }{ n_p^k(s,a) } } + \sqrt{ \frac{ \| \vub_p^k - \vlb_p^k \|_{\PP_p(\cdot \mid s,a)}^2 L(n_p^k(s,a)) }{ n_p^k(s,a) } } + \frac{\sqrt{S} H L(n_p^k(s,a))}{n_p^k(s,a)} \\
    \lesssim &  
    \sqrt{ \frac{ \VV_{\PP_p(\cdot \mid s,a)}\sbr{ V_p^{\pi^k} } L(n_p^k(s,a)) }{ n_p^k(s,a) } }
    +
    \sqrt{ \frac{ \| \vub_p^k - \vlb_p^k \|_{\PP_p(\cdot \mid s,a)}^2 L(n_p^k(s,a)) }{ n_p^k(s,a) } }
    + 
    \frac{\sqrt{S} H L(n_p^k(s,a))}{n_p^k(s,a)}
\end{align*}
where the first inequality is from Lemma~\ref{lem:var-x-y} and the observation that when $E$ happens, $\abs{ (\vub_p^k - V_p^\star) (s') }
\leq \abs{ (\vub_p^k - \vlb_p^k)(s') }$ for all $s' \in \Scal_{h+1}$; the second inequality is from Equation~\eqref{eqn:emp-bernstein-ind} of Lemma~\ref{lem:emp-bernstein-eps} and Equation~\eqref{eqn:ind-ell-2}; the third inequality again uses Lemma~\ref{lem:var-x-y} and the observation that when $E$ happens, $\abs{ (V_p^\star - V_p^{\pi^k}) (s') } \leq \abs{ (\vub_p^k - \vlb_p^k) (s') }$ for all $s' \in \Scal_{h+1}$. 

\item For Equation~\eqref{eqn:agg-ell-2}, using the definition of $E_{\aggre,\prob}$ and AM-GM inequality, when $E$ happens, 
we have for all $p,k,s, a, s'$, 
\begin{equation} 
\hat{\PP}^k(s' \mid s,a) \lesssim \bar{\PP}^k(s' \mid s,a) + \frac{ L(n^k(s,a))}{n^k(s,a)}. 
\label{eqn:pk-rel-bound}
\end{equation}
This implies that 
\begin{align*}
& \| \vub_p^k - \vlb_p^k \|_{\hat{\PP}^k(\cdot \mid s,a)}^2 \\
= & \sum_{s' \in \Scal_{h+1}} \hat{\PP}^k(s' \mid s,a) (\vub_p^k(s') - \vlb_p^k(s'))^2 \\
\lesssim & 2 \sum_{s' \in \Scal_{h+1}} \bar{\PP}^k(s' \mid s,a) (\vub_p^k(s') - \vlb_p^k(s'))^2 + \sum_{s' \in \Scal_{h+1}} \frac{ L(n_p^k(s,a))}{n_p^k(s,a)} \cdot H^2 \\
\lesssim & 2 \sum_{s' \in \Scal_{h+1}} \PP_p (s' \mid s,a) (\vub_p^k(s') - \vlb_p^k(s'))^2 + \epsilon H + \frac{S H^2 L(n_p^k(s,a))}{n_p^k(s,a)} \\
\lesssim & \| \vub_p^k - \vlb_p^k \|_{\PP_p(\cdot \mid s,a)}^2 + \frac{S  H^2 L(n_p^k(s,a))}{n_p^k(s,a)} + \epsilon H,
\end{align*}
where the first inequality is from Equation~\eqref{eqn:pk-rel-bound} and the fact that $\vub_p^k(s') - \vlb_p^k(s') \in [0,H]$ for any $s' \in \Scal_{h+1}$; the second inequality is from the observation that $\| \PP_p(\cdot \mid s,a) - \bar{\PP}^k(\cdot \mid s,a) \|_1 \leq \frac\epsilon H$; the third inequality is by algebra.

For Equation~\eqref{eqn:agg-leading-var}, we have:
\begin{align*}
    & \sqrt{ \frac{ \VV_{\hat{\PP}^k(\cdot \mid s,a)}\sbr{ \vub_p^k } L(n_p^k(s,a)) }{ n_p^k(s,a) } } \\
    \lesssim & \sqrt{ \frac{ \VV_{\hat{\PP}^k(\cdot \mid s,a)}\sbr{ V_p^\star } L(n_p^k(s,a)) }{ n_p^k(s,a) } } + \sqrt{ \frac{ \| \vub_p^k - \vlb_p^k \|_{\hat{\PP}^k(\cdot \mid s,a)}^2 L(n_p^k(s,a)) }{ n_p^k(s,a) } } \\
    \lesssim & \sqrt{ \frac{ \VV_{\PP_p(\cdot \mid s,a)}\sbr{ V_p^\star } L(n_p^k(s,a)) }{ n_p^k(s,a) } } + \sqrt{ \frac{ \| \vub_p^k - \vlb_p^k \|_{\PP_p(\cdot \mid s,a)}^2 L(n_p^k(s,a)) }{ n_p^k(s,a) } } + \frac{\sqrt{S} H L(n_p^k(s,a))}{n_p^k(s,a)}
    +
    \sqrt{ \frac{ H \epsilon L(n_p^k(s,a)) }{ n_p^k(s,a) } }
    \\
    \lesssim &  
    \sqrt{ \frac{ \VV_{\PP_p(\cdot \mid s,a)}\sbr{ V_p^{\pi^k} } L(n_p^k(s,a)) }{ n_p^k(s,a) } }
    +
    \sqrt{ \frac{ \| \vub_p^k - \vlb_p^k \|_{\PP_p(\cdot \mid s,a)}^2 L(n_p^k(s,a)) }{ n_p^k(s,a) } }
    + 
    \frac{\sqrt{S} H L(n_p^k(s,a))}{n_p^k(s,a)}
    +
    \sqrt{ \frac{ H \epsilon L(n_p^k(s,a)) }{ n_p^k(s,a) } },
\end{align*}
where the first inequality is from Lemma~\ref{lem:var-x-y} and the observation that when $E$ happens, $\abs{(\vub_p^k - V_p^\star)(s') }
\leq \abs{(\vub_p^k - \vlb_p^k)(s') }$ for $s' \in \Scal_{h+1}$; the second inequality uses Equation~\eqref{eqn:emp-bernstein-agg-2} of Lemma~\ref{lem:emp-bernstein-eps} and Equation~\eqref{eqn:agg-ell-2}; the third inequality is from Lemma~\ref{lem:var-x-y} and the observation that when $E$ happens, $\abs{(\vub_p^\star - V_p^{\pi^k})(s') }
\leq \abs{(\vub_p^k - \vlb_p^k)(s') }$ for $s' \in \Scal_{h+1}$. \qedhere
\end{enumerate}
\end{proof}

\subsubsection{Distribution-dependent bound on the surplus terms}
\label{subsubsec:surplus_final}

\begin{lemma}[Surplus bound]
\label{lem:final-surplus-bound}
If $E$ happens, then for all $p, k, s, a$:
\begin{align*}
E_p^k(s,a) 
\lesssim &
B_p^{k,\lead}(s,a)
+ \EE \sbr{ \sum_{t=h}^H B_p^{k, \fut}(s_t, a_t) \mid (s_h, a_h) = (s, a), \pi^k(p), \Mcal_p },
\end{align*}
where 
\begin{align*}
& B_p^{k,\lead}(s,a)
=
H \wedge \del{ 5 \epsilon + O\del{\sqrt{ \frac{\del{ 1+ \VV_{\PP_p(\cdot \mid  s,a)}[V_p^{\pi^k}]} L(n^k(s,a))}{n^k(s,a)} }} }
\\
& \qquad \qquad \qquad 
\;
\; \wedge O\del{\sqrt{ \frac{\del{ 1+\VV_{\PP_p(\cdot \mid s,a) } [V_p^{\pi^k}]} L(n_p^k(s,a))}{n_p^k(s,a)} }}, \\
& B_p^{k, \fut}(s,a) = H^3 \wedge O\del{\frac{H^3 S L(n_p^k(s,a))}{n_p^k(s,a)}}. 
\end{align*}
\end{lemma}

\begin{proof}[Proof of Lemma~\ref{lem:final-surplus-bound}]
First, Lemmas~\ref{lem:value-bounds-valid} and~\ref{lem:surplus-bounds} together imply that if $E$ holds, for all $p,k,s,a$,
$E_p^k(s,a) \leq 2\del{ H \wedge \bindp_p^k(s,a) \wedge  \baggp_p^k(s,a) }$. We now bound $\bindp_p^k(s,a)$ and $\baggp_p^k(s,a)$ respectively.

\paragraph{Bounding $\bindp_p^k(s,a)$:} We have
\begin{align*}
   & \bindp_p^k(s,a)  \\ 
   = & O\del{\sqrt{ \frac{ \VV_{\hat{\PP}_p^k (\cdot \mid s,a)}[\vub_p^k] \;  L(n_p^k(s,a))}{n_p^k(s,a)} } + \sqrt{ \frac{ L(n_p^k(s,a))}{n_p^k(s,a)} } + \sqrt{ \frac{ S \| \vub_p^k - \vlb_p^k \|_{\hat{\PP}_p^k (\cdot \mid s,a)}^2 \; L(n_p^k(s,a))}{n_p^k(s,a)} } + \frac{S H L(n_p^k(s,a))}{n_p^k(s,a)}  }\\
   \leq & O\del{ \sqrt{ \frac{ \VV_{\PP_p(\cdot \mid s,a)}[V_p^{\pi^k}] \;  L(n_p^k(s,a))}{n_p^k(s,a)} } + \sqrt{ \frac{ L(n_p^k(s,a))}{n_p^k(s,a)} } + \sqrt{ \frac{ S \| \vub_p^k - \vlb_p^k \|_{\PP_p(\cdot \mid s,a)}^2 \; L(n_p^k(s,a))}{n_p^k(s,a)} } + \frac{S H L(n_p^k(s,a))}{n_p^k(s,a)} }\\
   \leq & O\del{ \sqrt{ \frac{ \del{ 1 + \VV_{\PP_p(\cdot \mid s,a)}[V_p^{\pi^k}] } \;  L(n_p^k(s,a))}{n_p^k(s,a)} } + \sqrt{ \frac{ S \| \vub_p^k - \vlb_p^k \|_{\PP_p(\cdot \mid s,a)}^2 \; L(n_p^k(s,a))}{n_p^k(s,a)} } + \frac{S H L(n_p^k(s,a))}{n_p^k(s,a)} } \\
   \leq & O\del{ \sqrt{ \frac{ \del{ 1 + \VV_{\PP_p(\cdot \mid s,a)}[V_p^{\pi^k}] } \;  L(n_p^k(s,a))}{n_p^k(s,a)} } +  \| \vub_p^k - \vlb_p^k \|_{\PP_p(\cdot \mid s,a)}^2  + \frac{S H L(n_p^k(s,a))}{n_p^k(s,a)} } 
\end{align*}
where the first inequality is by expanding the definition of $\bindp_p^k(s,a)$ and algebra; the second inequality is from Equations Equation~\eqref{eqn:ind-ell-2} and \eqref{eqn:ind-leading-var} of Lemma~\ref{lem:leading-terms}, along with algebra; the third inequality is by the basic fact that $\sqrt{A}+\sqrt{B} \lesssim \sqrt{A+B}$; the fourth inequality is by AM-GM inequality. 

\paragraph{Bounding $\baggp_p^k(s,a)$:} We have:
\begin{align*}
   & \baggp_p^k(s,a)  \\ 
   \lesssim & 4\epsilon + O\del{ \sqrt{ \frac{ \VV_{\hat{\PP}^k(\cdot \mid s,a)}[\vub_p^k] \;  L(n^k(s,a))}{n^k(s,a)} } + \sqrt{ \frac{ L(n^k(s,a))}{n^k(s,a)} } + \sqrt{ \frac{ S \| \vub_p^k - \vlb_p^k \|_{\hat{\PP}^k(\cdot \mid s,a)}^2 \; L(n^k(s,a))}{n^k(s,a)} } + \frac{S H L(n^k(s,a))}{n^k(s,a)} } \\
   \lesssim & 5\epsilon + O\del{ \sqrt{ \frac{ \VV_{\PP_p(\cdot \mid s,a)}[V_p^{\pi^k}] \;  L(n^k(s,a))}{n^k(s,a)} } + \sqrt{ \frac{ L(n^k(s,a))}{n^k(s,a)} } + \sqrt{ \frac{ S \| \vub_p^k - \vlb_p^k \|_{\PP_p(\cdot \mid s,a)}^2 \; L(n^k(s,a))}{n^k(s,a)} } + \frac{S H L(n^k(s,a))}{n^k(s,a)} }\\
   \lesssim & 5\epsilon + O\del{ \sqrt{ \frac{ \del{ 1 + \VV_{\PP_p(\cdot \mid s,a)}[V_p^{\pi^k}] } \;  L(n^k(s,a))}{n^k(s,a)} } + \sqrt{ \frac{ S \| \vub_p^k - \vlb_p^k \|_{\PP_p(\cdot \mid s,a)}^2 \; L(n^k(s,a))}{n^k(s,a)} } + \frac{S H L(n^k(s,a))}{n^k(s,a)} } \\
   \leq & 5\epsilon + O\del{ \sqrt{ \frac{ \del{ 1 + \VV_{\PP_p(\cdot \mid s,a)}[V_p^{\pi^k}] } \;  L(n^k(s,a))}{n^k(s,a)} } +  \| \vub_p^k - \vlb_p^k \|_{\PP_p(\cdot \mid s,a)}^2  + \frac{S H L(n^k(s,a))}{n^k(s,a)}}
\end{align*}
where the first inequality is by expanding the definition of $\baggp_p^k(s,a)$ and algebra; the second inequality is from Equations~\eqref{eqn:agg-leading-var} and Equation~\eqref{eqn:agg-ell-2} of Lemma~\ref{lem:leading-terms}, along with the observation that 
$\sqrt{ \frac{ S \epsilon H L(n^k(s,a))}{n^k(s,a)} } \leq \frac{S H L(n^k(s,a))}{n^k(s,a)} + \epsilon$ by AM-GM inequality; the third inequality is by the basic fact that $\sqrt{A}+\sqrt{B} \lesssim \sqrt{A+B}$; the fourth inequality is from AM-GM inequality. 

Combining the above upper bounds, and using the observation that 
$\frac{L(n^k(s,a))}{n^k(s,a)} \leq \frac{L(n_p^k(s,a))}{n_p^k(s,a)}$,
we get
\begin{align*}
& \bindp_p^k(s,a) \wedge \baggp_p^k(s,a) \wedge H \\
\leq &  
O\del{ \sqrt{ \frac{ \del{ 1 + \VV_{\PP_p(\cdot \mid s,a)}[V_p^{\pi^k}] } \;  L(n_p^k(s,a))}{n_p^k(s,a)} } }
\wedge 
\del{ 5 \epsilon + O\del{ \sqrt{ \frac{ \del{ 1 + \VV_{\PP_p(\cdot \mid s,a)}[V_p^{\pi^k}] } \;  L(n^k(s,a))}{n^k(s,a)} } } }
\wedge
H \\
& + 
O\del{ \| \vub_p^k - \vlb_p^k \|_{\PP_p(\cdot \mid s,a)}^2  + \del{ \frac{S H L(n_p^k(s,a))}{n_p^k(s,a)} \wedge H} } \\
\leq & B^{k,\lead}(s,a) 
+  
O \del{ \| \vub_p^k - \vlb_p^k \|_{\PP_p(\cdot \mid s,a)}^2  + \del{ \frac{S H L(n_p^k(s,a))}{n_p^k(s,a)} \wedge H} }.
\end{align*}

We now show that 
\begin{equation}
    \| \vub_p^k - \vlb_p^k \|_{\PP_p(\cdot \mid s,a)}^2  + \del{ \frac{S H L(n_p^k(s,a))}{n_p^k(s,a)} \wedge H} \lesssim \EE \sbr{ \sum_{t=h}^H B^{k, \fut}(s_t, a_t) \mid (s_h, a_h) = (s, a), \pi^k(p), \Mcal_p },
    \label{eqn:surplus-future}
\end{equation} 
which will conclude the proof.
To this end, we simplify the left hand side of Equation~\eqref{eqn:surplus-future} using Lemma~\ref{lem:vub-vlb}:
\begin{align*}
    & \| \vub_p^k - \vlb_p^k \|_{\PP_p(\cdot \mid s,a)}^2  + \del{ \frac{S H L(n^k(s,a))}{n^k(s,a)} \wedge H} \\
    \lesssim & 
    \EE\sbr{  \del{ H \sum_{t=h+1}^H \EE\sbr{  \del{ 1 \wedge \sqrt{ \frac{S L(n_p^k(s_t, a_t))}{n_p^k(s_t, a_t)} } } \mid s_{h+1} } }^2 \mid (s_{h}, a_h) = (s,a), \pi^k(p), \Mcal_p } + \del{ \frac{S H L(n^k(s,a))}{n^k(s,a)} \wedge H} \\
    \lesssim & 
    H^3 \EE\sbr{   \sum_{t=h+1}^H \EE\sbr{  \del{ 1 \wedge \sqrt{ \frac{S L(n_p^k(s_t, a_t))}{n_p^k(s_t, a_t)} } }^2 \mid s_{h+1} }  \mid (s_{h}, a_h) = (s,a), \pi^k(p), \Mcal_p } + \del{ \frac{S H L(n^k(s,a))}{n^k(s,a)} \wedge H}\\
    \lesssim & 
     \EE\sbr{ \sum_{t=h}^H H^3 \wedge  \frac{H^3 S L(n_p^k(s_t, a_t))}{n_p^k(s_t, a_t)} \mid (s_{h}, a_h) = (s,a), \pi^k(p), \Mcal_p }  \\
     \lesssim & 
     \EE\sbr{ \sum_{t=h}^H B^{k, \fut}(s_t,a_t) \mid (s_{h}, a_h) = (s,a), \pi^k(p), \Mcal_p }, 
     \end{align*}
where the first inequality is from Equation~\eqref{eqn:coarse-vub-vlb} of Lemma~\ref{lem:vub-vlb}; the second inequality is by Cauchy-Schwarz and $\EE[X]^2 \leq \EE[X^2]$ for any random variable $X$; and the third inequality is by the law of total expectation and algebra. 
\end{proof}

\subsection{Concluding the regret bounds}
\label{subsec:conclude-reg-bounds}

In this section, we present the proofs of Theorems~\ref{thm:gap_indept_upper} and~\ref{thm:gap_dept_upper}. 

To bound the collective regret of \mteuler,
we first recall the following general result from~\cite{simchowitz2019non}, which is useful to establish instance-dependent regret guarantees for episodic RL. 

\begin{lemma}[Clipping lemma, \cite{simchowitz2019non}, Lemma B.6]
\label{lem:clipping-main}
Fix player $p \in [M]$; suppose for each episode $k$, it follows $\pi^k(p)$, the greedy policy with respect to $\qub_p^k$.
In addition, there exists some event $E$ and a collection of functions $\cbr{ B_p^{k,\lead}, B_p^{k,\fut}}_{k \in [K]} \subset  (\Scal \times \Acal \to \RR)$ , such that if $E$ happens, then for all $k \in [K]$, $h \in [H]$ and $(s,a) \in \Scal_h \times \Acal$, the surplus of $\qub_p^k$ satisfies that
\begin{align*}
0 \leq 
E_p^k(s,a) 
\lesssim 
B_p^{k,\lead}(s,a)
+ \EE \sbr{ \sum_{t=h}^H B_p^{k, \fut}(s_t, a_t) \mid (s_h, a_h) = (s, a), \pi^k(p), \Mcal_p },
\end{align*}
then, on $E$:
\[ 
\Reg(K, p)
\lesssim
\sum_{s,a}
\sum_{k} 
\rho_p^k(s,a) \clip\del{ B_p^{k,\lead}(s,a), \check{\gap}_p(s,a) }
+ 
H \sum_{s,a}
\sum_{k} 
\rho_p^k(s,a)
\clip\del{ B_p^{k, \fut}(s,a), \frac{\gap_{p,\min}}{8SAH^2} },
\]
here, recall that $\clip(\alpha, \Delta) = \alpha \one(\alpha \geq \Delta)$, and 
$
\check{\gap}_p(s,a)
=
\frac{\gap_p(s,a)}{4H} \vee \frac{\gap_{p,\min}}{4H}
$.
\end{lemma}

\begin{remark}
Our presentation of the clipping lemma is slightly different than the original one \cite[Lemma B.6]{simchowitz2019non}, in that:
\begin{enumerate}
    \item We consider layered MDPs, while \cite{simchowitz2019non} consider general stationary MDPs where one state may be experienced at multiple different steps in $[H]$. 
Specifically, in a layered MDP, the occupancy distributions $\omega_{k,h}$ defined in \cite{simchowitz2019non} is only supported over $\Scal_h \times \Acal$. As a result, in the presentation here, we no longer need to sum over $h$ -- this is already captured in the sum over all $s$ across all layers.
    \item Our presentation here is in the context of multitask RL, which is with respect to a player $p \in [M]$, its corresponding MDP $\Mcal_p$, and its policies used throughout the process $\cbr{\pi^k(p)}_{k=1}^K$. As a result, all quantities have $p$ as subscripts.
    
    \item Since every (state, action, step) in every MDP is trivially 1-transition optimal (see Definition B.3 of~\cite{simchowitz2019non}), when applying Lemma B.6 to $\Mcal_p$, we set $\alpha_{s,a,h} = 1$ and $\check{gap}_p(s,a) = \frac{\gap_{p,\min}}{2H} \vee \frac{\gap_p(s,a)}{4H}$. 
\end{enumerate}
\label{rem:clipping}
\end{remark}

We are now ready to prove Theorems~\ref{thm:gap_indept_upper} and~\ref{thm:gap_dept_upper}, \mteuler's main regret theorems.

\subsubsection{Proof of Theorem~\ref{thm:gap_indept_upper}}

\begin{proof}[Proof of Theorem~\ref{thm:gap_indept_upper}]
From Lemma~\ref{lem:clipping-main} and Lemma~\ref{lem:final-surplus-bound}, we have that when $E$ happens,
\begin{align}
\begin{split}
\Reg(K)
= 
& \sum_{p=1}^M \Reg(K, p) \\
\leq &
\underbrace{\sum_{s,a}
\sum_{k,p} 
\rho_p^k(s,a) \clip\del{ B^{k,\lead}(s,a), \check{\gap}_p(s,a) }}_{(A)}
+ 
\underbrace{
H \sum_{s,a}
\sum_{k,p} 
\rho_p^k(s,a)
\clip\del{ B^{k, \fut}(s,a), \frac{\gap_{p,\min}}{8SAH^2} }}_{(B)},
\end{split}
\label{eqn:reg-a-b-gap-indept}
\end{align}

We bound each term separately.
We can directly use Lemma~\ref{lem:lower-order-regret} to bound term $(B)$ as:
\begin{equation}
H \sum_{s,a}
\sum_{k,p} 
\rho_p^k(s,a)
\clip\del{ B^{k, \fut}(s,a), \frac{\gap_{p,\min}}{8SAH^2} }
\lesssim
M H^4 S^2 A \del{ \ln\del{ \frac{MSAK}{\delta} } }^2.
\label{eqn:term-b-gap-indept}
\end{equation}

For term $(A)$, we will group the sum by $(s,a) \in \Ical_\epsilon$ and $(s,a) \notin \Ical_\epsilon$ separately. 

\paragraph{Case 1: $(s,a) \in \Ical_\epsilon$.} In this case, we have that for all $p$, $\check{\gap}_p(s,a) = \frac{\gap_p(s,a)}{4H} \geq 24 \epsilon$. We simplify the corresponding term as follows:
\begin{align*}
& \sum_{(s,a) \in \Ical_\epsilon} \sum_{k,p} 
\rho_p^k(s,a) \clip\del{ B^{k,\lead}(s,a), \check{\gap}_p(s,a) } \\
\leq & 
\sum_{(s,a) \in \Ical_\epsilon} \sum_{k,p} 
\rho_p^k(s,a) \clip\del{ H \wedge \del{ 5\epsilon + O\del{ \sqrt{ \frac{ (1+\VV_{\PP_p(\cdot \mid  s,a)}[V_p^{\pi^k}]) L(n^k(s,a)) }{n^k(s,a)} }} } , \frac{\min_p \gap_p(s,a)}{4H} } \\
\leq &
\sum_{(s,a) \in \Ical_\epsilon} \sum_{k,p} \rho_p^k(s,a)
 \del{ H \wedge \clip\del{ 5 \epsilon + O\del{ \sqrt{ \frac{ (1 + \VV_{\PP_p(\cdot \mid  s,a)}[V_p^{\pi^k}]) L(n^k(s,a)) }{n^k(s,a)} }}, \frac{\min_p \gap_p(s,a)}{4H} } } \\
\lesssim &
\sum_{(s,a) \in \Ical_\epsilon} \sum_{k,p} \rho_p^k(s,a) \del{ H \wedge   \sqrt{ \frac{ (1 + \VV_{\PP_p(\cdot \mid  s,a)}[V_p^{\pi^k}]) L(n^k(s,a)) }{n^k(s,a)} } }
\end{align*}
where the first inequality is from the definition of $B^{k,\lead}$; the second inequality is from the basic fact that $\clip(A \wedge B, C) \leq A \wedge \clip(B, C)$; 
the third inequality uses Lemma~\ref{lem:clipping-decomp} with $a_1 = 5\epsilon$, $a_2 = \sqrt{ \frac{ (1 + \VV_{\PP_p(\cdot \mid  s,a)}[V_p^{\pi^k}]) L(n^k(s,a)) }{n^k(s,a)}}$, and $\Delta = \frac{\min_p \gap_p(s,a)}{4H} $, along with the observation that $\clip(5\epsilon, \frac{\min_p \gap_p(s,a)}{16H} ) = 0$, since for all $(s,a) \in \Ical_\epsilon$ and all $p \in [M]$, $\gap_p(s,a) \geq 96\epsilon H$.

We now decompose the inner sum over $k$, $\sum_{k=1}^K$, to $\sum_{k=1}^{\tau(s,a)-1}$ 
and $\sum_{k=\tau(s,a)}^{K}$. The first part is bounded by:
\begin{align*}
\sum_{(s,a) \in \Ical_\epsilon}
\sum_{k=1}^{\tau_p(s,a)-1} \sum_{p=1}^M
\rho_p^k(s,a) \del{ H \wedge    \sqrt{ \frac{ (1 + \VV_{\PP_p(\cdot \mid  s,a)}[V_p^{\pi^k}]) L(n^k(s,a)) }{n^k(s,a)} } }
\leq &
\sum_{(s,a) \in \Ical_\epsilon}
\sum_{k=1}^{\tau_p(s,a)-1} \sum_{p=1}^M 
\rho_p^k(s,a) H  
\leq S A H N_1,
\end{align*}
which is $\lesssim M H S A  \ln\del{\frac{SAK}{\delta}}$.

For the second part,
\begin{align*}
& \sum_{(s,a) \in \Ical_\epsilon} \sum_{k=\tau(s,a)}^K \sum_{p=1}^M
\rho_p^k(s,a) \del{ H \wedge   \sqrt{  \frac{(1+\VV_{\PP_p(\cdot \mid  s,a)}[V_p^{\pi^k}]) L(n^k(s,a)) }{n^k(s,a)} } } \\
\lesssim & \sum_{(s,a) \in \Ical_\epsilon} \sum_{k=\tau(s,a)}^K \sum_{p=1}^M
\rho_p^k(s,a) \sqrt{ \frac{ (1 + \VV_{\PP_p(\cdot \mid  s,a)}[V_p^{\pi^k}]) L(\bar{n}^k(s,a)) }{\bar{n}^k(s,a)} } \\
\lesssim & \sqrt{ \sum_{(s,a) \in \Ical_\epsilon} \sum_{k=\tau(s,a)}^K \sum_{p=1}^M \rho_p^k(s,a) \cdot \frac{L(\bar{n}^k(s,a))}{\bar{n}^k(s,a)} } \cdot \sqrt{ \sum_{(s,a) \in \Ical_\epsilon} \sum_{k=1}^K \sum_{p=1}^M
\rho_p^k(s,a) \del{ 1 + \VV_{\PP_p(\cdot \mid  s,a)}[V_p^{\pi^k}] } },
\end{align*}
where the first inequality is by dropping the ``$H \wedge$'' operator; the second inequality is by Cauchy-Schwarz.

We bound each factor as follows: for the first factor,
\begin{align*}
\sum_{(s,a) \in \Ical_\epsilon} \sum_{k=\tau(s,a)}^K \sum_{p=1}^M \rho_p^k(s,a) \cdot \frac{L(\bar{n}^k(s,a))}{\bar{n}^k(s,a)} 
= & 
\sum_{(s,a) \in \Ical_\epsilon} \sum_{k=\tau(s,a)}^K  \rho^k(s,a) \cdot \frac{L(\bar{n}^k(s,a))}{\bar{n}^k(s,a)} \\
\leq & 
L(M K) \sum_{(s,a) \in \Ical_\epsilon}  \sum_{k=\tau(s,a)}^K  \frac{\rho^k(s,a)}{\bar{n}^k(s,a)}  \\
\leq & 
\sum_{(s,a) \in \Ical_\epsilon} L(M K) \cdot \int_{1}^{\bar{n}^K(s,a)} \frac{1}{u} du \\
\leq & \abs{\Ical_\epsilon} L(M K)^2 
\lesssim \abs{\Ical_\epsilon} \del{\ln\del{\frac{MSAK}{\delta}}}^2,
\end{align*}
where the first inequality is because $L$ is monotonically increasing, and $\bar{n}^k(s,a) \leq M K$; the second inequality is from the observation that $\rho^k(s,a) \in [0,M]$, $\bar{n}^k(s,a) \geq 2M$, and $u \mapsto \frac1u$ is monotonically decreasing; the last two inequalities are by algebra.

For the second factor, 
\begin{align}
\begin{split}
\sum_{(s,a) \in \Ical_\epsilon} \sum_{k=1}^K \sum_{p=1}^M
\rho_p^k(s,a) \del{ 1 + \VV_{\PP_p(\cdot \mid  s,a)}[V_p^{\pi^k}] } 
\lesssim & MKH + \sum_{p=1}^M \sum_{k=1}^K \sum_{(s,a) \in \Scal \times \Acal} 
\rho_p^k(s,a) \VV_{\PP_p(\cdot \mid  s,a)}[V_p^{\pi^k}]  \\
\lesssim & M K H + \sum_{p=1}^M \sum_{k=1}^K \Var \sbr{ \sum_{h=1}^H r_{h,p}^k \mid \pi^k(p)} \\
\lesssim & M K H^2.
\end{split}
\label{eqn:ltv}
\end{align}
where the first inequality is by the fact that $\rho_p^k$ are probability distributions over every layer $h \in [H]$; the last two inequalities are by a law of total variance identity (see, e.g., \cite[Equation (26)]{azar2017minimax}). To summarize, the second part is at most 
\[
\sum_{(s,a) \in \Ical_\epsilon} \sum_{k=\tau(s,a)}^K \sum_{p=1}^M
\rho_p^k(s,a) \del{ H \wedge   \sqrt{ \frac{(1+\VV_{\PP_p(\cdot \mid  s,a)}[V_p^{\pi^k}]) L(n^k(s,a)) }{n^k(s,a)} } } 
\lesssim 
\sqrt{ M K H^2 |\Ical_\epsilon| } \ln\del{\frac{MSAK}{\delta}}.
\]
Combining the bounds for the first and the second parts, we have:
\[
\sum_{(s,a) \in \Ical_\epsilon} \sum_{k,p} 
\rho_p^k(s,a) \clip\del{ B^{k,\lead}(s,a), \check{\gap}_p(s,a) }
\lesssim \del{ \sqrt{  M K H^2 |\Ical_\epsilon| } +M H S A } \ln\del{\frac{MSAK}{\delta}}.
\]

\paragraph{Case 2: $(s,a) \notin \Ical_\epsilon$.} We simplify the corresponding term as follows:
\begin{align*}
& \sum_{(s,a) \notin \Ical_\epsilon} \sum_{k,p} 
\rho_p^k(s,a) \clip\del{ B^{k,\lead}(s,a), \check{\gap}_p(s,a) } \\
\lesssim & 
\sum_{(s,a) \notin \Ical_\epsilon} \sum_{k,p} 
\rho_p^k(s,a) \clip\del{ H \wedge \del{ \sqrt{ \frac{\del{1+\VV_{\PP_p(\cdot \mid  s,a)}[V_p^{\pi^k}]} L(n_p^k(s,a)) }{n_p^k(s,a)} }} , \frac{\check{\gap}_p(s,a)}{4H} } \\
\lesssim &
\sum_{(s,a) \notin \Ical_\epsilon} \sum_{k,p}  \del{ H \wedge   \sqrt{ \frac{\del{ 1 +\VV_{\PP_p(\cdot \mid  s,a)}[V_p^{\pi^k}]} L(n_p^k(s,a)) }{n_p^k(s,a)} } }
\end{align*}

For each $p$ and $(s,a)$, we now decompose the inner sum over $k$, $\sum_{k=1}^K$, to $\sum_{k=1}^{\tau_p(s,a)-1}$ and $\sum_{k=\tau_p(s,a)}^{K}$. The first part is bounded by:
\begin{align*}
\sum_{(s,a) \notin \Ical_\epsilon}
\sum_{p=1}^M
\sum_{k=1}^{\tau_p(s,a)-1} 
\rho_p^k(s,a) \del{ H \wedge    \sqrt{ \frac{(1+\VV_{\PP_p(\cdot \mid  s,a)}[V_p^{\pi^k}]) L(n_p^k(s,a)) }{n_p^k(s,a)} } }
\leq &
\sum_{(s,a) \notin \Ical_\epsilon}
\sum_{p=1}^M  \sum_{k=1}^{\tau_p(s,a)-1} 
\rho_p^k(s,a) H  \\
\leq & M H S A N_2,
\end{align*}
which is $\lesssim M H S A  \ln\del{ \frac{MSAK}{\delta}}$.

For the second part,
\begin{align*}
& \sum_{(s,a) \notin \Ical_\epsilon} \sum_{p=1}^M \sum_{k=\tau_p(s,a)}^K 
\rho_p^k(s,a) \del{ H \wedge   \sqrt{ \frac{\del{ 1+\VV_{\PP_p(\cdot \mid  s,a)}[V_p^{\pi^k}]} L(n_p^k(s,a)) }{n_p^k(s,a)} } } \\
\lesssim & \sum_{(s,a) \notin \Ical_\epsilon} \sum_{p=1}^M \sum_{k=\tau_p(s,a)}^K 
\rho_p^k(s,a) \sqrt{ \frac{\del{ 1+\VV_{\PP_p(\cdot \mid  s,a)}[V_p^{\pi^k}]} L(\bar{n}_p^k(s,a)) }{\bar{n}_p^k(s,a)} } \\
\leq & \sqrt{ \sum_{(s,a) \notin \Ical_\epsilon} \sum_{p=1}^M \sum_{k=\tau_p(s,a)}^K  \rho_p^k(s,a) \cdot \frac{L(\bar{n}_p^k(s,a))}{\bar{n}_p^k(s,a)} } \cdot \sqrt{ \sum_{(s,a) \notin \Ical_\epsilon} \sum_{k=1}^K \sum_{p=1}^M
\rho_p^k(s,a) \del{ 1 + \VV_{\PP_p(\cdot \mid  s,a)}[V_p^{\pi^k}] } }
\end{align*}

We bound each factor as follows: for the first factor,
\begin{align*}
\sum_{(s,a) \notin \Ical_\epsilon} \sum_{p=1}^M \sum_{k=\tau_p(s,a)}^K  \rho_p^k(s,a) \cdot \frac{L(\bar{n}^k(s,a))}{\bar{n}^k(s,a)}
\leq & 
L(K) \cdot \sum_{(s,a) \notin \Ical_\epsilon} \sum_{p=1}^M \sum_{k=\tau_p(s,a)}^K  \frac{\rho_p^k(s,a)}{\bar{n}^k(s,a)} \\
\leq & 
L(K) \cdot \sum_{(s,a) \notin \Ical_\epsilon} \sum_{p=1}^M \int_{1}^{\bar{n}_p^K(s,a)} \frac{1}{u} du \\
\leq & \abs{\Ical_\epsilon^C} M L(K)^2 
\leq \abs{\Ical_\epsilon^C} M \del{\ln\del{\frac{MSAK}{\delta}}}^2.
\end{align*}
where the first inequality is because $L$ is monotonically increasing, and $\bar{n}_p^k(s,a) \leq K$; the second inequality is from the observation that $\rho^k(s,a) \in [0,1]$, $\bar{n}^k(s,a) \geq 2$, and $u \mapsto \frac1u$ is monotonically decreasing; the last two inequalities are by algebra.

The second factor is again bounded by~\eqref{eqn:ltv}.
Therefore, the second part of the sum is at most 
\begin{align*}
& \sum_{(s,a) \notin \Ical_\epsilon} \sum_{p=1}^M \sum_{k=\tau_p(s,a)}^K 
\rho_p^k(s,a) \del{ H \wedge   \sqrt{ \frac{\del{ 1+\VV_{\PP_p(\cdot \mid  s,a)}[V_p^{\pi^k}]} L(n_p^k(s,a)) }{n_p^k(s,a)} } } \\ 
\leq & 
\del{ 
M \sqrt{K H^2 \abs{\Ical_\epsilon^C}} + MHSA } \ln\del{\frac{MSAK}{\delta}}.
\end{align*}
Combining the bounds for the first and the second parts, we have:
\[
\sum_{(s,a) \notin \Ical_\epsilon} \sum_{k,p} 
\rho_p^k(s,a) \clip\del{ B^{k,\lead}(s,a), \check{\gap}_p(s,a) }
\lesssim 
\del{ 
M \sqrt{K H^2 \abs{\Ical_\epsilon^C}} + MHSA } \ln\del{\frac{MSAK}{\delta}}.
\]

Now, combining the bounds for cases 1 and 2, we have that 
\begin{equation}
(A) \leq \del{ 
\sqrt{M K H^2 \abs{\Ical_\epsilon}}
+
M \sqrt{K H^2 \abs{\Ical_\epsilon^C}} + MHSA } \cdot \ln\del{\frac{MSAK}{\delta}}.
\label{eqn:term-a-gap-indept}
\end{equation}

In conclusion, by the regret decomposition Equation~\eqref{eqn:reg-a-b-gap-indept}, and Equations~\eqref{eqn:term-a-gap-indept} and~\eqref{eqn:term-b-gap-indept}, we have:
\[
\Reg(K) 
\leq 
\del{ 
\sqrt{M H^2 \abs{\Ical_\epsilon} K}
+
M \sqrt{H^2 \abs{\Ical_\epsilon^C} K} + M H^4 S^2 A \ln\del{\frac{MSAK}{\delta}} } \ln\del{\frac{MSAK}{\delta}}.
\qedhere
\]
\end{proof}

\subsubsection{Proof of Theorem~\ref{thm:gap_dept_upper}} 

\begin{proof}[Proof of Theorem~\ref{thm:gap_dept_upper}]
From Lemma~\ref{lem:clipping-main}, we have that when $E$ happens,
\begin{align*}
\begin{split}
\Reg(K)
= &
\sum_{p=1}^M \Reg(K, p) \\
\leq & 
\underbrace{\sum_{s,a}
\sum_{k,p} 
\rho_p^k(s,a) \clip\del{ B^{k,\lead}(s,a), \check{\gap}_p(s,a) }}_{(A)}
+ 
\underbrace{
H
\sum_{s,a}
\sum_{k,p} 
\rho_p^k(s,a)
\clip\del{ B^{k, \fut}(s,a), \frac{\gap_{p,\min}}{8SAH^2} }}_{(B)},
\end{split}
\end{align*}

We focus on each term separately.
We directly use Lemma~\ref{lem:lower-order-regret} to bound term (B) as:
\begin{equation}
H \sum_{s,a}
\sum_{k,p} 
\rho_p^k(s,a)
\clip\del{ B^{k, \fut}(s,a), \frac{\gap_{p,\min}}{8SAH^2} }
\lesssim
M H^4 S^2 A \ln\del{ \frac{MSAK}{\delta}} \cdot \ln\frac{MSA}{\gap_{\min}} .
\label{eqn:b-gap-dept}
\end{equation}

For the $(s,a)$-th term in term (A), we will consider the cases of $(s,a) \in \Ical_\epsilon$ and $(s,a) \notin \Ical_\epsilon$ separately.

\paragraph{Case 1: $(s,a) \in \Ical_\epsilon$.} In this case, we have that for all $p$, $\check{\gap}_p(s,a) = \frac{\gap_p(s,a)}{4H} \geq 24 \epsilon$. We simplify the corresponding term as follows:
\begin{align*}
& \sum_{k,p} 
\rho_p^k(s,a) \clip\del{ B^{k,\lead}(s,a), \check{\gap}_p(s,a) } \\
\leq & 
\sum_{k=1}^K \sum_{p=1}^M 
\rho_p^k(s,a) \clip\del{ H \wedge \del{ 5\epsilon + O\del{ \sqrt{ \frac{ (1 + \VV_{\PP_p(\cdot \mid  s,a)}[V_p^{\pi^k}]) L(n^k(s,a)) }{n^k(s,a)} }}} , \frac{\min_p \gap_p(s,a)}{4H} } \\
\leq & 
\sum_{k=1}^K 
\rho^k(s,a) \clip\del{ H \wedge \del{ 5\epsilon + O\del{ \sqrt{ \frac{H^2 L(n^k(s,a)) }{n^k(s,a)} }} }, \frac{\min_p \gap_p(s,a)}{4H} } \\
\leq &
\sum_{k=1}^k 
\rho^k(s,a) \del{ H \wedge \clip\del{ 5\epsilon + O\del{ \sqrt{ \frac{H^2 L(n^k(s,a)) }{n^k(s,a)} }}, \frac{\min_p \gap_p(s,a)}{4H} } } \\
\lesssim &
\sum_{k=1}^K 
\rho^k(s,a) \del{ H \wedge   \clip\del{ \sqrt{ \frac{H^2 L(n^k(s,a)) }{n^k(s,a)} }, \frac{\min_p \gap_p(s,a)}{16H} } }
\end{align*}

where the first inequality is by the definition of $B^{k,\lead}$; the second inequality is from that $\VV_{\PP_p(\cdot \mid  s,a)}[V_p^{\pi^k}] \leq H^2$; the third inequality is from that $\clip(A \wedge B, C) \leq A \wedge \clip(B, C)$; 
the third inequality uses Lemma~\ref{lem:clipping-decomp} with $a_1 = 5\epsilon$, $a_2 = \sqrt{ \frac{H^2 L(n^k(s,a)) }{n^k(s,a)} }$, and $\Delta = \frac{\min_p \gap_p(s,a)}{4H} $, along with the observation that $\clip(5\epsilon, \frac{\min_p \gap_p(s,a)}{16H} ) = 0$, since for all $(s,a) \in \Ical_\epsilon$ and all $p \in [M]$, $\gap_p(s,a) \geq 96\epsilon H$.

We now decompose the inner sum over $k$, $\sum_{k=1}^K$, to $\sum_{k=1}^{\tau(s,a)-1}$ and $\sum_{k=\tau(s,a)}^{K}$. 
The first part's contribution is at most $N_1 \cdot H \lesssim M H \ln\del{ \frac{SAK}{\delta} }$.
For the second part, its contribution is at most:
\begin{align*}
& \sum_{k=\tau(s,a)}^K 
\rho^k(s,a) \del{ H \wedge   \clip\del{ \sqrt{ \frac{H^2 L(n^k(s,a)) }{n^k(s,a)} }, \frac{\min_p \gap_p(s,a)}{16H} } } \\
\lesssim & M H + \int_1^{\bar{n}^K(s,a)} \del{ H \wedge   \clip\del{ \sqrt{ \frac{H^2 L(u) }{u} }, \frac{\min_p \gap_p(s,a)}{16H} } } du
\\
\lesssim & 
M H + \frac{H^3}{\min_p \gap_p(s,a)} \ln\del{ \frac{MSAK}{\delta} }
\end{align*}
where the second inequality is from Lemma~\ref{lem:integral} with $f_{\max} = H$, $C = H^2$, $\Delta = \frac{\min_p \gap_p(s,a)}{16H}$, $N = MSA$, $\xi = \delta$, $\Gamma = 1$, $n = \bar{n}^K(s,a) \leq K$. In summary, for all $(s,a) \in \Ical_\epsilon$,
\[
\sum_{k,p} 
\rho_p^k(s,a) \clip\del{ B^{k,\lead}(s,a), \check{\gap}_p(s,a) }
\leq 
\del{ MH + \frac{H^3}{\min_p \gap_p(s,a)} } \ln\del{\frac{MSAK}{\delta}}.
\]

\paragraph{Case 2: $(s,a) \notin \Ical_\epsilon$.} In this case, for each $p \in [M]$, we simplify the corresponding term as follows:
\begin{align*}
& \sum_{k} 
\rho_p^k(s,a) \clip\del{ B^{k,\lead}(s,a), \check{\gap}_p(s,a) } \\
\lesssim & 
\sum_{k=1}^K 
\rho_p^k(s,a) \del{ H \wedge   \clip\del{ \sqrt{ \frac{H^2 L(n_p^k(s,a)) }{n_p^k(s,a)} }, \frac{\check{\gap}_p(s,a)}{16H} } }
\end{align*}

We now decompose the inner sum over $k$, $\sum_{k=1}^K$, to $\sum_{k=1}^{\tau_p(s,a)-1}$ and $\sum_{k=\tau_p(s,a)}^{K}$. 
The first part's contribution is at most $N_2 \cdot H \lesssim H \ln\del{\frac{MSAK}{\delta}}$.

For the second part, its contribution is at most:
\begin{align*}
& \sum_{k=\tau_p(s,a)}^K 
\rho_p^k(s,a) \del{ H \wedge   \clip\del{ \sqrt{ \frac{H^2 L(n^k(s,a)) }{n^k(s,a)} }, \frac{\check{\gap}_p(s,a)}{16H} } } \\
\lesssim & H + \int_1^{\bar{n}_p^K(s,a)} \del{ H \wedge   \clip\del{ \sqrt{ \frac{H^2 L(u) }{u} }, \frac{\check{\gap}_p(s,a)}{16H} } } du
\\
\lesssim & 
H + \frac{H^3}{\check{\gap}_p(s,a)} \ln\del{\frac{MSAK}{\delta}}
\end{align*}
where the second inequality is from Lemma~\ref{lem:integral} with $f_{\max} = H$, $C = H^2$, $\Delta = \frac{ \check{\gap}_p(s,a)}{16H}$, $N = MSA$, $\xi = \delta$, $\Gamma = 1$, $n = \bar{n}_p^K(s,a) \leq K$. In summary, for any $(s,a) \in \Ical_{\epsilon}^C$ and $p \in [M]$,
\[
\sum_{k} 
\rho_p^k(s,a) \clip\del{ B^{k,\lead}(s,a), \check{\gap}_p(s,a) }
\lesssim
(H + \frac{H^3}{\check{\gap}_p(s,a)}) \ln\del{\frac{MSAK}{\delta}},
\]
summing over $p$, we get:
\[
\sum_{k,p} 
\rho_p^k(s,a) \clip\del{ B^{k,\lead}(s,a), \check{\gap}_p(s,a) }
\lesssim 
\del{ MH + \sum_{p=1}^M \frac{H^3}{ \check{\gap}_p(s,a)} } \ln\del{\frac{MSAK}{\delta}},
\]

In summary, combining the regret bounds of cases 1 and 2 for term $(A)$, along with Equation~\eqref{eqn:b-gap-dept} for term $(B)$, and observe that $\check{\gap}_p(s,a) = \gap_{p,\min}$ if $(s,a) \in Z_{p,\opt}$, and $\check{\gap}_p(s,a) = \gap_{p}(s,a)$ otherwise,
we have that on event $E$, \mteuler satisfies:
\begin{align*}
\Reg(K) 
\lesssim 
\ln\del{\frac{M S A K}{\delta}} & \vast(
\sum_{p \in [M]} \rbr{ 
\sum_{(s,a) \in Z_{p,\opt}} \frac{ H^3}{\gap_{p,\min}}
+
\sum_{(s,a) \in (\Ical_\epsilon \cup Z_{p,\opt})^C} \frac{H^3}{\gap_p(s,a)} }
+ \\
& \sum_{(s,a) \in \Ical_\epsilon} \frac{H^3}{\min_p \gap_p(s,a)}
\vast) 
+  \ln\del{\frac{M S A K}{\delta}} \cdot M S^2 A H^4 \ln\frac{MSA}{\gap_{\min}}. \qedhere
\end{align*}
\end{proof}

\begin{lemma}[Bounding the lower order terms]
\label{lem:lower-order-regret}
If $E$ happens, then
\[
\sum_{s,a}
\sum_{k,p} 
\rho_p^k(s,a)
\clip\del{ B^{k, \fut}(s,a), \frac{\gap_{p,\min}}{8SAH^2} }
\lesssim
M H^3 S^2 A \ln\del{ \frac{MSAK}{\delta} } \del{ \ln\del{ \frac{MSAK}{\delta} } \wedge  \ln\del{\frac{MSA}{\gap_{\min}}} }.
\]
\end{lemma}
\begin{proof}
We expand the left hand side using the definition of $B^{k,\fut}$, and the fact that $\gap_{p,\min} \geq \gap_{\min}$:
\begin{align}
& \sum_{k=1}^K 
\rho_p^k(s,a)
\clip\del{ B^{k, \fut}(s,a), \frac{\gap_{p,\min}}{8SAH^2} } \\
\lesssim & 
\sum_{k=1}^K 
\rho_p^k(s,a)
\del{ H^3 \wedge 
\clip\del{ \frac{H^3 S L(n_p^k(s,a))}{n_p^k(s,a)}, \frac{\gap_{\min}}{8SAH^2} } }
\end{align}
We now decompose the sum $\sum_{k=1}^K$ to $\sum_{k=1}^{\tau_p(s,a)-1}$ and $\sum_{k=\tau_p(s,a)}^{K}$. The first part can be bounded by
\[
\sum_{k=1}^{\tau_p(s,a)-1}
\rho_p^k(s,a)
\del{ H^3 \wedge 
\clip\del{ \frac{H^3 S L(n_p^k(s,a))}{n_p^k(s,a)}, \frac{\gap_{\min}}{8SAH^2} } }
\leq 
\sum_{k=1}^{\tau_p(s,a)-1}
H^3 \rho_p^k(s,a) 
\leq 
H^3 N_2,
\]
which is 
at most $O\del{ H^3 \cdot \ln\del{ \frac{MSAK}{\delta}} }$.
For the second part, it can be bounded by:
\begin{align*}
& \sum_{k=\tau_p(s,a)}^K 
\rho_p^k(s,a) \del{ H^3 \wedge   \clip\del{ \frac{H^3 S L(n_p^k(s,a))}{n_p^k(s,a)}, \frac{\gap_{\min}}{8SAH^2} } } \\
\leq & H^3 \cdot 1 + \int_1^{\bar{n}_p^K(s,a)} \del{ H^3 \wedge   \clip\del{ \frac{H^3 S L(u) }{u} , \frac{\gap_{\min}}{8SAH^2} } } du
\\
\lesssim & 
H^3 + H^3 \ln\del{ \frac{MSA}{\delta}} + H^3 S \ln\del{ \frac{MSAK}{\delta} } \del{ \ln\del{ \frac{MSAK}{\delta} } \wedge  \ln\del{\frac{MHSA}{\gap_{\min}}} },
\end{align*}
where the second inequality is from Lemma~\ref{lem:integral} with $f_{\max} = H^3$, $C = H^3 S$, $\Delta = \frac{ \gap_{\min}}{8SAH^2}$, $N = MSA$, $\xi = \delta$, $\Gamma = 1$, $n = \bar{n}_p^K(s,a) \leq K$. In addition, observe that $H \leq S$ by our layered MDP assumption, we have
\[
\sum_{k} 
\rho_p^k(s,a) \clip\del{ B^{k,\lead}(s,a), \frac{\gap_{\min}}{8SAH^2} }
\lesssim
H^3 S \ln\del{ \frac{MSAK}{\delta} } \del{ \ln\del{ \frac{MSAK}{\delta} } \wedge  \ln\del{\frac{MSA}{\gap_{\min}}} }
\]
Summing over $s \in \Scal$, $a \in \Acal$, and $p \in [M]$, we get
\[
\sum_{s,a}
\sum_{k,p} 
\rho_p^k(s,a) \clip\del{ B^{k,\lead}(s,a), \frac{\gap_{\min}}{8SAH^2} }
\lesssim
M H^3 S^2 A \ln\del{ \frac{MSAK}{\delta} } \del{ \ln\del{ \frac{MSAK}{\delta} } \wedge  \ln\del{\frac{MSA}{\gap_{\min}}} }.
\]
\end{proof}

\subsection{Miscellaneous lemmas}
\label{subsec:misc-lemmas}

This subsection collects a few miscellaneous lemmas used throughout the upper bound proofs.

\begin{lemma}[Bias-variance decomposition]
For any random variable $X$ with $\EE[X] = \mu \in \RR$, and any $m \in \RR$, $\EE\sbr{(X-m)^2} = \EE\sbr{(X-\mu)^2} + (\mu-m)^2$. 
\label{lem:bv}
\end{lemma}

\begin{lemma}[\cite{simchowitz2019non}, Lemma F.5]
For random variables $X$ and $Y$, $\abs{\sqrt{\VV[X]} - \sqrt{\VV[Y]}} \leq \sqrt{\EE\sbr{( X - Y )^2}}$.
\label{lem:var-x-y}
\end{lemma}

\begin{lemma}
\label{lem:var-tv}
Suppose distributions $P$ and $Q$ are supported over $[0,B]$, and $\| P - Q \|_1 \leq \epsilon \leq 2$. Then:
\[ 
\abs{ 
\EE_{X \sim P}[X] - \EE_{X \sim Q}[X]  
}
\leq 
B \epsilon,
\]
\[
\abs{
\VV_{X \sim P}[X] - \VV_{X \sim Q}[X] 
} \leq 3B^2 \epsilon.
\]
\end{lemma}
\begin{proof}
First,
\[
\abs{ \EE_{X \sim P}[X] - \EE_{X \sim Q}[X] }
=
\abs{ \int_0^B x (p_X(x) - q_X(x)) dx }
\leq 
\int_0^B \abs{x} \abs{p_X(x) - q_X(x)} dx
\leq
B \| P - Q \|_1 \leq B \epsilon.
\]

Second, observe that
\[
\abs{ 
\EE_{X \sim P}[X^2] - \EE_{X \sim Q}[X^2]  
} \leq B^2 \epsilon.
\]
Meanwhile, 
\[
\abs{ (\EE_{X \sim P}[X])^2 - (\EE_{X \sim Q}[X])^2 }
\leq 
\abs{ \EE_{X \sim P}[X] - \EE_{X \sim Q}[X] } \cdot \abs{ \EE_{X \sim P}[X] + \EE_{X \sim Q}[X] }
\leq 2B \cdot B\epsilon =  2B^2\epsilon.
\]
Combining the above, we have
\[
\abs{
\VV_{X \sim P}[X] - \VV_{X \sim Q}[X] 
} \leq 3B^2 \epsilon.
\]
\end{proof}

\begin{lemma}
For $A, B, C, D, E, F \geq 0$:
\begin{enumerate}
\item If $A \leq B + C \sqrt{A}$, then $\sqrt{A} \leq \sqrt{B} + C$.
\label{item:dhm}

\item If $|D - E| \leq \sqrt{E F} + F$, then we have $\abs{ \sqrt{D} - \sqrt{E} } \leq 2\sqrt{F}$.
\end{enumerate}
\label{lem:rel-bound}
\end{lemma}
\begin{proof}
\begin{enumerate}
    \item The roots of $x^2 - C x - B = 0$ are $\frac{C \pm \sqrt{C^2 + 4B}}{2}$, and therefore $A$ must satisfy $\sqrt{A} \leq \frac{C + \sqrt{C^2 + 4B}}{2} \leq \frac{C + C + 2\sqrt{B}}{2} = C + \sqrt{B}$.

    \item First, $D - E \leq \abs{D - E} \leq \sqrt{EF} + F$; this implies that $D \leq E + 2\sqrt{EF} + F$, and therefore $\sqrt{D} \leq \sqrt{E} + \sqrt{F}$.
    
    On the other hand, $E \leq D + F + \sqrt{EF}$; therefore, applying item~\ref{item:dhm} with $A = E$, $B = D + F$, and $F = \sqrt{E}$, we have $\sqrt{E} \leq \sqrt{D + F} + \sqrt{F} \leq \sqrt{D} + 2\sqrt{F}$.
    \qedhere
\end{enumerate}
\end{proof}

\begin{lemma}
For $a \geq 0$, $1 \wedge (a + \sqrt{a}) \leq 1 \wedge 2\sqrt{a}$.
\label{lem:wedge-a-sqrta}
\end{lemma}
\begin{proof}
We consider the cases of $a \geq 1$ and $a < 1$ respectively. 
If $a \geq 1$, $\LHS = 1 = \RHS$.
Otherwise, $a \leq 1$; in this case, $\LHS = 1 \wedge (a + \sqrt{a}) \leq 1 \wedge (\sqrt{a} + \sqrt{a}) = \RHS$.
\end{proof}

\begin{lemma}[Special case of~\cite{simchowitz2019non}, Lemma B.5]
For $a_1, a_2, \Delta \geq 0$,
$\clip(a_1+a_2, \Delta) \leq 2 \clip(a_1, \Delta/4) +  2 \clip(a_2, \Delta/4)$. 
\label{lem:clipping-decomp}
\end{lemma}

\begin{lemma}[Integral calculation, extracted from  \cite{simchowitz2019non}, Lemma B.9]
\label{lem:integral}
Let $f(u) \leq \min( f_{\max}, \clip(g(u), \Delta))$, where $\Delta \in [0, \Gamma]$ and $\Gamma \geq 1$, and $g(u)$ is nonincreasing. Let $N \geq 1$ and $\xi \in (0,\frac12)$. Then:
\begin{enumerate}
    \item If $g(u) \lesssim \sqrt{ \frac{C \log \frac{N u}{\xi}}{u} }$ for some $C > 0$ such that $\ln C \lesssim \ln N$, then
    \[
    \int_\Gamma^n f(u/4) du \lesssim \sqrt{C n \ln\frac{N n}{\xi}} \wedge \frac{C}{\Delta} \ln\del{\frac{N n}{\xi}} .
    \]
    \item If $g(u) \lesssim \frac{ C\ln\frac{Nu}{\xi}}{u}$ for some $C > 0$ such that $\ln C \lesssim \ln N$, then 
    \[
    \int_\Gamma^n f(u/4) du \lesssim 
    f_{\max} \ln\frac{N}{\xi} + C \ln\frac{N n}{\xi} \cdot \del{ \ln\frac{N n}{\xi} \wedge \ln\frac{N \Gamma}{\Delta} }.
    \]
\end{enumerate}
\end{lemma}

\section{Proof of the Lower Bounds}

\subsection{Auxiliary Lemmas}

\begin{lemma}[Regret decomposition, \cite{simchowitz2019non}, Section H.2]
\label{lem:regret_decomp}
For any MPERL problem instance and any algorithm, we have
\begin{align}
\EE \sbr{\Reg(K)} \ge \sum_{p=1}^M \sum_{(s,a) \in \Scal_1 \times \Acal} \EE \sbr{n_p^{K+1}(s,a)} \gap_p(s,a) \label{eqn:regret_decomposition_ind},
\end{align}
where we recall that $\Scal_1$ is the subset of state space where the initial state distribution $p_0$ is supported on, 
$n_p^{K+1}(s,a)$ is the number of visits of $(s,a)$ by player $p$ at the beginning of the $(K+1)$-th episode (after the first $K$ episodes).
Furthermore, for any $(s,a) \in \Scal_1 \times \Acal$, we have
\begin{align}
\sum_{p=1}^M \EE \sbr{n_p^{K+1}(s,a)} \gap_p(s,a) \ge \EE \sbr{n^{K+1}(s,a)} \rbr{ \min_{p \in [M]} \gap_p(s,a)}, \label{eqn:regret_decomposition_collective}
\end{align}
where we recall that $n^{K+1}(s,a) = \sum_{p=1}^M n_p^{K+1}(s,a)$.
\end{lemma}

\begin{proof}
Eq.~\eqref{eqn:regret_decomposition_collective} follows straightforwardly from the fact that for every $(s,a,p) \in \Scal_1 \times \Acal \times [M]$, $\min_{p' \in [M]} \gap_{p'}(s,a) \le \gap_{p} (s,a)$.

We now prove Eq.~\eqref{eqn:regret_decomposition_ind}. Let $\pi_p^k$ denote $\pi^k(p)$. We have
\begin{align}
\EE \sbr{\Reg(K)}
& =  \EE \sbr{
\sum_{p=1}^M \sum_{k=1}^K \sum_{s \in \Scal_1} p_0(s_{1,p}^k= s) 
\rbr{V_{p}^\star(s) - V_{p}^{\pi_p^k}(s)}
} \nonumber \\
& \ge \EE \sbr{
\sum_{p=1}^M \sum_{k=1}^K \sum_{s \in \Scal_1} p_0(s_{1,p}^k= s) 
\rbr{V_{p}^\star(s) - Q_{p}^{\star}(s, \pi^k_p(s))}
} \nonumber \\
& = \EE \sbr{
\sum_{p=1}^M \sum_{k=1}^K \sum_{s \in \Scal_1} p_0(s)
\gap_p(s,\pi_p^k(s))
} \nonumber \\
& = \EE \sbr{
\sum_{p=1}^M \sum_{k=1}^K \sum_{s \in \Scal_1} \one\rbr{s_{1,p}^k = s}
\gap_p(s,\pi_p^k(s)) 
} \nonumber \\
& = \EE \sbr{
\sum_{p=1}^M \sum_{k=1}^K \sum_{(s,a) \in \Scal_1 \times \Acal} \one\rbr{s_{1,p}^k, \pi_p^k(s) = (s,a)}
\gap_p(s,a)
} \nonumber \\
& = \sum_{p=1}^M \sum_{(s,a) \in \Scal_1 \times \Acal} \EE \sbr{n_p^K(s,a)} \gap_p(s,a) \nonumber \\
\end{align}
where the first equality is from the definition of collective regret; 
the first inequality is from the simple fact that $V_p^\pi(s) = Q_p^\pi(s, \pi(s)) \leq Q_p^\star(s, \pi(s))$ for any policy $\pi$;
the second equality is from the definition of suboptimality gaps; and
the third equality is from the basic observation that $s_{1,p}^k \sim p_0$.
\end{proof}

\begin{lemma}[Divergence decomposition \cite{lattimore2020bandit,xmd21}]
\label{lem:divergence_decomp}%
For two MPERL problem instances, $\mathfrak{M}$ and $\mathfrak{M'}$, which only differ in the transition probabilities $\cbr{\PP_p(\cdot \mid s,a)}_{p \in [M], (s,a) \in \Scal \times \Acal}$, and for a fixed algorithm, let $\PP_{\mathfrak{M}}$ and $\PP_{\mathfrak{M'}}$ be the probability measures on the outcomes of running the algorithm on $\mathfrak{M}$ and $\mathfrak{M'}$, respectively. Then,
\[
\KL(\PP_{\mathfrak{M}}, \PP_{\mathfrak{M'}}) = \sum_{p=1}^M \sum_{(s,a) \in \Scal \times \Acal} \EE_{\mathfrak{M}} \sbr{n^{K+1}_p(s,a)} \KL\rbr{\PP^\mathfrak{M}_p(\cdot \mid s,a), \PP^\mathfrak{M'}_p(\cdot \mid s,a)},
\]
where $\PP^\mathfrak{M}_p(\cdot \mid s,a)$ and $\PP^\mathfrak{M'}_p(\cdot \mid s,a)$ are the transition probabilities of the problem instance $\mathfrak{M}$ and $\mathfrak{M'}$, respectively.
\end{lemma}

\begin{lemma}[Bretagnolle-Huber inequality, \cite{lattimore2020bandit}, Theorem 14.2]
\label{lem:bretagnolle-huber}
Let $\PP$ and $\QQ$ be two distributions on the same measurable space, and $A$ be an event. Then,
\[
\PP(A) + \QQ(A^C) \ge \frac12 \exp \rbr{-\KL(\PP,\QQ)}.
\]
\end{lemma}

\begin{lemma}[see, e.g., \cite{wzsrc21}, Lemma 25]
For any $x,y \in [\frac14, \frac34]$, $\KL\rbr{\Ber(x), \Ber(y)} \le 3(x-y)^2$.
\label{lem:KL_auxiliary}
\end{lemma}

\begin{lemma}
\label{lem:auxiliary_binomial}
Let $X$ be a Binomial random variable and $X \sim \mathrm{Bin}(n,p)$, where $n \ge \frac{1}{p}$.
Then,
\[
\EE \sbr{X^{\frac32}} \le 2(np)^{\frac32}.
\]
\end{lemma}
\begin{proof}
Let $Y = X^2$, and $f(y) = y^{\frac34}$.
We have $\EE \sbr{Y} = \EE \sbr{X^2} = \VV \sbr{X} + \EE \sbr{X}^2 = (np)^2 + np(1-p) \le (np)^2 + np \le 2(np)^2$, where the last inequality follows from the assumption that $n \ge \frac{1}{p}$. By Jensen's inequality, we have $\EE \sbr{X^{\frac32}} = \EE \sbr{f(Y)} \le f \rbr{\EE\sbr{Y}} \le \rbr{2n^2p^2}^{\frac34} \le 2(np)^{\frac32}$.
\end{proof}

\subsection{Gap independent lower bounds}
\begin{theorem}[Restatement of Theorem~\ref{thm:gap_ind_lb}]
For any $A \ge 2$, $H \ge 2$, {$S \ge 4H$}, $K \ge SA$, $M \in \NN$, and $l, l^C \in \NN$ such that $l + l^C = SA$ and $l \leq SA - 4(S + HA)$, there exists some $\epsilon$ such that for any algorithm $\Alg$, there exists an $\epsilon$-MPERL problem instance
with $S$ states, $A$ actions, $M$ players and an episode length of $H$
such that $\abr{\Ical_{\frac{\epsilon}{192H}}} \geq l$, and
\[
\EE \sbr{\Reg_\Alg(K)} \ge \Omega \rbr{ M \sqrt{H^2 l^C K } + \sqrt{ M H^2 l K }}.
\]
\end{theorem}

\begin{proof}[Proof]
The construction and techniques in this proof are inspired by \cite[Section E.1]{wzsrc21} and \cite{simchowitz2019non}.

Fix any algorithm $\Alg$;
we consider two cases:
\begin{enumerate}
    \item $l > M l^C$;
    \item $M l^C \ge l$.
\end{enumerate}

\paragraph{Case 1: $l > M l^C$.} 
Let $S_1 = S - 2(H-1)$, and $b = \ceil{\frac{l}{S_1}} \geq 1$. Let $\Delta = \sqrt{\frac{l + 1}{384MK}}$, 
and let $\epsilon = \frac12 H\Delta$.
We note that under the assumption that $K \ge SA$, and the observation that $l \leq SA$, we have $\Delta \le \frac14$. We define $(b+1)^{S_1}$ $\epsilon$-MPERL problem instances, each indexed by an element in $[b+1]^{S_1}$.
It suffices to show that, on at least one of these problem instances, $\EE \sbr{\Reg_\Alg(K)} \ge \Omega \rbr{\sqrt{ M H^2 l K }}$.

\subparagraph{Construction.} For $\textbf{a} = (a_1, \ldots, a_{S_1}) \in [b + 1]^{S_1}$ , we define the following $\epsilon$-MPERL problem instance, $\mathfrak{M}(\textbf{a}) = \cbr{\Mcal_p}_{p=1}^M$, with $S$ states, $A$ actions, and an episode length of $H$, such that for each $p \in [M]$, $\Mcal_p$ is constructed as follows:

\begin{itemize}
    \item $\Scal_1 = [S_1]$, and $p_0$ is a uniform distribution over the states in $\Scal_1$.
    
    \item For $h \in [2,H]$, $\Scal_h = \cbr{S_1 + 2h - 3, S_1 + 2h - 2}$.
    
    \item $\Acal = [A]$.
    
    \item For each $(s,a) \in \Scal \times \Acal$, the reward distribution $r_p(s,a)$ is a Bernoulli distribution, $\Ber(R_p(s,a))$, and we will specify $R_p(s,a)$ subsequently.
    
    \item For each state $s \in [S_1]$,
    \begin{align*}
        \PP_p(S_1+1 \mid s,a) = 
        \begin{cases}
        \frac{1}{2} + \Delta, & \text{if } a = a_s; \\
        \frac{1}{2}, & \text{if } a \in [b + 1] \setminus \cbr{a_s}; \\
        0, & \text{if } a \notin [b + 1];
        \end{cases}
    \end{align*}
    and for each $a \in \Acal$, $\PP_p(S_1+2 \mid s,a) = 1 - \PP_p(S_1+1 \mid s,a)$, and $R_p(s,a) = 0$.
    
    \hide{Just for my own note: 
    We can go one step further, and define $A^{S_1}$ instances, where for $\textbf{a} = (a_1, \ldots, a_{S_1}) \in [A]^{S_1}$, we have
        \begin{align*}
        \PP_p(S_1+1 \mid s,a) = 
        \begin{cases}
        \frac{1}{2} + \Delta, & \text{if } a = a_s; \\
        \frac{1}{2}, & \text{otherwise };
        \end{cases}
    \end{align*}
    but this does not help much, as this only increase the number of ``effective'' arms from $b+1$ to $A$, which is only a constant factor. 
    }
    
    \item For $h \in [2,H]$, and $a \in \Acal$, let 
    \begin{itemize}
        \item $\PP_p \rbr{S_1+2h-1 \mid S_1+2h-3,a} = 1$, $\PP_p \rbr{S_1+2h \mid S_1+2h-3,a} = 0$, and $R_p(S_1+2h-3, a) = 1$.
        \item $\PP_p \rbr{S_1+2h \mid S_1+2h-2,a} = 0$, $\PP_p \rbr{S_1+2h-1 \mid S_1+2h-2,a} = 1$, and $R_p(S_1+2h-2, a) = 0$.
    \end{itemize}
\end{itemize}

It can be easily verified that $\mathfrak{M}(\textbf{a}) = \cbr{\Mcal_p}_{p=1}^M$ is a $0$-MPERL problem instance, and hence an $\epsilon$-MPERL problem instance---the reward distributions and the transition probabilities are the same for all players, i.e., for every $p,q \in [M]$, and every $(s,a) \in \Scal \times \Acal$,
\[
\abs{R_p(s,a) - R_q(s,a)} = 0 \le \epsilon,\quad \abs{\PP_p(\cdot \mid s,a) - \PP_q(\cdot \mid s,a)} = 0 \le \frac \epsilon H.    
\]

\subparagraph{Suboptimality gaps.} We now calculate the suboptimality gaps of the state-action pairs in the above MDPs. For each $p \in [M]$ and each $(s,a) \in \Scal \times \Acal$,
\[
\gap_p(s,a) = V_p^\star(s) - Q_p^\star(s,a) = \max_{a'} Q_p^\star(s,a') - Q_p^\star(s,a).
\]
In $\mathfrak{M}(\mathbf{a})$, it can be easily observed that for every $p \in [M]$, and every $(s,a) \in \rbr{\Scal \setminus \Scal_1} \times \Acal$, $\gap_p(s,a) = 0$.
Now, for every $p \in [M]$, $(s,a) \in \Scal_1 \times \Acal$, we have
\[
\gap_p(s,a) = \max_{a'} Q_p^\star(s,a') - Q_p^\star(s,a) = (H-1) \rbr{\max_{a'} \PP_p(S_1 + 1 \mid s,a') - \PP_p(S_1 + 1 \mid s,a)}.
\]
It follows that, for every $p \in [M]$ and every state $s \in [S_1]$,
\begin{align*}
    \gap_p(s,a) = 
    \begin{cases}
    0, & \text{if } a = a_s; \\
    (H-1)\Delta, & \text{if } a \in [b + 1] \setminus a_s; \\
    (H-1)\rbr{\frac12 + \Delta}, & \text{if } a \notin [b + 1].
    \end{cases}
\end{align*}

\hide{Here is an important observation. For this instance, we can set $\epsilon = 0$ (if you are not comfortable with this, can set $\epsilon$ to be say, $O(\Delta)$); this  implies that $(H-1) \Delta > 96 H \epsilon$, which shows the stronger statement that $|\Ical_\epsilon| \geq l$!}

\subparagraph{Subpar state-action pairs.} It can be verified that in $\mathfrak{M}(\textbf{a})$, $\abs{\Ical_{\frac{\epsilon}{192H}}} \ge l$. Indeed,
since $(H-1)\Delta = (H-1)\frac{2\epsilon}{H} \ge \epsilon \ge \frac{\epsilon}{2} = 96H \frac{\epsilon}{192H}$, 
we have that $\Ical_{\frac{\epsilon}{192H}}$ is a superset of $\cbr{ (s,a): s \in [S_1], a \in [b+1] \setminus \cbr{a_s} }$, whose size is at least $S_1 b = S_1 \ceil{\frac{l}{S_1}} \ge l$.

It suffices to prove that
\[
\EE_{\mathbf{a} \sim \Unif\rbr{\sbr{b+1}^{S_1}}} \EE_{\mathfrak{M}(\mathbf{a})} \sbr{\Reg_{\Alg}(K)} \ge \frac{1}{640} \sqrt{MH^2lK},
\]
where we recall that $\mathbf{a} = (a_1, \ldots, a_{S_1}) $; furthermore, 
it suffices to show that, for any $s' \in [S_1]$,
\begin{align}
\EE_{\mathbf{a} \sim \Unif\rbr{\sbr{b+1}^{S_1}}} \EE_{\mathfrak{M}(\mathbf{a})} \sbr{N^{K+1}(s') - n^{K+1}(s', a_{s'})} \ge \frac{MK}{4S_1}, \label{eqn:gap_ind_lb_decomp_case1}
\end{align}
where $N^{K+1}(s') = \sum_{a \in \Acal} n^{K+1}(s',a)$;
this is because it follows from Eq.~\eqref{eqn:gap_ind_lb_decomp_case1} that
\begin{align*}
\EE_{\mathbf{a} \sim \Unif\rbr{\sbr{b+1}^{S_1}}} \EE_{\mathfrak{M}(\mathbf{a})} \sbr{\Reg_{\Alg}(K)} & \ge \sum_{s' \in \Scal_1} (H-1) \frac{\Delta}{4} \cdot \EE_{\mathbf{a} \sim \Unif\rbr{\sbr{b+1}^{S_1}}} \EE_{\mathfrak{M}(\mathbf{a})}  \sbr{ N^{K+1}(s') - n^{K+1}(s', a_{s'})} \\
& \ge \sum_{s' \in \Scal_1} \frac{H}{2} \cdot \frac{\Delta}{4} \cdot \frac{MK}{4S_1} \\
& \ge \frac{1}{640} \sqrt{MH^2lK},
\end{align*}
where the first inequality uses  Lemma~\ref{lem:regret_decomp} (the regret decomposition lemma).

Without loss of generality, we prove Eq.~\eqref{eqn:gap_ind_lb_decomp_case1} with $s' = 1$; the inequality for other $s'$ values are shown symmetrically. To this end, we use a standard technique and
define a set of helper problem instances. Specifically, for any $(a_2, a_3, \ldots, a_{S_1}) \in [b+1]^{S_1-1}$, we define a problem instance $\mathfrak{M}(0,a_2,\ldots,a_{s_1})$ such that it agrees with $\mathfrak{M}(a_1,a_2,\ldots,a_{s_1})$ on everything but $\PP_p(\cdot \mid 1, a_1)$'s, i.e., in $\mathfrak{M}(0,a_2,\ldots,a_{s_1})$, for every $p \in [M]$,
\[
\PP_p (S_1 + 1 \mid 1, a_1) = \frac12.
\]

Now, for each $(j,a_2,\ldots,a_{s_1}) \in \rbr{[0] \cup [b+1]} \times [b+1]^{S_1-1}$, let $\PP_{j, a_2, \ldots, a_{S_1}}$ denote the probability measure on the outcomes of running $\Alg$ on the problem instance $\mathfrak{M}(j,a_2,\ldots,a_{s_1})$.
Further, for each $j \in \{0\} \cup [b+1]$, we define
\[
\PP_j = \frac{1}{(b+1)^{S_1-1}} \sum_{a_2, \ldots, a_{S_1} \in [b+1]^{S_1 - 1}} \PP_{j, a_2, \ldots, a_{S_1}};
\]
and we use $\EE_j$ to denote the expectation with respect to $\PP_j$.

In subsequent calculations, for any index $m \in \rbr{[0] \cup [b+1]} \times [b+1]^{S_1-1}$,
we also denote by $\PP_m \rbr{\cdot \mid N^{K+1}(1)}$ and 
$\EE_m \sbr{\cdot \mid N^{K+1}(1)}$ the probability and expectation, respectively, conditional on a realization of $N^{K+1}(1)$ under $\PP_m$.
Observe that, for any $j \in \{0\} \cup [b+1]$,
\begin{align}
\PP_j(\cdot \mid N^{K+1}(1)) & = 
\frac{\PP_j (\cdot, N^{K+1}(1))}{\PP_j (N^{K+1}(1))} \nonumber \\
& = \frac{\frac{1}{(b+1)^{S_1-1}} \sum_{a_2, \ldots, a_{S_1} \in [b+1]^{S_1 - 1}} \PP_{j, a_2, \ldots, a_{S_1}}(\cdot, N^{K+1}(1))}{\PP_j(N^{K+1}(1))} \nonumber \\
& = \frac{1}{(b+1)^{S_1-1}} \sum_{a_2, \ldots, a_{S_1} \in [b+1]^{S_1 - 1}} \frac{\PP_{j, a_2, \ldots, a_{S_1}}(\cdot, N^{K+1}(1))}{\PP_{j, a_2, \ldots, a_{S_1}}(N^{K+1}(1))} \nonumber \\
& = \frac{1}{(b+1)^{S_1-1}} \sum_{a_2, \ldots, a_{S_1} \in [b+1]^{S_1 - 1}} \PP_{j, a_2, \ldots, a_{S_1}}(\cdot \mid N^{K+1}(1)), \label{eq:gap_ind_lb_conditional_helper}
\end{align}
where the first equality is from the definition of conditional probability; the second equality is from the definition of $\PP_j$; the third equality uses the fact that $\PP_j(N^{K+1}(1)) = \PP_{j, a_2, \ldots, a_{S_1}}(N^{K+1}(1))$ for any $a_2, \ldots, a_{S_1}$, which is true
because $N^{K+1}(1)$ is independent of $a_2, \ldots, a_{S_1}$ conditional on $j$; 
and the last equality, again, is from the definition of conditional probability.

We have, for each $j \in [b+1]$,
\begin{align}
    & \EE_{j} \sbr{n^{K+1}(1,j) \mid N^{K+1}(1)} - \EE_{0} \sbr{n^{K+1}(1,j) \mid N^{K+1}(1)} \nonumber \\
    \le & N^{K+1}(1) \nbr{\PP_j \rbr{\cdot \mid N^{K+1}(1)} - \PP_0 \rbr{\cdot \mid N^{K+1}(1)}}_1  \nonumber \\
    \le & N^{K+1}(1) \cdot \frac{1}{(b+1)^{S_1 - 1}} \sum_{a_2, \ldots, a_{S_1} \in [b+1]^{S_1 - 1}} \nbr{\PP_{j, a_2, \ldots, a_{S_1}}\rbr{\cdot \mid N^{K+1}(1)} - \PP_{0, a_2, \ldots, a_{S_1}}\rbr{\cdot \mid N^{K+1}(1)}}_1  \nonumber \\
    \le & N^{K+1}(1) \cdot \frac{1}{(b+1)^{S_1 - 1}} \sum_{a_2, \ldots, a_{S_1} \in [b+1]^{S_1 - 1}} \sqrt{2\KL\rbr{\Ber(\frac12 + \Delta), \Ber(\frac12)} \EE_{0, a_2, \ldots, a_{S_1}} \sbr{n^{K+1}(1,j) \mid N^{K+1}(1)}}  \nonumber \\
    \le & N^{K+1}(1) \cdot \frac{1}{(b+1)^{S_1 - 1}} \sum_{a_2, \ldots, a_{S_1} \in [b+1]^{S_1 - 1}} \sqrt{6 \Delta^2 \EE_{0, a_2, \ldots, a_{S_1}} \sbr{n^{K+1}(1,j) \mid N^{K+1}(1)}}  \nonumber \\
    \le & N^{K+1}(1) \sqrt{6 \frac{l+1}{384MK}  \cdot \EE_0 \sbr{n^{K+1}(1,j) \mid N^{K+1}(1)}}  \nonumber \\
    = & \frac18 N^{K+1}(1) \sqrt{ \frac{l+1}{MK} \cdot \EE_0 \sbr{n^{K+1}(1,j) \mid N^{K+1}(1)}} \label{eqn:gap_ind_lb_case1_eqn}.
\end{align}
where the first inequality is based on Lemma~\ref{lem:var-tv} and the fact that, conditional on $N^{K+1}(1)$, $n^{K+1}(1,j)$ has distribution supported on $[0, N^{K+1}(1)]$;
the second inequality follows from Equation~\eqref{eq:gap_ind_lb_conditional_helper} and the triangle inequality; the third inequality uses Pinsker's inequality and Lemma~\ref{lem:divergence_decomp} (the divergence decomposition lemma); the fourth inequality uses Lemma~\ref{lem:KL_auxiliary} and the fact that $\Delta \le \frac14$; and the last inequality follows from Jensen's inequality.

Since $N^{K+1}(1)$ has the same distribution under both $\PP_0$ and any $\PP_j$ (which is $\text{Bin}(K, \frac1{S_1})$), taking expectation with respect to $N^{K+1}(1)$, we have that, for any $j \in [b+1]$,
\begin{align*}
\EE_{j} \sbr{n^{K+1}(1,j)} - \EE_{0} \sbr{n^{K+1}(1,j)}
\leq &
\EE_0 \sbr{ \frac18 N^{K+1}(1) \sqrt{ \frac{l+1}{MK} \cdot \EE_0 \sbr{n^{K+1}(1,j) \mid N^{K+1}(1)}} }. 
\end{align*}

In subsequent derivations, we can now avoid bounding the conditional expectation. Specifically, we have
\begin{align}
    & \frac{1}{b+1} \sum_{j \in [b+1]} \EE_{j} \sbr{n^{K+1}(1,j)} \nonumber \\ 
    \le & \frac{1}{b+1}\sum_{j \in [b+1]} \EE_{0} \sbr{n^{K+1}(1,j)} + \frac{1}{b+1}\sum_{j \in [b+1]} \EE_0 \sbr{\frac18 N^{K+1}(1)\sqrt{ \frac{l+1}{MK} \cdot \EE_0 \sbr{n^{K+1}(1,j) \mid N^{K+1}(1)}}} \nonumber \\
    \le &  \frac{1}{b+1} \EE_0 \sbr{\sum_{j \in [b+1]} n^{K+1}(1,j)} + \EE_0 \sbr{ \frac{1}{8} N^{K+1}(1)\sqrt{ \frac{l+1}{MK} \cdot \frac{1}{b+1} \sum_{j \in [b+1]} \EE_0 \sbr{n^{K+1}(1,j) \mid N^{K+1}(1)}}} \nonumber \\
    \le & \frac{1}{b+1} \EE_0 \sbr{N^{K+1}(1)} + \EE_0 \sbr{\frac{1}{8} \sqrt{ \frac{l+1}{MK} \cdot \frac{1}{b+1}} \rbr{N^{K+1}(1)}^{\frac32}} \nonumber \\
    \le & \frac{1}{b+1} \EE_0 \sbr{N^{K+1}(1)} + \frac{1}{8} \sqrt{ \frac{S_1}{MK}} \cdot \EE_0 \sbr{\rbr{N^{K+1}(1)}^{\frac32}},
    \label{eqn:gap_ind_lb_case1_helper2}
\end{align}
where the first inequality follows from Eq.~\eqref{eqn:gap_ind_lb_case1_eqn} and algebra; the second inequality uses linearity of expectation and Jensen's inequality; the third inequality uses the facts that $\sum_{j \in [b+1]} n^{K+1}(1,j) \le N^{K+1}(1)$ and, for every $z \in [0] \cup [b+1]$, $$\sum_{j \in [b+1]} \EE_{z} \sbr{n^{K+1}(1,j) \mid N^{K+1}(1)} \le \sum_{j \in \Acal} \EE_{z} \sbr{n^{K+1}(1,j) \mid N^{K+1}(1)} = N^{K+1}(1);$$
and the last inequality uses the linearity of expectation and the construction that $b = \ceil{\frac{l}{S_1}}$, which implies that $l \le bS_1$ and therefore $l + 1 \le bS_1 + 1 \le bS_1 + S_1 = (b+1)S_1$.

It follows from Equation~\eqref{eqn:gap_ind_lb_case1_helper2} that
\begin{align*}
    \frac{1}{b+1}\sum_{j \in [b+1]} \EE_{j} \sbr{n^{K+1}(1,j)} 
    & \le \frac{1}{b+1} \cdot \frac{MK}{S_1} + \frac18 \sqrt{ \frac{S_1}{MK}} \cdot \EE_0 \sbr{\rbr{N^{K+1}(1)}^{\frac32}} \\
    & \le \frac{MK}{2S_1} + \frac14 \sqrt{ \frac{S_1}{MK} \rbr{\frac{MK}{S_1}}^3} \\
    & \le \frac{3MK}{4S_1},
\end{align*}
where %
the second inequality uses the fact that $\frac{1}{b+1} \le \frac12$ and Lemma~\ref{lem:auxiliary_binomial} under the assumption that $K \ge S_1$.

It then follows that
\[
\frac{1}{b+1}\sum_{j \in [b+1]} \EE_{j} \sbr{N^{K+1}(1) - n^{K+1}(1,j)} \ge \frac{1}{b+1}\sum_{j \in [b+1]} \EE_{j} \sbr{N^{K+1}(1)} - \frac{3MK}{4S_1} = \frac{MK}{4S_1},
\]
and we have
\[
\EE_{\mathbf{a} \sim \Unif\rbr{\sbr{b+1}^{S_1}}} \EE_{\mathfrak{M}(\mathbf{a})} \sbr{N^{K+1}(1) - n^{K+1}(1, a_{1})} \ge \frac{MK}{4S_1}.
\]

\paragraph{Case 2: $Ml^C \ge l$.} Again, let $S_1 = S - 2(H-1)$. Let $u = \ceil{\frac{l}{S_1}}$ and $v = A - u = A - \ceil{\frac{l}{S_1}}$. Furthermore, let $\Delta = \sqrt{\frac{vS_1}{384K}}$, and $\epsilon = 2H\Delta$. We note that under the assumption that $K \ge SA$ and the fact that $v S_1 \leq SA$, we have $\Delta \le \frac14$. We will define $v^{S_1 \times M}$ $\epsilon$-MPERL problem instances, each indexed by an element in $[v]^{S_1 \times M}$.
It suffices to show that, on at least one of the instances, $\EE \sbr{\Reg_{\Alg}(K)} \ge \Omega \rbr{M\sqrt{H^2l^CK}}$.

\subparagraph{Facts about $v$.}
There are two helpful facts about $v$ that can be easily verified:
\begin{itemize}
    \item $vS_1 \ge \frac12 l^C$. This is true because, by definition, $vS_1 \ge S_1A - l - S_1 = S_1A - (SA - l^C) - S_1 = l^C - (SA - S_1A) - S_1 = l^C - \rbr{2(H-1)A + S_1}$; since, by assumption, $l \le SA - 4(S + HA)$, we have $l^C \ge 4(HA + S) \ge 2\rbr{2(H-1)A + S_1}$; it then follows that $vS_1 \ge l^C - \rbr{2(H-1)A + S_1} \ge \frac12 l^C$.
    
    \item $v \geq 2$. This is true because, as shown above, $vS_1 \ge \frac12 l^C$ and $l^C \ge 4(HA + S)$, which imply that $v \ge \frac{2(HA + S)}{S_1} \ge \frac{2 S_1}{S_1} = 2$.
\end{itemize}

\subparagraph{Construction.} For $\textbf{a} = (a_{1,1}, \ldots, a_{1,M},a_{2,1}, \ldots, a_{S_1, M}) \in [v]^{S_1 \times M}$ , we define the following $\epsilon$-MPERL problem instance, $\mathfrak{M}(\textbf{a}) = \cbr{\Mcal_p}_{p=1}^M$, with $S$ states, $A$ actions, and an episode length of $H$, such that for each $p \in [M]$, $\Mcal_p$ is constructed in the same way as it is for case 1, except for the transition probabilities of $(s,a) \in \Scal_1 \times \Acal$:

\begin{itemize}
    \item For each state $s \in [S_1]$,
    \begin{align*}
        \PP_p(S_1+1 \mid s,a) = 
        \begin{cases}
        \frac{1}{2} + \Delta, & \text{if } a = a_{s,p}; \\
        \frac{1}{2}, & \text{if } a \in [v] \setminus \cbr{a_{s,p}}; \\
        0, & \text{if } a \notin [v];
        \end{cases}
    \end{align*}
    and for each $a \in \Acal$, $\PP_p(S_1+2 \mid s,a) = 1 - \PP_p(S_1+1 \mid s,a)$, and $R_p(s,a) = 0$.
    
\end{itemize}

We now verify that $\mathfrak{M}(\mathbf{a})$ is an $\epsilon$-MPMAB problem instance.
It can be easily observed that the reward distributions are the same for all players, i.e., for every $p,q \in [M]$ and every $(s,a) \in \Scal \times \Acal$,
\[
\abs{R_p(s,a) - R_q(s,a)} = 0 \le \epsilon.
\]
Regarding the transition probabilities, for every $(s,a) \in \rbr{(\Scal_1 \times \rbr{\Acal \setminus [v]})} \cup \rbr{ \rbr{\Scal \setminus \Scal_1} \times \Acal}$, we observe that the transition probabilities are the same for all players.
Furthermore, for every $p,q \in [M]$ and every $(s,a) \in \Scal_1 \times [v]$,
\[
\nbr{\PP_p \rbr{\cdot \mid s,a} - \PP_q \rbr{\cdot \mid s,a}}_1 \le 2\Delta = \frac{\epsilon}{H}.
\]
Therefore, $\mathfrak{M}(\mathbf{a})$ is an $\epsilon$-MPMAB problem instance.

\subparagraph{Suboptimality gaps.} Similar to the arguments in Case 1, it can be shown that
for every $p \in [M]$, and every $(s,a) \in \rbr{\Scal \setminus \Scal_1} \times \Acal$, $\gap_p(s,a) = 0$. And, for every $p \in [M]$, and every $s \in \Scal_1$,
\begin{align*}
    \gap_p(s,a) = 
    \begin{cases}
    0, & \text{if } a = a_{s,p}; \\
    (H-1)\Delta, & \text{if } a \in [v] \setminus a_{s,p}; \\
    (H-1)\rbr{\frac12 + \Delta}, & \text{if } a \notin [v].
    \end{cases}
\end{align*}

\hide{Here is the key observation. Recall $\epsilon = 2 H \Delta$. If $K \geq \Omega(S A H^2)$ (which is a mild assumption to satisfy), we have $(H-1)(\frac12 + \Delta) > 192 H^2 \Delta = 96 H \epsilon$, implying that $|\Ical_\epsilon| \geq l$!}

\hide{This observation, together with the observation in case 1, implies that our gap-dependent lower bound is sharpened to $\Ical_\epsilon$.}

\subparagraph{Subpar state-action pairs.} Based on the above construction, 
for every $(s,a) \in \Scal_1 \times \rbr{\Acal \setminus [v]}$ and every $p \in [M]$, $\gap_p(s,a) = (H-1) \rbr{\frac12 + \Delta} \ge 3(H-1)\Delta = \frac{3(H-1)}{2H} \epsilon \ge \frac{3}{4}\epsilon \ge 96H \rbr{\frac{\epsilon}{192H}}$, where the first inequality uses the fact that $\Delta \le \frac14$. Therefore, $\Ical_{\frac{\epsilon}{192H}}$ is a superset of $[S_1] \times ([A] \setminus [v])$, whose cardinality is at least $(A - v)S_1 = uS_1 \ge l$, i.e., $\abs{\Ical_{\frac{\epsilon}{192H}}} \ge l$. 

Now, it suffices to prove that
\[
\EE_{\mathbf{a} \sim \Unif\rbr{\sbr{v}^{S_1 \times M}}} \EE_{\mathfrak{M}(\mathbf{a})} \sbr{\Reg_{\Alg}(K)} \ge \frac{1}{240} M\sqrt{H^2l^CK},
\]
where we recall that $\mathbf{a} = (a_{1,1}, \ldots, a_{1,M}, a_{2,1}, \ldots, a_{S_1, M})$.
It suffices to show, for any $s' \in [S_1]$ and any $p' \in [M]$,
\begin{align}
\EE_{\mathbf{a} \sim \Unif\rbr{\sbr{v}^{S_1 \times M}}} \EE_{\mathfrak{M}(\mathbf{a})} \sbr{N^{K+1}_{p'}(s') - n^{K+1}_{p'}(s', a_{s'})} \ge \frac{K}{4S_1}, \label{eqn:gap_ind_lb_decomp_case2}
\end{align}
where $N^{K+1}_{p'}(s') = \sum_{a \in \Acal} n_{p'}^{K+1}(s',a)$.
To see this, by Lemma~\ref{lem:regret_decomp}, we have
\begin{align*}
\EE_{\mathbf{a} \sim \Unif\rbr{\sbr{v}^{S_1 \times M}}} \EE_{\mathfrak{M}(\mathbf{a})} \sbr{\Reg_{\Alg}(K)} & \ge \sum_{p=1}^M \sum_{s' \in \Scal_1} (H-1)\Delta \cdot \EE_{\mathbf{a} \sim \Unif\rbr{\sbr{v}^{S_1 \times M}}} \EE_{\mathfrak{M}(\mathbf{a})} \sbr{N^{K+1}_p(s') - n^{K+1}_p(s', a_{s'})} \\
& \ge \frac{H-1}{4} MK \sqrt{\frac{vS_1}{384K}} \\
& \ge \frac{1}{160} M\sqrt{H^2(vS_1)K} \\
& \ge \frac{1}{240} M\sqrt{H^2l^CK},
\end{align*}
where the last inequality uses the fact that $vS_1 \ge \frac12 l^C$. 

Without loss of generality, it suffices to prove Eq.~\eqref{eqn:gap_ind_lb_decomp_case2} for $s' = 1$ and $p' = 1$; the other settings of $(s',p')$ can be handled symmetrically. Similar to case 1, we define a set of helper problem instances:
for any $(a_{1,2}, \ldots, a_{S_1, M}) \in [v]^{S_1 \times M -1}$, we define a problem instance $\mathfrak{M}(0,a_{1,2}, \ldots, a_{S_1, M})$ such that it agrees with $\mathfrak{M}(a_{1,1},a_{1,2}, \ldots, a_{S_1, M})$ on everything but $\PP_1(\cdot \mid 1, a_1)$, namely, in $\mathfrak{M}(0,a_{1,2}, \ldots, a_{S_1, M})$, $\PP_1 (S_1 + 1 \mid 1, a_1) = \frac12$.

For each $(j,a_{1,2}, \ldots, a_{S_1, M}) \in \rbr{[0] \cup [v]} \times [v]^{S_1 \times M -1}$, let $\PP_{j,a_{1,2}, \ldots, a_{S_1, M}}$ denote the probability measure on the outcomes of running $\Alg$ on the problem instance $\mathfrak{M}(j,a_{1,2}, \ldots, a_{S_1, M})$.
Further, for each $j \in \{0\} \cup [v]$, we define
\[
\PP_j = \frac{1}{v^{S_1 \times M -1}} \sum_{a_{1,2}, \ldots, a_{S_1, M} \in [v]^{S_1 \times M - 1}} \PP_{j,a_{1,2}, \ldots, a_{S_1, M}};
\]
and we use $\EE_j$ to denote the expectation with respect to $\PP_j$. 
In subsequent calculations, for any $m \in \rbr{[0] \cup [v]} \times [v]^{S_1 \times M -1}$, we also denote by $\PP_m \rbr{\cdot \mid N_1^{K+1}(1)}$ and 
$\EE_m \sbr{\cdot \mid N_1^{K+1}(1)}$ the probability and expectation conditional on a realization of $N_1^{K+1}(1)$ under $\PP_m$. Similar to case 1, it can be shown that, for any $j \in \{0\} \cup [v]$,

\begin{align}
\PP_j(\cdot \mid N^{K+1}(1)) 
= 
\frac{1}{v^{S_1 \times M -1}} \sum_{a_{1,2}, \ldots, a_{S_1, M} \in [v]^{S_1 \times M - 1}} \PP_{j,a_{1,2}, \ldots, a_{S_1, M}} \rbr{\cdot \mid N^{K+1}(1)}. \label{eq:gap_ind_lb_conditional_helper2}
\end{align}

Now, for each $j \in [v]$, we have
\begin{align}
    & \EE_{j} \sbr{n_1^{K+1}(1,j) \mid N_1^{K+1}(1)} - \EE_{0} \sbr{n_1^{K+1}(1,j) \mid N_1^{K+1}(1)} \nonumber \\
    \le & N_1^{K+1}(1) \nbr{\PP_j \rbr{\cdot \mid N_1^{K+1}(1)} - \PP_0 \rbr{\cdot \mid N_1^{K+1}(1)}}_1  \nonumber \\
    \le & N_1^{K+1}(1) \cdot \frac{1}{v^{S_1 \times M - 1}} \sum_{a_{1,2}, \ldots, a_{S_1,M} \in [v]^{S_1 \times M - 1}} \nbr{\PP_{j, a_{1,2}, \ldots, a_{S_1,M}}\rbr{\cdot \mid N_1^{K+1}(1)} - \PP_{0, a_{1,2}, \ldots, a_{S_1,M}}\rbr{\cdot \mid N_1^{K+1}(1)}}_1  \nonumber \\
    \le & N_1^{K+1}(1) \cdot \frac{1}{v^{S_1 \times M - 1}} \sum_{a_{1,2}, \ldots, a_{S_1,M} \in [v]^{S_1 \times M - 1}} \sqrt{2\KL\rbr{\Ber(\frac12 + \Delta), \Ber(\frac12)} \EE_{0, a_2, \ldots, a_{S_1}} \sbr{n_1^{K+1}(1,j) \mid N_1^{K+1}(1)}}  \nonumber \\
    \le & N_1^{K+1}(1) \cdot \frac{1}{v^{S_1 \times M - 1}} \sum_{a_{1,2}, \ldots, a_{S_1,M} \in [v]^{S_1 \times M - 1}} \sqrt{6 \Delta^2 \EE_{0, a_2, \ldots, a_{S_1}} \sbr{n_1^{K+1}(1,j) \mid N_1^{K+1}(1)}}  \nonumber \\
    \le & N_1^{K+1}(1) \cdot \sqrt{\frac{6vS_1}{384K}  \cdot \EE_0 \sbr{n_1^{K+1}(1,j) \mid N_1^{K+1}(1)}}  \nonumber \\
    = & \frac18 N_1^{K+1}(1) \sqrt{ \frac{vS_1}{K} \cdot \EE_0 \sbr{n_1^{K+1}(1,j) \mid N_1^{K+1}(1)}} \label{eqn:gap_ind_lb_case2_eqn}.
\end{align}
where the first inequality is based on Lemma~\ref{lem:var-tv} and the fact that, conditional on $N_1^{K+1}(1)$, $n_1^{K+1}(1,j)$ has distribution supported on $[0, N_1^{K+1}(1)]$;
the second inequality follows from Equation~\eqref{eq:gap_ind_lb_conditional_helper2} and the triangle inequality; the third inequality uses Pinsker's inequality and Lemma~\ref{lem:divergence_decomp} (the divergence decomposition lemma); the fourth inequality uses Lemma~\ref{lem:KL_auxiliary} and the fact that $\Delta \le \frac14$; and the last inequality follows from Jensen's inequality.

Using arguments similar to the ones shown for case 1, we have that
\begin{align}
    & \frac{1}{v} \sum_{j \in [v]} \EE_{j} \sbr{n_1^{K+1}(1,j)} \nonumber \\ 
    \le & \frac{1}{v} \EE_0 \sbr{n_1^{K+1}(1,j)} + \EE_0 \sbr{\frac{1}{8} N_1^{K+1}(1)\sqrt{ \frac{vS_1}{K} \cdot \frac{1}{v} \sum_{j \in [v]} \EE_0 \sbr{n_1^{K+1}(1,j) \mid N_1^{K+1}(1)}}} \nonumber \\
    \le & \frac{1}{v} \EE_0 \sbr{ N^{K+1}(1)} + \frac{1}{8} \sqrt{ \frac{S_1}{K}} \cdot \EE_0 \sbr{\rbr{N_1^{K+1}(1)}^{\frac32}} \nonumber \\
    \le & \frac{1}{v} \cdot \frac{K}{S_1} + \frac14 \sqrt{ \frac{S_1}{K} \rbr{\frac{K}{S_1}}^3} \nonumber \\
    \le & \frac{3K}{4S_1}, \nonumber
\end{align}
where the second to last inequality is from Lemma~\ref{lem:auxiliary_binomial} under the assumption that $K \ge S_1$, and the last inequality uses the fact that $v \ge 2$.

It then follows that
\[
\frac{1}{v}\sum_{j \in [v]} \EE_{j} \sbr{N_1^{K+1}(1) - n_1^{K+1}(1,j)} \ge \frac{1}{v}\sum_{j \in [v]} \EE_{j} \sbr{N_1^{K+1}(1)} - \frac{3K}{4S_1} = \frac{K}{4S_1},
\]
and we thereby have shown that
\[
\EE_{\mathbf{a} \sim \Unif\rbr{\sbr{v}^{S_1 \times M}}} \EE_{\mathfrak{M}(\mathbf{a})} \sbr{N_1^{K+1}(1) - n_1^{K+1}(1, a_{1})} \ge \frac{K}{4S_1}. \qedhere
\]
\end{proof}

\hide{Note that the red part is the additional assumption that needs to be added to fix this proof.}
\subsection{Gap dependent lower bound}
\begin{theorem}[Restatement of Theorem~\ref{thm:gap_dep_lb}]
Fix $\epsilon \ge 0$. For any $S \in \NN$, $A \ge 2$, $H \ge 2$, $M \in \NN$, such that $S \geq 2(H-1)$, let $S_1 = S - 2(H-1)$; and let $\cbr{\Delta_{s,a,p}}_{(s,a,p) \in [S_1] \times [A] \times [M]}$ be any set of values such that 
\begin{itemize}
    \item for every $(s,a,p) \in [S_1] \times [A] \times [M]$, $\Delta_{s,a,p} \in [0,\frac{H}{48\sqrt{M}}]$;
    \item for every $(s,p) \in [S_1] \times [M]$, there exists at least one action $a \in [A]$ such that $\Delta_{s,a,p} = 0$;
    \item and, for every $(s,a) \in [S_1] \times [A]$ and $p,q \in [M]$, $\abs{\Delta_{s,a,p} - \Delta_{s,a,q}} \leq \epsilon/4$.
\end{itemize}
There exists an $\epsilon$-MPERL problem instance with $S$ states, $A$ actions, $M$ players and an episode length of $H$, such that
$\Scal_1 = [S_1]$, $\abr{\Scal_h} = 2$ for all $h \ge 2$, and
\[
\gap_p(s,a) = 
\Delta_{s,a,p}, \quad \forall (s,a,p) \in [S_1] \times [A] \times [M].
\]
For this problem instance, any sublinear regret algorithm $\Alg$ for the $\epsilon$-MPERL problem must have regret at least 
\[
\EE \sbr{\Reg_\Alg(K)} \ge \Omega \rbr{ \ln K \rbr{ \sum_{p \in [M]} \sum_{\substack{(s,a) \in \Ical^C_{ \rbr{\epsilon/768 H}}: \\ \gap_p(s,a) > 0}} \frac{H^2}{\gap_p(s,a)}
+
\sum_{(s,a) \in \Ical_{ \rbr{\epsilon/768 H}}} \frac{H^2}{\min_p \gap_p(s,a)}}}.
\]
\end{theorem}

\begin{proof}[Proof]
The construction and techniques in this proof are inspired by \cite{simchowitz2019non} and \cite{wzsrc21}. 

\paragraph{Proof outline.} We will construct an $\epsilon$-MPERL problem instance, $\mathfrak{M}$, and show that, for any sublinear regret algorithm and sufficiently large $K$, the following two claims are true:
\begin{enumerate}
    \item for any $(s,a) \in \Scal \times \Acal$ such that for all $p$, $\gap_{p}(s,a) > 0$,
    \begin{align}
    \label{eq:gap_dep_lb_claim1}
    \EE_{\mathfrak{M}} \sbr{n^K(s,a)} \ge
    \Omega\del{\frac{H^2}{\rbr{\min_p \gap_p(s,a)}^2} \ln K };
    \end{align}
    \item for any $(s,a) \in \Ical^C_{\frac{\epsilon}{768H}}$ and any $p \in [M]$ such that $\gap_{p}(s,a) > 0$,
    \begin{align}
    \label{eq:gap_dep_lb_claim2}
    \EE_{\mathfrak{M}} \sbr{n_p^K(s,a)} \ge \Omega \del{ \frac{H^2}{\rbr{ \gap_p(s,a)}^2} \ln K}.
    \end{align}
\end{enumerate}
The rest then follows from Lemma~\ref{lem:regret_decomp} (the regret decomposition lemma).

\paragraph{Construction of $\mathfrak{M}$.} Given any set of values $\cbr{\Delta_{s,a,p}}_{(s,a,p) \in [S_1] \times [A] \times [M]}$ that satisfies the assumptions in the theorem statement, we can construct a collection of MDPs $\cbr{\Mcal_p}_{p=1}^M$, such that for each $p \in [M]$, $\Mcal_p$ is as follows, and $\mathfrak{M} = \cbr{\Mcal_p}_{p=1}^M$ is an $\epsilon$-MPERL problem instance: %
\begin{itemize}
    \item $\Scal_1 = [S_1]$, and $p_0$ is a uniform distribution over the states in $\Scal_1$. %
    
    \item For $h \in [2,H]$, $\Scal_h = \cbr{S_1+2h-3, S_1+2h-2}$.
    
    \item $\Acal = [A]$.
    
    \item %
    For all $(s,a) \in \Scal \times \Acal$, the reward distribution $r_p(s,a)$ is a Bernoulli distribution, $\Ber(R_p(s,a))$, and we specify $R_p(s,a)$ subsequently.
    
    \item For every $(s,a) \in \Scal_1 \times [A]$, set
    $\bar{\Delta}_{s,a}^p = \frac{\Delta_{s,a,p}}{H-1}.$ Then, let 
    \[
    \PP_p \rbr{S_1+1 \mid s,a} = \frac12 - \bar{\Delta}_{s,a}^p,\quad \PP_p \rbr{S_1+2 \mid s,a} = \frac12 + \bar{\Delta}_{s,a}^p,
    \]
    and $R_p(s,a) = 0$. Since $\Delta_{s,a,p} \in [0, H/48]$, $\bar{\Delta}_{s,a}^p \le \frac{H}{48(H-1)} \le \frac{1}{24}$, where the last inequality follows from the assumption that $H \ge 2$. Therefore, $\PP_p \rbr{S_1+1 \mid s,a} \in [0,1]$, and $\PP_p \rbr{S_1+2 \mid s,a} \in [0,1]$.
    
    \item For $h \in [2,H]$, and $a \in [A]$, let 
    \begin{itemize}
        \item $\PP_p \rbr{S_1+2h-1 \mid S_1+2h-3,a} = 1$, $\PP_p \rbr{S_1+2h \mid S_1+2h-3,a} = 0$, and $R_p(S_1+2h-3, a) = 1$.
        \item $\PP_p \rbr{S_1+2h \mid S_1+2h-2,a} = 0$, $\PP_p \rbr{S_1+2h-1 \mid S_1+2h-2,a} = 1$, and $R_p(S_1+2h-2, a) = 0$.
    \end{itemize}
\end{itemize}

By the assumption that for every $(s,p) \in [S_1] \times [M]$, there exists at least one action $a \in [A]$ such that $\Delta_{s,a,p} = 0$, we have that there is at least one action $a$ such that $\bar{\Delta}_{s,a}^{p} = 0$. We verify that for every $(s,a,p) \in [S_1] \times [A] \times [M]$,
\begin{align*}
    \gap_p(s,a) & = V_p^\star(s) - Q_p^\star(s,a) \\
    & = \max_{a'} Q_p^\star(s,a') - Q_p^\star(s,a) \\
    & = (H-1) \bar{\Delta}_{s,a}^{p} \\
    & = \Delta_{s,a,p}.
\end{align*}

We now verify that the above MPERL problem instance $\mathfrak{M} = \cbr{\Mcal_p}_{p=1}^M$ is an $\epsilon$-MPERL problem instance:
\begin{enumerate}

\item The reward distributions are the same for all players, namely, for all $p, q$, 
\[
\abs{ R_p(s,a) - R_q(s,a)} = 0 \leq \epsilon, \forall (s,a) \in \Scal \times \Acal.
\]

\item Further, by the assumption that for every $(s,a) \in [S_1] \times [A]$ and $p,q \in [M]$, $\abs{\Delta_{s,a,p} - \Delta_{s,a,q}} \leq \epsilon/4$, we have that
\[
\abs{\bar{\Delta}_{s,a}^{p} - \bar{\Delta}_{s,a}^{q}} = \frac{\abs{\Delta_{s,a,p} - \Delta_{s,a,q}}}{H-1} \le \frac{\epsilon}{4(H-1)} \le \frac{\epsilon}{2H}.
\]
It then follows that
\[
\| \PP_p \rbr{\cdot \mid s,a} - \PP_q \rbr{\cdot \mid s,a} \|_1 = 2\abs{\bar{\Delta}_{s,a}^{p} - \bar{\Delta}_{s,a}^{q}} \le \frac{\epsilon}{H}.
\]

Meanwhile, for every $(s,a) \in \rbr{\Scal \setminus \Scal_1}  \times \Acal$
\[
\| \PP_p \rbr{\cdot \mid s,a} - \PP_q \rbr{\cdot \mid s,a} \|_1
= 0 \leq \frac{\epsilon}{H}.
\]

In summary, for every $(s,a) \in \Scal \times \Acal$,
\[
\| \PP_p \rbr{\cdot \mid s,a} - \PP_q \rbr{\cdot \mid s,a} \|_1 \leq \frac{\epsilon}{H}.
\]
\end{enumerate}

We are now ready to prove the two claims:
\begin{enumerate}
    \item \textbf{Proving claim 1 (Equation~\eqref{eq:gap_dep_lb_claim1}):}
    
    Fix any $(s_0,a_0) \in [S_1] \times [A]$ such that $\bar{\Delta}_{s_0,a_0}^{\min} = \min_p \bar{\Delta}_{s_0,a_0}^p > 0$. It can be easily observed that $\gap_p(s_0,a_0) > 0$ for all $p$. Define $p_0 = \argmin_p \bar{\Delta}_{s_0,a_0}^p$. We can construct a new problem instance, $\mathfrak{M'}$, which agrees with $\mathfrak{M}$, except that
    \[
    \forall p \in [M], \PP_p \rbr{S_1+1 \mid s_0,a_0} = \frac12 - \bar{\Delta}_{s_0,a_0}^p + 2\bar{\Delta}_{s_0,a_0}^{\min}, \PP_p \rbr{S_1+2 \mid s_0,a_0} = \frac12 + \bar{\Delta}_{s_0,a_0}^p - 2\bar{\Delta}_{s_0,a_0}^{\min}.
    \]
    $\mathfrak{M'}$ is an $\epsilon$-MPERL problem instance. 
    To see this, we note that the only change is in $\PP_p \rbr{\cdot \mid s_0,a_0}$ for all $p \in [M]$. In this new instance, it is still true that for every $p,q \in [M]$,
    \[
    \| \PP_p \rbr{\cdot \mid s_0,a_0} - \PP_q \rbr{\cdot \mid s_0,a_0} \|_1 = 2\abs{\bar{\Delta}_{s_0,a_0}^{p} - \bar{\Delta}_{s_0,a_0}^{q}} \le \frac{\epsilon}{H}.
    \]
    
    {
    Fix any sublinear regret algorithm $\Alg$ for the $\epsilon$-MPERL problem.
    By Lemma \ref{lem:divergence_decomp} (the divergence decomposition lemma), we have
    \[
    \KL(\PP_{\mathfrak{M}}, \PP_{\mathfrak{M'}}) = \sum_{p=1}^M \EE_{\mathfrak{M}} \sbr{n^K_p(s_0,a_0)} \KL\rbr{\PP^\mathfrak{M}_p(\cdot \mid s_0,a_0), \PP^\mathfrak{M'}_p(\cdot \mid s_0,a_0)},
    \]
    where $\PP_{\mathfrak{M}}$ and $\PP_{\mathfrak{M'}}$ are the probability measures on the outcomes of running $\Alg$ on $\mathfrak{M}$ and $\mathfrak{M'}$, respectively; $\PP^\mathfrak{M}_p(\cdot \mid s_0,a_0)$, $\PP^\mathfrak{M'}_p(\cdot \mid s_0,a_0)$ are the transition probabilities for $(s_0,a_0)$ and player $p$ in $\mathfrak{M}$ and $\mathfrak{M'}$, respectively.
    
    We observe that, for any $p \in [M]$,
    \begin{align*}
    & \KL \rbr{\PP^\mathfrak{M}_p( \cdot \mid s_0,a_0), \PP^\mathfrak{M'}_p( \cdot \mid s_0,a_0)}  \\
    = & \KL \rbr{\Ber\rbr{\frac12 - \bar{\Delta}_{s_0,a_0}^p},
    \Ber\rbr{\frac12 - \bar{\Delta}_{s_0,a_0}^p + 2\bar{\Delta}_{s_0,a_0}^{\min}}} \\
    \le & 12(\bar{\Delta}_{s_0,a_0}^{\min})^2,
    \end{align*}
    where the last inequality follows from Lemma~\ref{lem:KL_auxiliary} and the assumption that $\Delta_{s,a,p} \le \frac{H}{48}$.

    In addition, $\sum_{p=1}^M \EE_{\mathfrak{M}} \sbr{n^K_p(s_0,a_0)} = \EE_{\mathfrak{M}}\sbr{ n^K(s_0, a_0)}$. It then follows that
    \begin{align}
        \KL(\PP_{\mathfrak{M}}, \PP_{\mathfrak{M'}}) \le & 12 \EE_{\mathfrak{M}} \sbr{n^K(s_0,a_0)} (\bar{\Delta}_{s_0,a_0}^{\min})^2. \label{eqn:KL_delta} %
    \end{align}

    Now, in the original $\epsilon$-MPERL problem instance, $\mathfrak{M}$, by Equation~\eqref{eqn:regret_decomposition_ind} and Markov's Inequality,
    we have
    \[
    \EE_\mathfrak{M} \sbr{\Reg_{\Alg}(K)} \ge \frac{K}{4S_1} \rbr{(H-1)\bar{\Delta}_{s_0,a_0}^{\min}} \PP_{\mathfrak{M}} \rbr{n_{p_0}^K(s_0,a_0) \ge \frac{K}{4S_1}};
    \]
    where we note that $\bar{\Delta}_{s_0,a_0}^{p_0} = \bar{\Delta}_{s_0,a_0}^{\min}$.
    In $\mathfrak{M'}$, the new $\epsilon$-MPERL problem instance, we have
    \begin{align*}
    \EE_\mathfrak{M'} \sbr{\Reg_{\Alg}(K)} \ge &  \rbr{(H-1)\bar{\Delta}_{s_0,a_0}^{\min}} \EE_{\mathfrak{M'}} \sbr{ \sum_{a \neq a_0} n_{p_0}(s_0, a) } \\
    = &  \rbr{(H-1)\bar{\Delta}_{s_0,a_0}^{\min}} \EE_{\mathfrak{M'}} \sbr{ N_{p_0}^K(s_0) - n_{p_0}(s_0, a_0) } \\
    \ge &  \frac{K}{4S_1} \rbr{(H-1)\bar{\Delta}_{s_0,a_0}^{\min}} \PP_{\mathfrak{M'}} \rbr{ N_{p_0}^K(s_0) - n_{p_0}(s_0, a_0) \geq \frac{K}{4S_1}} \\
    \ge &  \frac{K}{4S_1} \rbr{(H-1)\bar{\Delta}_{s_0,a_0}^{\min}} \PP_{\mathfrak{M'}} \rbr{ N_{p_0}^K(s_0) \geq \frac{K}{2S_1}, n_{p_0}(s_0, a_0) \leq \frac{K}{4S_1}} \\
    \ge &  \frac{K}{4S_1} \rbr{(H-1)\bar{\Delta}_{s_0,a_0}^{\min}} \del{ \PP_{\mathfrak{M'}} \rbr{ n_{p_0}(s_0, a_0) \leq \frac{K}{4S_1}} - \exp(-\frac{K}{8S_1}) },
    \end{align*}
    where the first inequality is by Equation~\eqref{eqn:regret_decomposition_ind}; the second inequality is by Markov's Inequality; the third inequality is by simple algebra; and the last inequality is by Chernoff bound that  $\PP_{\mathfrak{M'}} \del{N_{p_0}^K(s_0) < \frac{K}{2S_1}} \leq \exp(-\frac{K}{8S_1})$, and $\PP(A \cap B) \geq \PP(B) - \PP(A^C)$ for events $A, B$.

    It then follows that
    \begin{align*}
        & \EE_\mathfrak{M} \sbr{\Reg_{\Alg}(K)} 
        + 
        \EE_\mathfrak{M'} \sbr{\Reg_{\Alg}(K)} \\
        = &
        \frac{K}{2} \rbr{(H-1)\bar{\Delta}_{s_0,a_0}^{\min}} \rbr{\PP_{\mathfrak{M}} \rbr{n_{p_0}^K(s_0,a_0) \ge \frac{K}{4 S_1}} + \PP_{\mathfrak{M'}} \rbr{n_{p_0}^K(s_0,a_0) < \frac{K}{4 S_1}}
        -
        \exp(-\frac{K}{8S_1})
        } \\
        \ge & 
        \frac{K}{2} \rbr{(H-1)\bar{\Delta}_{s_0,a_0}^{\min}} \del{\frac12 \exp \rbr{-\KL(\PP_{\mathfrak{M}}, \PP_{\mathfrak{M'}})} - \exp(-\frac{K}{8S_1})} \\
        \ge &
        \frac{K}{2} \rbr{(H-1)\bar{\Delta}_{s_0,a_0}^{\min}} \del{ \frac12 \exp \rbr{-12 \EE_{\mathfrak{M}} \sbr{n^K(s_0,a_0)} (\bar{\Delta}_{s_0,a_0}^{\min})^2 } - \exp(-\frac{K}{8S_1}) },
    \end{align*}
    where the first inequality follows from Lemma~\ref{lem:bretagnolle-huber} (the Bretagnolle-Huber inequality), and the second inequality follows from Eq.~\eqref{eqn:KL_delta}.
    Observe that $\EE_{\mathfrak{M}} \sbr{n^K(s_0,a_0)} \leq \frac{M K}{S_1}$; in addition, by our assumption that $\Delta_{s,a,p} \le \frac{H}{48 \sqrt{M} }$ for every $(s,a,p)$, we have  $\bar{\Delta}_{s_0,a_0}^{\min} \leq \frac{1}{24 \sqrt{M}}$. These together implies that
    $\frac14 \exp \rbr{-12 \EE_{\mathfrak{M}} \sbr{n^K(s_0,a_0)} (\bar{\Delta}_{s_0,a_0}^{\min})^2 } \geq \frac14 \exp(-\frac{K}{48 S_1}) \geq \exp(-\frac{K}{8S_1})$, as long as $K \geq 20 S_1$.
    Therefore, we have 
    \[
    \EE_\mathfrak{M} \sbr{\Reg_{\Alg}(K)} 
    + 
    \EE_\mathfrak{M'} \sbr{\Reg_{\Alg}(K)}
    \geq 
    \frac{K}{2} \rbr{(H-1)\bar{\Delta}_{s_0,a_0}^{\min}} \cdot 
    \frac14 \exp \rbr{-12 \EE_{\mathfrak{M}} \sbr{n^K(s_0,a_0)} (\bar{\Delta}_{s_0,a_0}^{\min})^2 }.
    \]
    
    Now, under the assumption that $\Alg$ is a sublinear regret algorithm, we have
    \[
    \frac{K}{8} \rbr{(H-1)\bar{\Delta}_{s_0,a_0}^{\min}} \exp \rbr{-12 \EE_{\mathfrak{M}} \sbr{n^K(s_0,a_0)} (\bar{\Delta}_{s_0,a_0}^{\min})^2} \le 2CK^{\alpha}.
    \]
    It follows that
    \begin{align*}
        \EE_{\mathfrak{M}} \sbr{n^K(s_0,a_0)} & \ge \frac{1}{12\rbr{\bar{\Delta}_{s_0,a_0}^{\min}}^2} \ln \rbr{\frac{(H-1)\bar{\Delta}_{s_0,a_0}^{\min}K^{1 - \alpha}}{16 C}} \\
        & = \frac{(H-1)^2}{12\rbr{\min_p \gap_p(s_0,a_0)}^2} \ln \rbr{\frac{\min_p \gap_p(s_0,a_0) K^{1 - \alpha}}{16 C}} \\
        & \ge \frac{H^2}{48 \rbr{\min_p \gap_p(s_0,a_0)}^2} \ln \rbr{\frac{\min_p \gap_p(s_0,a_0)K^{1 - \alpha}}{16 C}}.
    \end{align*}
    
    We then have
    \[
    \EE_{\mathfrak{M}} \sbr{n^K(s_0,a_0)} \ge
    \Omega\del{\frac{H^2}{\rbr{\min_p \gap_p(s_0,a_0)}^2} \ln K }.
    \]
    }
    
    \item \textbf{Proving Claim 2 (Equation~\eqref{eq:gap_dep_lb_claim2}):}
    
    Fix any $(s_0,a_0) \in \Ical^C_{\frac{\epsilon}{768 H}}$ and $p_0 \in [M]$ such that $\bar{\Delta}^{p_0}_{(s_0,a_0)} > 0$, which means that $\gap_{p_0}(s_0,a_0) > 0$.
    We have that for all $p \in [M]$,
    \begin{equation}
    \bar{\Delta}_{s_0,a_0}^{p} = \frac{\Delta_{s_0,a_0}^{p}}{H-1} = \frac{\gap_{p}(s_0,a_0)}{H-1}  \le \frac{96 H (\epsilon/( 768 H))}{(H-1)} \le \frac{\epsilon}{8(H-1)} \le \frac{\epsilon}{4H}.
    \label{eqn:small-gap-i-eps}
    \end{equation}
    We construct a new problem instance, $\mathfrak{M'}$, which agrees with $\mathfrak{M}$ except that
    \begin{align*}
    \PP_{p_0} \rbr{S_1+1 \mid s_0,a_0} = \frac12 - \bar{\Delta}_{s_0,a_0}^{p_0} + 2\bar{\Delta}_{s_0,a_0}^{p_0} = \frac12 + \bar{\Delta}_{s_0,a_0}^{p_0}, \\
    \PP_{p_0} \rbr{S_1+2 \mid s_0,a_0} = \frac12 + \bar{\Delta}_{s_0,a_0}^{p_0} - 2\bar{\Delta}_{s_0,a_0}^{p_0} = \frac12 - \bar{\Delta}_{s_0,a_0}^{p_0}.
    \end{align*}
    $\mathfrak{M'}$ is an $\epsilon$-MPERL problem instance. 
    To see this, we note that the only change is in $\PP_{p_0} \rbr{\cdot \mid s_0,a_0}$. In this new instance, it is still true that for any $q \neq p_0$,
    \[
    \|\PP_{p_0}(\cdot \mid s_0,a_0) - \PP_q(\cdot \mid s_0,a_0) \|_1
    \le 2\abs{\bar{\Delta}_{s_0,a_0}^{p_0} + \bar{\Delta}_{s_0,a_0}^q} \le \frac{\epsilon}{H}.
    \]
    where the last inequality uses Equation~\eqref{eqn:small-gap-i-eps} that $\bar{\Delta}_{s_0,a_0}^p \le \frac{\epsilon}{4H}$ for every $p \in [M]$. 
    
    {
    Fix any sublinear regret algorithm $\Alg$.
    By Lemma~\ref{lem:divergence_decomp} (the divergence decomposition lemma), we have
    \[
    \KL(\PP_{\mathfrak{M}}, \PP_{\mathfrak{M'}}) = \EE_{\mathfrak{M}} \sbr{n^K_{p_0}(s_0,a_0)} \KL\rbr{\PP^\mathfrak{M}_{p_0}(\cdot \mid s_0,a_0), \PP^\mathfrak{M'}_{p_0}(\cdot \mid s_0,a_0)}.
    \]
    
    Using a similar reasoning as before, and recall that $\bar{\Delta}_{s_0,a_0}^{p} \leq \frac{1}{24}$, we can show that
    \begin{align}
        \KL(\PP_{\mathfrak{M}}, \PP_{\mathfrak{M'}}) \le & 12 \EE_{\mathfrak{M}} \sbr{n_{p_0}^K(s_0,a_0)} (\bar{\Delta}_{s_0,a_0}^{p_0})^2 
        \leq \frac{K}{48 S_1}, \label{eqn:KL_delta_p0}
    \end{align}
    and consequently, as long as $K \geq 20 S_1$,
    \[
    \frac14 \exp \rbr{-12 \EE_{\mathfrak{M}} \sbr{n^K(s_0,a_0)} (\bar{\Delta}_{s_0,a_0}^{\min})^2 } \geq \frac14 \exp(-\frac{K}{48 S_1}) \geq \exp(-\frac{K}{8S_1}).
    \]

    Similar to the proof of Claim 1, we have the following argument. 
    In the original $\epsilon$-MPERL problem instance, $\mathfrak{M}$, we have
    $\EE_\mathfrak{M} \sbr{\Reg_{\Alg}(K)} \ge \frac{K}{4S_1} \rbr{(H-1)\bar{\Delta}_{s_0,a_0}^{p_0}} \PP_{\mathfrak{M}} \rbr{n_{p_0}^K(s_0,a_0) \ge \frac{K}{4S_1}}$;
    and in $\mathfrak{M'}$, the new $\epsilon$-MPERL problem instance, we have $\EE_\mathfrak{M'} \sbr{\Reg_{\Alg}(K)} \ge \frac{K}{4S_1} \rbr{(H-1)\bar{\Delta}_{s_0,a_0}^{p_0}} \del{ \PP_{\mathfrak{M'}} \rbr{n_{p_0}^K(s_0,a_0) < \frac{K}{4S_1}} - \exp(-\frac{K}{8S_1}) }$.

    It then follows that
    \begin{align*}
        & \EE_\mathfrak{M} \sbr{\Reg_{\Alg}(K)} 
        + 
        \EE_\mathfrak{M'} \sbr{\Reg_{\Alg}(K)} \\
        \ge & 
        \frac{K}{2} \rbr{(H-1)\bar{\Delta}_{s_0,a_0}^{p_0}} \del{ \frac12 \exp \rbr{-\KL(\PP_{\mathfrak{M}}, \PP_{\mathfrak{M'}})} - \exp(-\frac{K}{8S_1}) } \\
        \ge &
        \frac{K}{8} \rbr{(H-1)\bar{\Delta}_{s_0,a_0}^{p_0}} \exp \rbr{-12 \EE_{\mathfrak{M}} \sbr{n_{p_0}^K(s_0,a_0)} (\bar{\Delta}_{s_0,a_0}^{p_0})^2}.
    \end{align*}
    
    Now, under the assumption that $\Alg$ is a sublinear regret algorithm, we have
    \[
    \frac{K}{8} \rbr{(H-1)\bar{\Delta}_{s_0,a_0}^{p_0}} \exp \rbr{-12 \EE_{\mathfrak{M}} \sbr{n_{p_0}^K(s_0,a_0)} (\bar{\Delta}_{s_0,a_0}^{p_0})^2} \le 2CK^{\alpha}.
    \]
    It follows that
    \begin{align*}
        \EE_{\mathfrak{M}} \sbr{n_{p_0}^K(s_0,a_0)} & \ge \frac{1}{12\rbr{\bar{\Delta}_{s_0,a_0}^{p_0}}^2} \ln \rbr{\frac{(H-1)\bar{\Delta}_{s_0,a_0}^{p_0}K^{1 - \alpha}}{16 C}} \\
        & \ge \frac{H^2}{24\rbr{\gap_{p_0}(s_0,a_0)}^2} \ln \rbr{\frac{\gap_{p_0}(s_0,a_0) K^{1 - \alpha}}{16 C}}.
    \end{align*}
    
    We then have that
    \[
    \EE_{\mathfrak{M}} \sbr{n_{p_0}^K(s_0,a_0)} \ge \Omega \rbr{\frac{H^2}{\rbr{ \gap_{p_0}(s_0,a_0)}^2} \ln K}.
    \]
    }    
    
    \paragraph{Combing the two claims:}
    We note that in $\mathfrak{M}$, for any $(s,a,p) \in \rbr{\Scal \setminus \Scal_1} \times \Acal \times [M]$, $\gap_p(s,a) = 0$. 
    It then follows from Lemma~\ref{lem:regret_decomp} (the regret decomposition lemma) and the fact that for any $(s,a,p) \in \Ical_{ {\epsilon/768H}} \times [M], \gap_p(s,a) > 0$,
    that%
    \begin{align*}
    \EE \sbr{\Reg_\Alg(K)} \ge & 
    \sum_{p=1}^M \sum_{(s,a) \in \Scal_1 \times \Acal} \EE \sbr{n_p^K(s,a)} \gap_p(s,a) \\
    \ge &
    \Omega \rbr{\ln K \rbr{ \sum_{p \in [M]} \sum_{\substack{(s,a) \in \Ical^C_{{\epsilon/768H}}: \\ \gap_p(s,a) > 0}} \frac{H^2}{\gap_{p}(s,a)}
    +
    \sum_{(s,a) \in \Ical_{ {\epsilon/768 H}}} \frac{H^2}{\min_p \gap_p(s,a)}}} .
    \end{align*}
\end{enumerate}
\end{proof}

\end{document}